\documentclass[11pt,twoside, english]{article} 
\usepackage{babel}
\usepackage{array}
\usepackage{float}
\usepackage{textcomp}
\usepackage{mathtools}
\usepackage{amsmath}
\usepackage{amsthm}
\usepackage{amssymb}
\usepackage{graphicx}
\usepackage{xargs}[2008/03/08]

\makeatletter

\providecommand{\tabularnewline}{\\}
\floatstyle{ruled}
\newfloat{algorithm}{tbp}{loa}
\providecommand{\algorithmname}{Algorithm}
\floatname{algorithm}{\protect\algorithmname}

\theoremstyle{plain}
\newtheorem{thm}{\protect\theoremname}
\theoremstyle{definition}
\newtheorem{defn}[thm]{\protect\definitionname}
\theoremstyle{plain}
\newtheorem{prop}[thm]{\protect\propositionname}
\theoremstyle{plain}
\newtheorem*{thm*}{\protect\theoremname}
\theoremstyle{plain}
\newtheorem{lem}[thm]{\protect\lemmaname}
\theoremstyle{plain}
\newtheorem{cor}[thm]{\protect\corollaryname}

\usepackage{algorithm}
\usepackage{algorithmic}
\renewcommand{\algorithmicrequire}{\textbf{Input:}}
\renewcommand{\algorithmicensure}{\textbf{Output:}}

\newcommand{\bx}{\boldsymbol{x}}

\providecommand{\corollaryname}{Corollary}
\providecommand{\lemmaname}{Lemma}
\providecommand{\theoremname}{Theorem}

\@ifundefined{showcaptionsetup}{}{%
 \PassOptionsToPackage{caption=false}{subfig}}
\usepackage{subfig}
\makeatother

\providecommand{\definitionname}{Definition}
\providecommand{\propositionname}{Proposition}
\providecommand{\theoremname}{Theorem}

\usepackage[utf8]{inputenc} 
\usepackage[T1]{fontenc}
\usepackage{booktabs}       
\usepackage{amsfonts}       
\usepackage{nicefrac}       
\usepackage{microtype}      

\usepackage{epsf}
\usepackage{epsfig}
\usepackage{fancyhdr}
\usepackage{graphics}
\usepackage{psfrag}
\usepackage{fullpage}
\usepackage{pdfpages}

\usepackage{booktabs} 
\usepackage{mathtools}

\usepackage{color}
\usepackage{bbm}
\usepackage{caption}
\usepackage{sidecap}
\sidecaptionvpos{figure}{c}

\usepackage{url}
\usepackage[colorlinks=True,linkcolor=magenta,citecolor=blue,urlcolor=blue,pagebackref=true,backref=true]
{hyperref}
\renewcommand*{\backref}[1]{\ifx#1\relax \else Page #1 \fi}
\renewcommand*{\backrefalt}[4]{%
    \ifcase #1 \footnotesize{(Not cited.)}%
    \or        \footnotesize{(Cited on page~#2.)}%
    \else      \footnotesize{(Cited on pages~#2.)}%
    \fi}

\begin{document}
\global\long\def\sidenote#1{\marginpar{\small\emph{{\color{Medium}#1}}}}%

\global\long\def\se{\hat{\text{se}}}%
\global\long\def\interior{\text{int}}%
\global\long\def\boundary{\text{bd}}%
\global\long\def\ML{\textsf{ML}}%
\global\long\def\GML{\mathsf{GML}}%
\global\long\def\HMM{\mathsf{HMM}}%
\global\long\def\support{\text{supp}}%
\global\long\def\new{\text{*}}%
\global\long\def\stir{\text{Stirl}}%
\global\long\def\mA{\mathcal{A}}%
\global\long\def\mB{\mathcal{B}}%
\global\long\def\mF{\mathcal{F}}%
\global\long\def\mK{\mathcal{K}}%
\global\long\def\mH{\mathcal{H}}%
\global\long\def\mX{\mathcal{X}}%
\global\long\def\mZ{\mathcal{Z}}%
\global\long\def\mS{\mathcal{S}}%
\global\long\def\Ical{\mathcal{I}}%
\global\long\def\mT{\mathcal{T}}%
\global\long\def\Pcal{\mathcal{P}}%
\global\long\def\dist{d}%
\global\long\def\HX{\entro\left(X\right)}%
\global\long\def\entropyX{\HX}%
\global\long\def\HY{\entro\left(Y\right)}%
\global\long\def\entropyY{\HY}%
\global\long\def\HXY{\entro\left(X,Y\right)}%
\global\long\def\entropyXY{\HXY}%
\global\long\def\mutualXY{\mutual\left(X;Y\right)}%
\global\long\def\mutinfoXY{\mutualXY}%
\global\long\def\given{\mid}%
\global\long\def\gv{\given}%
\global\long\def\goto{\rightarrow}%
\global\long\def\asgoto{\stackrel{a.s.}{\longrightarrow}}%
\global\long\def\pgoto{\stackrel{p}{\longrightarrow}}%
\global\long\def\dgoto{\stackrel{d}{\longrightarrow}}%
\global\long\def\lik{\mathcal{L}}%
\global\long\def\logll{\mathit{l}}%
\global\long\def\vectorize#1{\mathbf{#1}}%

\global\long\def\vt#1{\mathbf{#1}}%
\global\long\def\gvt#1{\boldsymbol{#1}}%
\global\long\def\idp{\ \bot\negthickspace\negthickspace\bot\ }%
\global\long\def\cdp{\idp}%
\global\long\def\das{}%
\global\long\def\id{\mathbb{I}}%
\global\long\def\idarg#1#2{\id\left\{  #1,#2\right\}  }%
\global\long\def\iid{\stackrel{\text{iid}}{\sim}}%
\global\long\def\bzero{\vt 0}%
\global\long\def\bone{\mathbf{1}}%
\global\long\def\boldm{\boldsymbol{m}}%
\global\long\def\bff{\vt f}%
\global\long\def\bx{\boldsymbol{x}}%

\global\long\def\bd{\boldsymbol{d}}%
\global\long\def\bl{\boldsymbol{l}}%
\global\long\def\bu{\boldsymbol{u}}%
\global\long\def\bo{\boldsymbol{o}}%
\global\long\def\bh{\boldsymbol{h}}%
\global\long\def\bs{\boldsymbol{s}}%
\global\long\def\bz{\boldsymbol{z}}%
\global\long\def\xnew{y}%
\global\long\def\bxnew{\boldsymbol{y}}%
\global\long\def\bX{\boldsymbol{X}}%
\global\long\def\tbx{\tilde{\bx}}%
\global\long\def\by{\boldsymbol{y}}%
\global\long\def\bY{\boldsymbol{Y}}%
\global\long\def\bZ{\boldsymbol{Z}}%
\global\long\def\bU{\boldsymbol{U}}%
\global\long\def\bv{\boldsymbol{v}}%
\global\long\def\bn{\boldsymbol{n}}%
\global\long\def\bV{\boldsymbol{V}}%
\global\long\def\bI{\boldsymbol{I}}%
\global\long\def\bw{\vt w}%
\global\long\def\balpha{\gvt{\alpha}}%
\global\long\def\bbeta{\gvt{\beta}}%
\global\long\def\bmu{\gvt{\mu}}%
\global\long\def\btheta{\boldsymbol{\theta}}%
\global\long\def\blambda{\boldsymbol{\lambda}}%
\global\long\def\bgamma{\boldsymbol{\gamma}}%
\global\long\def\bpsi{\boldsymbol{\psi}}%
\global\long\def\bphi{\boldsymbol{\phi}}%
\global\long\def\bpi{\boldsymbol{\pi}}%
\global\long\def\bomega{\boldsymbol{\omega}}%
\global\long\def\bepsilon{\boldsymbol{\epsilon}}%
\global\long\def\btau{\boldsymbol{\tau}}%
\global\long\def\bxi{\boldsymbol{\xi}}%
\global\long\def\realset{\mathbb{R}}%
\global\long\def\realn{\realset^{n}}%
\global\long\def\integerset{\mathbb{Z}}%
\global\long\def\natset{\integerset}%
\global\long\def\integer{\integerset}%

\global\long\def\natn{\natset^{n}}%
\global\long\def\rational{\mathbb{Q}}%
\global\long\def\rationaln{\rational^{n}}%
\global\long\def\complexset{\mathbb{C}}%
\global\long\def\comp{\complexset}%

\global\long\def\compl#1{#1^{\text{c}}}%
\global\long\def\and{\cap}%
\global\long\def\compn{\comp^{n}}%
\global\long\def\comb#1#2{\left({#1\atop #2}\right) }%
\global\long\def\nchoosek#1#2{\left({#1\atop #2}\right)}%
\global\long\def\param{\vt w}%
\global\long\def\Param{\Theta}%
\global\long\def\meanparam{\gvt{\mu}}%
\global\long\def\Meanparam{\mathcal{M}}%
\global\long\def\meanmap{\mathbf{m}}%
\global\long\def\logpart{A}%
\global\long\def\simplex{\Delta}%
\global\long\def\simplexn{\simplex^{n}}%
\global\long\def\dirproc{\text{DP}}%
\global\long\def\ggproc{\text{GG}}%
\global\long\def\DP{\text{DP}}%
\global\long\def\ndp{\text{nDP}}%
\global\long\def\hdp{\text{HDP}}%
\global\long\def\gempdf{\text{GEM}}%
\global\long\def\rfs{\text{RFS}}%
\global\long\def\bernrfs{\text{BernoulliRFS}}%
\global\long\def\poissrfs{\text{PoissonRFS}}%
\global\long\def\grad{\gradient}%
\global\long\def\gradient{\nabla}%
\global\long\def\partdev#1#2{\partialdev{#1}{#2}}%
\global\long\def\partialdev#1#2{\frac{\partial#1}{\partial#2}}%
\global\long\def\partddev#1#2{\partialdevdev{#1}{#2}}%
\global\long\def\partialdevdev#1#2{\frac{\partial^{2}#1}{\partial#2\partial#2^{\top}}}%
\global\long\def\closure{\text{cl}}%
\global\long\def\cpr#1#2{\Pr\left(#1\ |\ #2\right)}%
\global\long\def\var{\text{Var}}%
\global\long\def\Var#1{\text{Var}\left[#1\right]}%
\global\long\def\cov{\text{Cov}}%
\global\long\def\Cov#1{\cov\left[ #1 \right]}%
\global\long\def\COV#1#2{\underset{#2}{\cov}\left[ #1 \right]}%
\global\long\def\corr{\text{Corr}}%
\global\long\def\sst{\text{T}}%
\global\long\def\SST{\sst}%
\global\long\def\ess{\mathbb{E}}%

\global\long\def\Ess#1{\ess\left[#1\right]}%
\newcommandx\ESS[2][usedefault, addprefix=\global, 1=]{\underset{#2}{\ess}\left[#1\right]}%
\global\long\def\fisher{\mathcal{F}}%

\global\long\def\bfield{\mathcal{B}}%
\global\long\def\borel{\mathcal{B}}%
\global\long\def\bernpdf{\text{Bernoulli}}%
\global\long\def\betapdf{\text{Beta}}%
\global\long\def\dirpdf{\text{Dir}}%
\global\long\def\gammapdf{\text{Gamma}}%
\global\long\def\gaussden#1#2{\text{Normal}\left(#1, #2 \right) }%
\global\long\def\gauss{\mathbf{N}}%
\global\long\def\gausspdf#1#2#3{\text{Normal}\left( #1 \lcabra{#2, #3}\right) }%
\global\long\def\multpdf{\text{Mult}}%
\global\long\def\poiss{\text{Pois}}%
\global\long\def\poissonpdf{\text{Poisson}}%
\global\long\def\pgpdf{\text{PG}}%
\global\long\def\wshpdf{\text{Wish}}%
\global\long\def\iwshpdf{\text{InvWish}}%
\global\long\def\nwpdf{\text{NW}}%
\global\long\def\niwpdf{\text{NIW}}%
\global\long\def\studentpdf{\text{Student}}%
\global\long\def\unipdf{\text{Uni}}%
\global\long\def\transp#1{\transpose{#1}}%
\global\long\def\transpose#1{#1^{\mathsf{T}}}%
\global\long\def\mgt{\succ}%
\global\long\def\mge{\succeq}%
\global\long\def\idenmat{\mathbf{I}}%
\global\long\def\trace{\mathrm{tr}}%
\global\long\def\argmax#1{\underset{_{#1}}{\text{argmax}} }%
\global\long\def\argmin#1{\underset{_{#1}}{\text{argmin}\ } }%
\global\long\def\diag{\text{diag}}%
\global\long\def\norm{}%
\global\long\def\spn{\text{span}}%
\global\long\def\vtspace{\mathcal{V}}%
\global\long\def\field{\mathcal{F}}%
\global\long\def\ffield{\mathcal{F}}%
\global\long\def\inner#1#2{\left\langle #1,#2\right\rangle }%
\global\long\def\iprod#1#2{\inner{#1}{#2}}%
\global\long\def\dprod#1#2{#1 \cdot#2}%
\global\long\def\norm#1{\left\Vert #1\right\Vert }%
\global\long\def\entro{\mathbb{H}}%
\global\long\def\entropy{\mathbb{H}}%
\global\long\def\Entro#1{\entro\left[#1\right]}%
\global\long\def\Entropy#1{\Entro{#1}}%
\global\long\def\mutinfo{\mathbb{I}}%
\global\long\def\relH{\mathit{D}}%
\global\long\def\reldiv#1#2{\relH\left(#1||#2\right)}%
\global\long\def\KL{KL}%
\global\long\def\KLdiv#1#2{\KL\left(#1\parallel#2\right)}%
\global\long\def\KLdivergence#1#2{\KL\left(#1\ \parallel\ #2\right)}%
\global\long\def\crossH{\mathcal{C}}%
\global\long\def\crossentropy{\mathcal{C}}%
\global\long\def\crossHxy#1#2{\crossentropy\left(#1\parallel#2\right)}%
\global\long\def\breg{\text{BD}}%
\global\long\def\lcabra#1{\left|#1\right.}%
\global\long\def\lbra#1{\lcabra{#1}}%
\global\long\def\rcabra#1{\left.#1\right|}%
\global\long\def\rbra#1{\rcabra{#1}}%

\begin{center}

{\bf{\LARGE{On Label Shift in Domain Adaptation \\
via Wasserstein Distance}}}
  
\vspace*{.2in}
{\large{
\begin{tabular}{cccc}
Trung Le$^{\ddagger}$ &  Dat Do$^{\mathsection}$  & Tuan Nguyen$^{\ddagger}$ & Huy Nguyen$^{\diamond}$
\end{tabular}
\begin{tabular}{ccc}
Hung Bui$^{\diamond}$ & Nhat Ho$^{\dagger}$ & Dinh Phung$^{\ddagger}$
\end{tabular}
}}

\vspace*{.2in}

\begin{tabular}{c}
Monash University$^\ddagger$; VinAI Research$^\diamond$; \vspace*{-2mm}\\ University of Michigan, Ann Arbor$^{\mathsection}$; University of Texas, Austin$^{\dagger}$\\
\end{tabular}


\vspace*{.2in}

\begin{abstract}
We study the label shift problem between the source and target domains
in general domain adaptation (DA) settings. We consider transformations
transporting the target to source domains, which enable us to align
the source and target examples. Through those transformations, we
define the label shift between two domains via optimal transport and
develop theory to investigate the properties of DA under various DA
settings (e.g., closed-set, partial-set, open-set, and universal settings).
Inspired from the developed theory, we propose\emph{ }\textbf{\emph{L}}\emph{abel
and }\textbf{\emph{D}}\emph{ata Shift }\textbf{\emph{R}}\emph{eduction
via }\textbf{\emph{O}}\emph{ptimal }\textbf{\emph{T}}\emph{ransport}
(LDROT) which can mitigate the data and label shifts simultaneously.
Finally, we conduct comprehensive experiments to verify our theoretical
findings and compare LDROT with state-of-the-art baselines. \vspace{-2mm}
\end{abstract}
\end{center}
\section{Introduction}

The remarkable success of deep learning can be largely attributed
to computational power advancement and large-scale annotated datasets.
However, in many real-world applications such as medicine and autonomous
driving, labeling a sufficient amount of high-quality data to train
accurate deep models is often prohibitively labor-expensive, error-prone,
and time-consuming. Domain adaptation (DA) or transfer learning has
emerged as a vital solution for this issue by transferring knowledge
from a label-rich domain (a.k.a. source domain) to a label-scarce
domain (a.k.a. target domain). Along with practical DA methods \cite{Ganin2015,TzengHDS15,long2015,shu2018a,french2018selfensembling}
which have achieved impressive performance on real-world datasets,
the theoretical results \cite{Mansour2009,Ben-David:2010,redko2017theoretical,zhang19_theory,cortes2019adaptation}
are abundant to provide rigorous and insightful understanding of various
aspects of transfer learning.

For domain adaptation, the source domain consists of the data distribution
$\mathbb{P}^{S}$ with the density $p^{S}$, and the unknown ground-truth
labeling function $f^{S}$ assigning label $y$ to source data $x$,
whilst these are $\mathbb{P}^{T}$, $p^{T}$, and $f^{T}$ for the
target domain, respectively. Moreover, while the data shift can be
characterized as a divergence between $\mathbb{P}^{S}$ and $\mathbb{P}^{T}$
\cite{Mansour2009,Ben-David:2010,redko2017theoretical,zhang19_theory,cortes2019adaptation},
the label shift in these works is commonly characterized as $\mathbb{E}_{\mathbb{P}^{T}}\left[\left|f^{S}\left(\bx\right)-f^{T}\left(\bx\right)\right|\right]$
or $\mathbb{E}_{\mathbb{P}^{S}}\left[\left|f^{S}\left(\bx\right)-f^{T}\left(\bx\right)\right|\right]$
in which the binary classification with deterministic labeling functions
$f^{S}\left(\cdot\right),f^{T}\left(\cdot\right)\in\left\{ 0,1\right\} $
was examined. Additionally, although this label shift term has occurred
in the theoretical analysis of \cite{Mansour2009,Ben-David:2010,redko2017theoretical,zhang19_theory,cortes2019adaptation},
it is restricted in considering the shift between $f^{S}\left(\bx\right)$
and $f^{T}\left(\bx\right)$ at the same data $\bx$, which ignores
the data shift between $\mathbb{P}^{S}$ and $\mathbb{P}^{T}$. This
limitation is illustrated in Figure \ref{fig:motivation_LS}. In particular,
for a white/square point $\bx$ drawn from the target domain as in
$\mathbb{E}_{\mathbb{P}^{T}}\left[\left|f^{S}\left(\bx\right)-f^{T}\left(\bx\right)\right|\right]$,
the source labeling function $f^{S}$ cannot give reasonable prediction
probabilities for $\bx$, hence leading to inaccurate $\left|f^{S}\left(\bx\right)-f^{T}\left(\bx\right)\right|$. 

Label shift has also been examined in an anti-causal setting \cite{pmlr-v80-lipton18a,unified_labelshift_NEURIPS2020},
wherein an intervention on \emph{$p(y)$} induces the shift, but the
process generating $\bx$ given $y$ is fixed, i.e., $p^{S}\left(\bx\mid y\right)=p^{T}\left(\bx\mid y\right)$.
Although this setting is useful in some specific scenarios (e.g.,
a diagnostic problem in which diseases cause symptoms), it is not
sufficiently powerful to cope with a general DA setting. Particularly,
in an anti-causal setting, the source and target data distributions
(i.e., $p^{S}\left(\bx\right)$ and $p^{T}\left(\bx\right)$) are
just simply two different mixtures of the class conditional distributions
$p^{S}\left(\bx\mid y\right)=p^{T}\left(\bx\mid y\right)$, hence
sharing the same support set. This is certainly far from a general
DA setting in which both \emph{data shift}: $p^{S}\left(\bx\right)\neq p^{T}\left(\bx\right)$
with arbitrarily separated support sets and \emph{non-covariate shift}:
$p^{S}\left(y\mid\bx\right)\neq p^{T}\left(y\mid\bx\right)$ appear. 

\textbf{Contribution.} In this paper, we study the label shift for
a general domain adaptation setting in which we have both \emph{data
shift}: $p^{S}\left(\bx\right)\neq p^{T}\left(\bx\right)$ with arbitrarily
separated support sets and \emph{non-covariate shift}: $p^{S}\left(y\mid\bx\right)\neq p^{T}\left(y\mid\bx\right)$.
More specifically, our developed label shift is applicable to a general
DA setting with a data shift between two domains and two totally
different labeling functions (i.e., we cannot use $f^{S}$ to predict
accurately target examples and vice versa). To define the label shift
between two given domains, we utilize transformation $L$ to transport
the target to the source data distributions (i.e., $L\#\mathbb{P}^{T}=\mathbb{P}^{S}$).
This transformation allows us to align the data of two domains. Subsequently,
the label shift between two domains is defined as the infimum of the
label shift induced by such a transformation with respect to all feasible
transformations. This viewpoint of label shift has a connection to
optimal transport~\cite{santambrogio2015optimal,villani2008optimal,peyre2019computational},
which enables us to develop theory to quantify the label shift for
various DA settings, e.g., anti-causal, closed-set, partial-set, open-set,
and universal settings. Overall, our contributions can be summarized
as follows: 

\textbf{1.} We characterize the label shift for a general DA setting
via optimal transport. From that, we develop a theory to estimate
the label shift for various DA settings and study the trade-off of
learning domain-invariant representations and WS label shift. 

\textbf{2.} Inspired from the theoretical development, we propose
\textbf{\emph{L}}\emph{abel and }\textbf{\emph{D}}\emph{ata Shift
}\textbf{\emph{R}}\emph{eductions via }\textbf{\emph{O}}\emph{ptimal
}\textbf{\emph{T}}\emph{ransport} (LDROT) which aims to mitigate both
data and label shifts. We conduct comprehensive experiments to verify
our theoretical findings and compare the proposed LDROT with the baselines
to demonstrate the favorable performance of our method. 

\textbf{Related works.} Several attempts have been proposed to characterize
the gap between general losses of source and target domains in DA,
notably \cite{Mansour2009,Ben-David:2010,redko2017theoretical,zhang19_theory,cortes2019adaptation}.
\cite{ben2012b,ben2012c,zhang19_theory} study the impossibility
theorems for DA, attempting to characterize the conditions under which
it is nearly impossible to perform transferability between domains.
PAC-Bayesian view on DA using weighted majority vote learning has
been rigorously studied in \cite{ger13,ger16}. Meanwhile, \cite{zhao2019learning,pmlr-v89-johansson19a}
interestingly indicate the insufficiency of learning domain-invariant
representation for successful adaptation. Specifically, \cite{zhao2019learning}
points out the degradation in target predictive performance if forcing
domain invariant representations to be learned while two marginal
label distributions of the source and target domains are overly divergent.
\cite{pmlr-v89-johansson19a} analyzes the information loss of non-invertible
transformations and proposes a generalization upper bound that directly
takes it into account. \cite{pmlr-v139-le21a} employed a transformation
to align two domains and developed theories based on this assumption.
Moreover, label shift has been examined for the anti-causal setting
\cite{pmlr-v80-lipton18a,unified_labelshift_NEURIPS2020}, which
seems not sufficiently realistic for a general DA setting. Optimal
transport theory has been theoretically leveraged with domain adaptation
\cite{courty2017joint}. We compare our proposed LDROT to DeepJDOT
\cite{damodaran2018deepjdot} (a deep DA approach based on the theory
of \cite{courty2017joint}), and other OT-based DDA approaches, including
SWD \cite{chenyu2019swd}, DASPOT \cite{yujia2019onscalable}, ETD
\cite{li2020enhanceOT}, RWOT \cite{xu2020reliable}. Finally, in
\cite{tachet2020domain} , a generator $g$ is said to produce generalized
label shift (GLS) representations if it transports source class conditional
distributions to corresponding target ones. Further theories were
developed to indicate that GLR representations are satisfied if we
enforce clustering structure assumption assisting us in training a
perfect classifier. Evidently, our work which focuses on how to quantify
the label shift between two different domains taking into account
the inherent data shift via optimal transport theory is totally different
form that work in terms of motivation and developed theory.

\section{Label Shift with Wasserstein Distance}

\subsection{Preliminaries }

\textbf{Notation.} For a positive integer $n$ and a real number $p\in[1,\infty)$,
$[n]$ indicates the set $\{1,2,\ldots,n\}$ while $\|x\|_{p}$ denotes
the $l_{p}$-norm of a vector $x\in\mathbb{R}^{n}$. Let $\mathcal{Y}^{S}$
and $\mathcal{Y}^{T}$ be the label sets of the source and target
domains that have $M^{S}:=\left|\mathcal{Y}^{S}\right|$ and $M^{T}:=\left|\mathcal{Y}^{T}\right|$
elements, respectively. Meanwhile, $\mathcal{Y}=\mathcal{Y}^{S}\cup\mathcal{Y}^{T}$
stands for the label set of both domains which has the cardinality
of $M:=\left|\mathcal{Y}\right|$. Subsequently, we denote $\mathcal{Y}_{\Delta}$,
$\mathcal{Y}_{\Delta}^{S}$, and $\mathcal{Y}_{\Delta}^{T}$ as the
simplices corresponding to $\mathcal{Y},\mathcal{Y}^{S}$, and $\mathcal{Y}^{T}$
respectively. Finally, let $f^{S}\left(\cdot\right)\in\mathcal{Y}_{\simplex}$
and $f^{T}\left(\cdot\right)\in\mathcal{Y}_{\simplex}$ be the labeling
functions of the source and target domains, respectively, by filling
zeros for the missing labels.

We now examine a general supervised learning setting. Consider a hypothesis
$h\left(\cdot\right)\in\mathcal{Y}_{\triangle}$ in a hypothesis class
$\mathcal{H}$ and a labeling function $f\left(\cdot\right)\in\mathcal{Y}_{\triangle}$
where $\mathcal{Y}_{\triangle}:=\left\{ \bpi\in\mathbb{R}^{M}:\norm{\bpi}_{1}=1\,\text{and \ensuremath{\bpi\geq\bzero}}\right\} $.
Let $d_{Y}$ be a metric over $\mathcal{Y}_{\triangle}$, we further
define a general loss of the hypothesis $h$ with respect to the labeling
function $f$ and the data distribution $\mathbb{P}$ as: $\mathcal{L}\left(h,f,\mathbb{P}\right):=\int d_{Y}\left(h\left(\bx\right),f\left(\bx\right)\right)\mathrm{d}\mathbb{P}\left(\bx\right)$.

Next, we consider a domain adaptation setting in which we have source
space $\mathcal{X}^{S}$ endowed with a distribution $\mathbb{P}^{S}$
and the density function $p^{S}\left(\bx\right)$, and a target space
$\mathcal{X}^{T}$ endowed with a distribution $\mathbb{P}^{T}$ and
the density function $p^{T}\left(\bx\right)$. We examine various
DA settings based on the labels of source and target domains including
(1) \emph{closed-set DA}: $\mathcal{Y}^{S}=\mathcal{Y}^{T}$, (2)
\emph{open-set DA}: $\mathcal{Y}^{S}\subset\mathcal{Y}^{T}$, (3)
\emph{partial-set DA}: $\mathcal{Y}^{T}\subset\mathcal{Y}^{S}$, and
(4) \emph{universal DA}: $\mathcal{Y}^{S}\subsetneq\mathcal{Y}^{T}$
and $\mathcal{Y}^{T}\subsetneq\mathcal{Y}^{S}$. 
\begin{figure}[!t]
\centering{}
\includegraphics[width=0.8\textwidth]{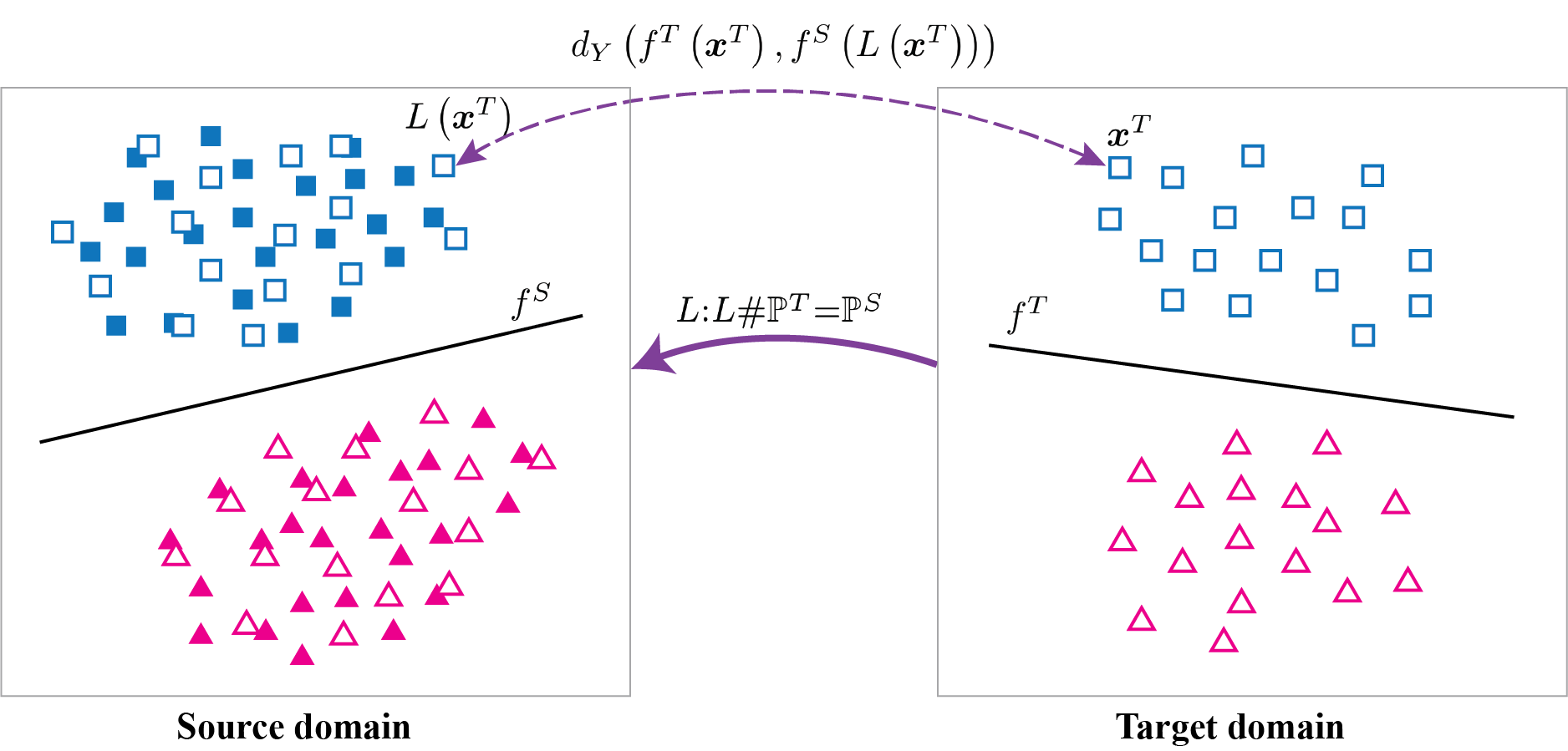}
\caption{An illustration of our label shift definition. We employ a transformation
$L:L\#\mathbb{P}^{T}=\mathbb{P}^{S}$ to align two domains. The white/square
points $L\left(\protect\bx^{T}\right)$ on the source domain correspond
to the white/square points $\protect\bx^{T}$ on the target domain.
We measure $d_{Y}\left(d_{Y}\left(f^{T}\left(\protect\bx^{T}\right),f^{S}\left(L\left(\protect\bx^{T}\right)\right)\right)\right)$
and define the label shift w.r.t. $L$ as $LS\left(S,T;L\right):=\mathbb{E}_{\mathbb{P}^{T}}\left[d_{Y}\left(f^{T}\left(\protect\bx^{T}\right),f^{S}\left(L\left(\protect\bx^{T}\right)\right)\right)\right]$.
Finally, we take infimum over all valid $L$ to define the label shift
between two domains.\label{fig:motivation_LS}}
\end{figure}

\subsection{Background on label shift}

Together with data shift, the study of label shift is important for
a general DA problem. However, due to the occurrence of data shift,
it is challenging to formulate label shift in a general DA setting.
Recent works \cite{pmlr-v80-lipton18a,unified_labelshift_NEURIPS2020}
have studied label shift for the anti-causal setting in which an intervention
on \emph{$p(y)$} induces the shift, but the process generating $\bx$
given $y$ is fixed, i.e., $p^{S}\left(\bx\mid y\right)=p^{T}\left(\bx\mid y\right)$.
In spite of being useful in some specific cases, the anti-causal setting
is restricted and cannot represent data shift broadly because the
source data distribution $p^{S}\left(\bx\right)$ and the target data
distribution $p^{T}\left(\bx\right)$ are simply just two different
mixtures of identical class conditional distributions. Furthermore,
the label shift framework from these works is non-trivial to generalize
to all settings of DA.

In this paper, we address the issues of the previous works by defining
a novel label shift framework via Wasserstein (WS) distance for a
general DA setting that takes into account the data shift between
two domains. We then develop a theory for our proposed label shift
based on useful properties of WS distance, such as its horizontal
view, numeric stability, and continuity \cite{pmlr-v70-arjovsky17a}.
We refer readers to \cite{santambrogio2015optimal,villani2008optimal,peyre2019computational}
for a comprehensive knowledge body of optimal transport theory and
WS distance, and Appendix \ref{sec:Notations} for necessary backgrounds
of WS for this work. 

\subsection{Label Shift via Wasserstein Distance \label{sec:label_shift_optimal_transport}
}

To facilitate our ensuing discussion, we assume that the source (S)
and target (T) distributions $\mathbb{P}^{S}$ and $\mathbb{P}^{T}$
are atomless distributions on Polish spaces. Therefore, there exists
a transformation $L:\mathcal{X}^{T}\goto\mathcal{X}^{S}$ such that
$L\#\mathbb{P}^{T}=\mathbb{P}^{S}$ \cite{villani2008optimal}. Given
that mapping $L$, a data example $\bx^{T}\sim\mathbb{P}^{T}$ with
the ground-truth prediction probability $f^{T}\left(\bx^{T}\right)$
corresponds to another data example $\bx^{S}=L\left(\bx^{T}\right)\sim\mathbb{P}^{S}$
with the ground-truth prediction probability $f^{S}\left(\bx^{S}\right)$.
Hence, it induces a label mismatch loss 
\[
d_{Y}\left(f^{T}\left(\bx^{T}\right),f^{S}\left(\bx^{S}\right)\right)=d_{Y}\left(f^{T}\left(\bx^{T}\right),f^{S}\left(L\left(\bx^{T}\right)\right)\right),
\]
where $d_{Y}$ is a given metric over $\mathcal{Y}_{\Delta}$. Based
on that concept, the label shift between the source and target domains
induced by the transformation $L$ can be defined as 
\begin{align}
LS\left(S,T;L\right):=\mathbb{E}_{\mathbb{P}^{T}}\left[d_{Y}\left(f^{T}\left(\bx^{T}\right),f^{S}\left(L\left(\bx^{T}\right)\right)\right)\right].
\end{align}
By finding the optimal mapping $L$, the label shift between two domains
is defined as follows.
\begin{defn}
\label{def:label_shift} Let $d_{Y}$ be a metric over the simplex
$\mathcal{Y}_{\Delta}$. The label shift between the source and the
target domains is defined as the infimum of the label shift induced
by all valid transformations $L$: {\footnotesize{}
\begin{equation}
LS\left(S,T\right):=\inf_{L:L\#\mathbb{P}^{T}=\mathbb{P}^{S}}LS\left(S,T;L\right)=\inf_{L:L\#\mathbb{P}^{T}=\mathbb{P}^{S}}\mathbb{E}_{\mathbb{P}^{T}}\left[d_{Y}\left(f^{T}\left(\bx^{T}\right),f^{S}\left(L\left(\bx^{T}\right)\right)\right)\right].\label{eq:ls_def}
\end{equation}
} We give an illustration for Definition~\ref{def:label_shift} in
Figure \ref{fig:motivation_LS}. The label shift in Eq.~(\ref{eq:ls_def})
suggests finding the optimal transformation $L^{*}$ to optimally
align the source and target domains with a minimal label mismatch.
\end{defn}

\textbf{Properties of label shift via Wasserstein distance:} To show
the connection between the aforementioned label shift and optimal
transport, we introduce two ways of calculating the label shift via
Wasserstein distance.
\begin{prop}
\label{thm:label_shift} (i) Denote by $\mathbb{P}_{f^{S}}^{S}$ the
joint distribution of $\left(\bx,f^{S}\left(\bx\right)\right)$, where
$\bx\sim\mathbb{P}^{S}$, and $\mathbb{P}_{f^{T}}^{T}$ the joint
distribution of $\left(\bx,f^{T}\left(\bx\right)\right)$, where $\bx\sim\mathbb{P}^{T}$.
Then, we have: $LS\left(S,T\right)=\mathcal{W}_{d_{Y}}\left(\mathbb{P}_{f^{S}}^{S},\mathbb{P}_{f^{T}}^{T}\right).$

 (ii) Let $\mathbb{P}_{f^{S}}$ and $\mathbb{P}_{f^{T}}$ be the push-forward
measures of $\mathbb{P}^{S}$ and $\mathbb{P}^{T}$ via $f^{S}$ and
$f^{T}$ respectively, i.e., $\mathbb{P}_{f^{S}}=f^{S}\#\mathbb{P}^{S}$
and $\mathbb{P}_{f^{T}}=f^{T}\#\mathbb{P}^{T}$. Then, we have: $LS\left(S,T\right)=\mathcal{W}_{d_{Y}}\left(\mathbb{P}_{f^{S}},\mathbb{P}_{f^{T}}\right).$
\end{prop}

The results of Proposition~\ref{thm:label_shift} indicate that
we can compute the label shift via the Wasserstein distance on the
simplex. For example, when $d_{Y}(y,y')=\|y-y'\|_{p}^{p}$, the label
shift can be computed via the familiar $\mathcal{W}_{p}$ distance
between $\mathbb{P}_{f^{S}}$ and $\mathbb{P}_{f^{T}}$, i.e., $LS(S,T)=\mathcal{W}_{p}^{p}(\mathbb{P}_{f^{S}},\mathbb{P}_{f^{T}})$.
Note that $\mathcal{W}_{d}\left(\mathbb{P}_{f^{S}}^{S},\mathbb{P}_{f^{T}}^{T}\right)$
with $d=\lambda d_{X}+d_{Y}$ was studied in \cite{courty2017joint}
for proposing a DA method that can mitigate both label and data shifts.
However, the concept label shift was not characterized and defined  explicitly
in that work. Moreover, our motivation and theory development in this
work are different and independent from \cite{courty2017joint}.

In addition, \cite{pmlr-v139-le21a} uses a transformation $L$ with
the aim to reduce the data shift between the source and target domains. The variance of general losses of a source classifier when predicting on the source domain and that of the corresponding target classifier when predicting on the target domain is inspected, which introduces the label shift via this transformation (cf. Theorem 1 in that paper). However, the definition of the label shift is totally dependent on the transformation $L$. It is worth noting that in this work, we do not consider an arbitrary transformation $L$ as in \cite{pmlr-v139-le21a}. Instead, we look in $L:L\#\mathbb{P}^{T}=\mathbb{P}^{S}$ and define the label shift as the infimum of the label shifts induced by \textit{valid} transformations. To give a better understanding of our label shift definition, we now present some bounds for it in general and specific cases.
\begin{prop}
\label{proposition:properties_labelshift} Denote by $p_{Y}^{S}=(p_{Y}^{S}(y))_{y=1}^{M}$
and $p_{Y}^{T}=(p_{Y}^{T}(y))_{y=1}^{M}$ the marginal distributions
of the source and target domain labels, i.e., $p_{Y}^{S}(y)=\int_{\mathcal{X}^{S}}p^{S}(x,y)dx$
and $p_{Y}^{T}(y)=\int_{\mathcal{X}^{T}}p^{T}(x,y)dx$. For $d_{Y}(y,y')=\|y-y'\|_{p}^{p}$
when $p\geq1$, the following holds:

(i) $\mathcal{L}\left(h^{T},f^{T},\mathbb{P}^{T}\right)\leq LS\left(S,T\right)+\mathcal{L}\left(h^{S},f^{S},\mathbb{P}^{S}\right)+\mathcal{W}_{d_{Y}}\left(\mathbb{P}_{h^{S}}^{S},\mathbb{P}_{h^{T}}^{T}\right)+const$
where the constant can be viewed as a reconstruction term: $\sup_{L,K:L\#\mathbb{P}^{T}=\mathbb{P}^{S},K\#\mathbb{P}^{S}=\mathbb{P}^{T}}\mathbb{E}_{\mathbb{P}^{T}}\left[d_{Y}\left(f^{T}\left(K\left(L\left(\bx\right)\right)\right),f^{T}\left(\bx\right)\right)\right]$; 

(ii) $LS(S,T)\geq\|p_{Y}^{S}-p_{Y}^{T}\|_{p}^{p}$;

(iii) In the setting that $\mathbb{P}^{S}$ and $\mathbb{P}^{T}$
are mixtures of well-separated Gaussian distributions, i.e., $p^{a}(\bx)=\sum_{y=1}^{M}p^{a}(y)\mathcal{N}(\bx|\mu_{y}^{a},\Sigma_{y}^{a})$
with $\|\mu_{y}^{a}-\mu_{y'}^{a}\|_{2}\geq D\times\max\{\|\Sigma_{y}^{a}\|_{op}^{1/2},\|\Sigma_{y'}^{a}\|_{op}^{1/2}\}\,\forall\,a\in\{S,T\},y\neq y'$,
in which $\|\cdot\|_{op}$ denotes the operator norm and $D$ is sufficiently
large, we have 
\begin{equation}
\left|LS(S,T)-\mathcal{W}_{p}^{p}(\mathbb{P}_{Y}^{S},\mathbb{P}_{Y}^{T})\right|\leq\epsilon(D),
\end{equation}
where $\epsilon(D)$ is a small constant depending on $D,p_{Y}^{S},p_{Y}^{T},(\Sigma_{y}^{S},\Sigma_{y}^{T})_{y=1}^{M}$,
and it goes to 0 as $D\to\infty$. 
\end{prop}

A few comments on Proposition~\ref{proposition:properties_labelshift}
are in order. The\emph{ inequality in (i) }bounds the target loss
by the source loss and the label shift. Though this inequality has
the same form as those in \cite{Mansour2009,Ben-David:2010,redko2017theoretical,zhang19_theory,cortes2019adaptation},
the label shift in our inequality is more reasonably expressed. The
\emph{inequality in (ii)} reads that the \emph{marginal label shift}
$\|p_{Y}^{S}-p_{Y}^{T}\|_{p}^{p}$ is a lower bound of our label shift.
Therefore, the label shift induced by the best transformation $L^{*}$
can not be less than this quantity. A direct consequence is that $LS(S,T)=0$
implies $p^{S}(y)=p^{T}(y)$, for all $y\in[M]$ (no label shift).
Finally, the \emph{inequality in (iii)} shows that when the classes
are well-separated, the label shift will almost achieve the lower
bound in the first inequality, which implies its tightness. The key
step in the proof is proving that $\mathbb{P}_{f^{S}}$ and $\mathbb{P}_{f^{T}}$
will be concentrated around the vertices and it is also provable for
sub-Gaussian distributions with some extra work. The bound gives a
simple way to estimate the label shift in this scenario: instead of
measuring the Wasserstein distance $\mathcal{W}(\mathbb{P}_{f^{S}},\mathbb{P}_{f^{T}})$
on the simplex $\mathcal{Y}_{\Delta}$, we only need to measure the
Wasserstein distance between the vertices equipped with the masses
$p_{Y}^{S}$ and $p_{Y}^{T}$. The first experiment in Section \ref{subsec:label_shift_estimate}
also supports this finding.

Our label shift formulation can also serve as a tool to elaborate
other aspects of DA, as we will see below.

\textbf{Minimizing data shift while ignoring label shift can hurt
the prediction on test set:} Consider two classifiers on the source
and target domains $h^{S}=h\circ g^{S}$ and $h^{T}=h\circ g^{T}$
where $g^{S}:\mathcal{X}^{S}\goto\mathcal{Z}$, $g^{T}:\mathcal{X}^{T}\goto\mathcal{Z}$
are the source and target feature extractors, and $h:\mathcal{Z}\goto\mathcal{Y}_{\simplex}$
with $h\in\mathcal{H}$. We define a new metric $d_{Z}$ with respect
to the family $\mathcal{H}$ as follows: 
\[
d_{Z}\left(\bz_{1},\bz_{2}\right)=\sup_{h\in\mathcal{H}}d_{Y}\left(h\left(\bz_{1}\right),h\left(\bz_{2}\right)\right),
\]
where $\bz_{1}$ and $\bz_{2}$ lie on the latent space $\mathcal{Z}$.
The necessary (also sufficient) condition under which $d_{Z}$ is
a proper metric on the latent space (see the proof in Appendix \ref{sec:proof_key_results})
is realistic and not hard to be satisfied (e.g., the family $\mathcal{H}$
contains any bijection). We now can define a Wasserstein distance
$W_{d_{Z}}$ that will be used in the development of Theorem~\ref{thm:bound_py_diff}. 
\begin{thm}
\label{thm:bound_py_diff} With regard to the latent space $\mathcal{Z}$,
we can upper-bound the label shift as 
\begin{align*}
LS\left(S,T\right)\leq\mathcal{L}\left(h^{S},f^{S},\mathbb{P}^{S}\right)+\mathcal{L}\left(h^{T},f^{T},\mathbb{P}^{T}\right)+\mathcal{W}_{d_{Z}}\left(g^{S}\#\mathbb{P}^{S},g^{T}\#\mathbb{P}^{T}\right).
\end{align*}
\end{thm}

Theorem \ref{thm:bound_py_diff} indicates a trade-off of learning
domain-invariant representation by forcing $g^{S}\#\mathbb{P}^{S}=g^{T}\#\mathbb{P}^{T}$
(e.g., $\min\,\mathcal{W}_{d_{Z}}\left(g^{S}\#\mathbb{P}^{S},g^{T}\#\mathbb{P}^{T}\right)$).
It is evident that if the label shift between domains is significant,
because $\mathcal{L}\left(h^{S},f^{S},\mathbb{P}^{S}\right)$ can
be trained to be sufficiently small, learning domain-invariant representation
by minimizing $\mathcal{W}_{d_{Z}}\left(g^{S}\#\mathbb{P}^{S},g^{T}\#\mathbb{P}^{T}\right)$
leads to a hurt in the performance of the target classifier $h^{T}$
on the target domain. Similar theoretical result was discovered in
\cite{zhao2019learning} for the binary classification (see Theorem
4.9 in that paper). However, our theory is developed based on our
label shift formulation in a more general multi-class classification
setting and uses WS distance rather than Jensen-Shannon (JS) distance
\cite{JS_distance} as in \cite{zhao2019learning} for which the advantages of WS distance over JS distance
have been thoughtfully discussed in \cite{pmlr-v70-arjovsky17a}.

\textbf{Label shift under different settings of DA:} An advantage
of our method is that it can measure the label shift under various
DA settings (i.e., open-set, partial-set, and universal DA). Doing
this task is not straight-forward using other label shift methods.
For example, if there is a label that appears in a domain but not
in the other, it is not meaningful to measure the ratio between the
marginal distribution of this label as in~\cite{NEURIPS2020_219e0524}. 

In what follows, we elaborate our label shift in those settings. Particularly,
we provide some lower bounds for it, implying that the label shift
is higher if there is label mismatch between two domains. Recall that
when the source and target domains do not have the same number of
labels, we can extend $f^{S}$ and $f^{T}$ to be functions taking
values on $\mathcal{Y}_{\Delta}$ by filling $0$ for the missing
labels. For the sake of presenting the results, let $\mathcal{Y}_{S}\cap\mathcal{Y}_{T}=\{1,\dots,C\}$
be the common labels of two domains, $\mathbb{Q}_{S}$ the marginal
of $\mathbb{P}_{f_{S}}$ and $\mathbb{Q}_{T}$ the marginal of $\mathbb{P}_{f_{T}}$
on the first $(C-1)$ dimensions. $\mathbb{Q}_{S\setminus T}$ denotes
the marginal of $\mathbb{P}_{f_{S}}$ in the space of variables having
labels $\mathcal{Y}_{S}\setminus\mathcal{Y}_{T}$ and $\mathbb{Q}_{T\setminus S}$
denotes the marginal of $\mathbb{P}_{f_{T}}$ in the space of variables
having labels $\mathcal{Y}_{T}\setminus\mathcal{Y}_{S}$.
\begin{thm}
\label{thm:label_shift_setting}Assume that $d_{Y}(y,y')=\|y-y'\|_{p}^{p}$.
Then, the following holds:

i) For the partial-set setting (i.e., $\mathcal{Y}^{T}\subset\mathcal{Y}^{S}$),
we obtain 
\begin{equation}
LS(S,T)\geq W_{p}^{p}(\mathbb{Q}_{S},\mathbb{Q}_{T})+\mathbb{E}_{X\sim\mathbb{Q}_{S\setminus T}}\left[\norm X_{p}^{p}\right].\label{eq:partial_bound}
\end{equation}

ii) For the open-set setting (i.e., $\mathcal{Y}^{S}\subset\mathcal{Y}^{T}$),
we obtain 
\begin{equation}
LS(S,T)\geq W_{p}^{p}(\mathbb{Q}_{S},\mathbb{Q}_{T})+\mathbb{E}_{X\sim\mathbb{Q}_{T\setminus S}}\left[\norm X_{p}^{p}\right].\label{eq:open_set_bound}
\end{equation}

iii) For the universal setting (i.e., $\mathcal{Y}^{T}\subsetneq\mathcal{Y}^{S}$
and $\mathcal{Y}^{S}\subsetneq\mathcal{Y}^{T}$), we have 
\begin{equation}
LS(S,T)\geq W_{p}^{p}(\mathbb{Q}_{S},\mathbb{Q}_{T})+\mathbb{E}_{X\sim\mathbb{Q}_{T\setminus S}}\left[\norm X_{p}^{p}\right]+\mathbb{E}_{Y\sim\mathbb{Q}_{S\setminus T}}\left[\norm Y_{p}^{p}\right].\label{eq:universal_bound}
\end{equation}
\end{thm}

 Theorem \ref{thm:label_shift_setting} reveals that the label shifts
for the partial-set, open-set, or universal DA settings are higher
than the vanilla closed-set setting due to the missing and unmatching
labels of two domains (see Section \ref{subsec:label_shift_estimate}). Our theory also implicitly
indicates that by setting appropriate weights for source and target
examples (i.e., low weights for the examples with missing and unmatching
labels), we can deduct label shift by reducing the second term of
lower-bounds in Eqs.~\eqref{eq:partial_bound}, \eqref{eq:open_set_bound},
and \eqref{eq:universal_bound} to mitigate the negative transfer
\cite{cao2019_partial}. We leave this interesting investigation
to our future work. Moreover, further analysis can be found in Appendix
\ref{sec:proof_remaining_results}.

\section{Label and Data Shift Reductions via Optimal Transport}
\subsection{Motivation theorem for our approach}
We consider the source classifier $h^{S}=\bar{h}^{S}\circ g$ and the target classifier $h^{T}=\bar{h}^{T}\circ g$. Given a decreasing function $\phi:\mathbb{R}\to [0,1]$, a classifier $h$ is said to be $\phi$-Lipschitz transferable \cite{courty2017joint} w.r.t a joint distribution $\gamma\in\Gamma(\mathbb{P}^{S},\mathbb{P}^{T})$, the metric $d_X$ on the data space, and the metric $d_Y$ on the $\mathcal{Y}_{\Delta}$ if for all $\delta>0$, we have
\begin{align*}
    \mathbb{P}_{(\bx^S,\bx^T)\sim\gamma}\Big[d_Y\big(h(\bx^S),h(\bx^T)\big)>\delta d_X(\bx^S,\bx^T)\Big]\leq \phi(\delta).
\end{align*}
\begin{thm}\label{theorem:motivation}
Assume that $\mathcal{X}^{T}=\mathcal{X}^{S}=\mathcal{X}$, the target classifier $h^T$ is $\phi$-Lipschitz transferable w.r.t the optimal joint distribution $\gamma^*\in\Gamma(\mathbb{P}^{S},\mathbb{P}^{T})$ of $W_{d_X}(\mathbb{P}^{S},\mathbb{P}^{T})$ with the metric $d_X$ on the data space $\mathcal{X}$ and $d_Y$ on the simplex $\mathcal{Y}_{\Delta}$, we have
\begin{align}\label{eq:motivation_theorem}
    LS(h^T,f^T,\mathbb{P}^{T})&\leq LS(S,T)+\mathcal{L}(h^S,f^S,\mathbb{P}^S)+\frac{1}{2}\delta W_{d_X}(\mathbb{P}^{S},\mathbb{P}^{T})\nonumber\\
    &\qquad\qquad+\frac{1}{2}W_{d_Y}(\mathbb{P}^{S}_{h^S},\mathbb{P}^{T}_{h^T})+\frac{[1+\phi(\delta)]N}{2},
\end{align}
where $N=\sup_{\pi,\pi^\prime\in\mathcal{Y}_{\Delta}}d_Y(\pi,\pi^\prime)$ and the label shift $LS(h^T,f^T,\mathbb{P}^T)$ is defined as 
\begin{align*}
    \inf_{L:L\#\mathbb{P}^T=\mathbb{P}^T}\mathbb{E}_{\mathbb{P}^T}\Big[d_Y\big(h^T(\bx),f^T(L(\bx))\big)\Big],
\end{align*}
representing how accurate the target classifier imitates its ground-truth labeling function on the target domain.
\end{thm}
Theorem~\ref{theorem:motivation} hints us how to devise our \textbf{\emph{L}}\emph{abel and}\textbf{\emph{D}}\emph{ata Shift }\textbf{\emph{R}}\emph{eductions
via }\textbf{\emph{O}}\emph{ptimal }\textbf{\emph{T}}\emph{ransport}
(LDROT), aiming to reduce both data and label shifts simultaneously. First, we train the source classifier to work well on the source domain by minimizing $\mathcal{L}(h^S,f^S,\mathbb{P}^S)$. Second, we minimize the shifting loss $W_{d_Y}(\mathbb{P}^{S}_{h^S},\mathbb{P}^{T}_{h^T})$ to encourage $h^T$ to imitate $h^S$. We will later explain how minimizing this shifting loss helps to reduce both data and label shifts simultaneously. Third, we encourage the target classifier $h^T$ to atisfy the
Lipschitz condition or become smoother by making $\delta$ and $\phi(\delta)$ as small as possible.

\subsection{Objective function of LDROT}
Our LDROT consists of three losses which are as follows:

(i) \textit{Standard loss $\mathcal{L}^{S}$}: We train the source
classifier $h^{S}$ on the labeled source data by minimizing the loss
$\mathcal{L}^{S}:=\mathcal{L}\left(h^{S},f^{S},\mathbb{P}^{S}\right)$;

(ii) \textit{Shifting loss $\mathcal{L}^{shift}$}: Furthermore, to
mitigate both label and data shifts, we propose to further regularize
the loss $\mathcal{L}^{S}$ by $\mathcal{L}^{shift}:=\mathcal{W}_{d}\left(\mathbb{P}_{h^{S}}^{S},\mathbb{P}_{h^{T}}^{T}\right)$,
where the ground metric $d$ is defined as 
\begin{equation}
d\left(\bz^{S},\bz^{T}\right)=\lambda\cdot d_{X}\left(g\left(\bx^{S}\right),g\left(\bx^{T}\right)\right)+d_{Y}\left(h^{S}\left(\bx^{S}\right),h^{T}\left(\bx^{T}\right)\right),\label{eq:d_metric}
\end{equation}
where $\bz^{S}=\left(\bx^{S},h^{S}\left(\bx^{S}\right)\right)$ with
$\bx^{S}\sim\mathbb{P}^{S}$ and $\bz^{T}=\left(\bx^{T},h^{T}\left(\bx^{T}\right)\right)$
with $\bx^{T}\sim\mathbb{P}^{T}$;

(iii) \textit{Clustering loss $\mathcal{L}^{clus}$}: Finally, to \textit{enhance the smoothness} of $h^{T}$, we enforce the clustering
assumption \cite{chapelle2005semi} to enable $h^{T}$ giving the
same prediction for source and target examples on the same cluster.
To employ the clustering assumption \cite{shu2018a}, we use Virtual
Adversarial Training (VAT) \cite{VAT} in conjunction with minimizing
the entropy of prediction \cite{grandvalet2015entropy}: $\mathcal{L}^{clus}:=\mathcal{L}^{ent}+\mathcal{L}^{vat}$
with{\small{}
\begin{align*}
\mathcal{L}^{ent} & =\mathbb{E}_{\mathbb{P}^{T}}\left[\mathbb{H}\left(h^{T}\left(g\left(\bx\right)\right)\right)\right], \\
\mathcal{L}^{vat} & =  \mathbb{E}_{0.5\mathbb{P}^{S}+0.5\mathbb{P}^{T}}\left[\text{max }_{\bx'\in B_{\theta}\left(\bx\right)}D_{KL}\left(h^{T}\left(g\left(\bx\right)\right),h^{T}\left(g\left(\bx'\right)\right)\right)\right],
\end{align*}
} where $D_{KL}$ represents a Kullback-Leibler divergence, $\theta$
is a very small positive number, $B_{\theta}\left(\bx\right):=\left\{ \bx':\norm{\bx'-\bx}_{2}<\theta\right\} $,
and $\mathcal{\mathbb{H}}$ specifies the entropy.

Combining the above losses, we arrive at the following objective function
of LDROT. 
\begin{align}
\inf_{g,\bar{h}^{S},\bar{h}^{T}}\{\mathcal{L}^{S}+\alpha\mathcal{L}^{shift}+\beta\mathcal{L}^{clus}\},\label{eq:LDROT_loss}
\end{align}
where $\alpha,\beta>0$. In addition, we use the target classifier
$h^{T}$ to predict target examples.

\textbf{Remark on the shifting term:} We now explain why including
the shifting term $\mathcal{W}_{d}\left(\mathbb{P}_{h^{S}}^{S},\mathbb{P}_{h^{T}}^{T}\right)$
supports to reduce label and data shifts. First, we have the following
inequality whose proof can be found in Appendix \ref{sec:proof_remaining_results}:
\begin{equation}
\mathcal{W}_{d_{X}}\left(g\#\mathbb{P}^{S},g\#\mathbb{P}^{T}\right)=\mathcal{W}_{d_{X}}\left(\mathbb{P}_{h^{S}}^{S},\mathbb{P}_{h^{T}}^{T}\right)\leq\mathcal{W}_{d}\left(\mathbb{P}_{h^{S}}^{S},\mathbb{P}_{h^{T}}^{T}\right)\label{eq:data_shift_bound}
\end{equation}
since $d_{X}\leq d$. 

Therefore, by including the shifting term $\mathcal{W}_{d}\left(\mathbb{P}_{h^{S}}^{S},\mathbb{P}_{h^{T}}^{T}\right)$,
we aim to reduce $\mathcal{W}_{d_{X}}\left(g\#\mathbb{P}^{S},g\#\mathbb{P}^{T}\right)$,
which is useful for reducing the data shift on the latent space.

Second, we find that $\mathcal{W}_{d_{Y}}\left(\mathbb{P}_{h^{S}}^{S},\mathbb{P}_{h^{T}}^{T}\right)\leq\mathcal{W}_{d}\left(\mathbb{P}_{h^{S}}^{S},\mathbb{P}_{h^{T}}^{T}\right)$
since $d_{Y}\leq d$. The inequality (i) in Proposition \ref{proposition:properties_labelshift}
suggests that reducing the label shift $\mathcal{W}_{d_{Y}}\left(\mathbb{P}_{h^{S}}^{S},\mathbb{P}_{h^{T}}^{T}\right)$
helps to reduces the loss on the target domain, therefore increasing
the quality of our DA method.  Finally, by including $\mathcal{W}_{d}\left(\mathbb{P}_{h^{S}}^{S},\mathbb{P}_{h^{T}}^{T}\right)$,
we aim to simultaneously reduce both terms $\mathcal{W}_{d_{Y}}\left(\mathbb{P}_{h^{S}}^{S},\mathbb{P}_{h^{T}}^{T}\right)$
and $\mathcal{W}_{d_{X}}\left(\mathbb{P}_{h^{S}}^{S},\mathbb{P}_{h^{T}}^{T}\right)$,
which is equal to $\mathcal{W}_{d_{X}}\left(g\#\mathbb{P}^{S},g\#\mathbb{P}^{T}\right)$.
That step helps reduce the data shift between $g\#\mathbb{P}^{S}$and
$g\#\mathbb{P}^{T}$, while forcing $h^{T}$ to mimic $h^{S}$ for
predicting well on the target domain via reducing the label shift
$\mathcal{W}_{d_{Y}}\left(\mathbb{P}_{h^{S}}^{S},\mathbb{P}_{h^{T}}^{T}\right)$.

\subsection{Training procedure of LDROT}

We now discuss a few important aspects of the training procedure of
LDROT.

\vspace{0.5 em}
\noindent
\textbf{Similarity-aware version of the ground metric $d$:} The weight
$\lambda$ in the ground metric $d$ in Eq.~(\ref{eq:d_metric})
represents the matching extent of $g\left(\bx^{T}\right)$ and $g\left(\bx^{S}\right)$.
Ideally, we would like to replace this fixed constant $\lambda$ by
varied weights $\bw\left(\bx^{S},\bx^{T}\right)$ in such a way that
$\bw\left(\bx^{S},\bx^{T}\right)$ is high if $\bx^{T}$ and $\bx^{S}$
share the same label and low otherwise. However, it is not possible
because the label of $\bx^{T}$ is unknown. As an alternative, it
appears that if we can have a good way to estimate the pairwise similarity
$s\left(\bx^{S},\bx^{T}\right)$ of $\bx^{T}$ and $\bx^{S}$, $s\left(\bx^{S},\bx^{T}\right)$
seems to be high if $\bx^{T}$ and $\bx^{S}$ share the same label
and low if otherwise. Based on this observation, we instead propose
using a similarity-aware version of metric $d$ as follows: 
\[
\bar{d}\left(\bz^{S},\bz^{T}\right)=\bw\left(\bx^{S},\bx^{T}\right)d_{X}\left(g\left(\bx^{S}\right),g\left(\bx^{T}\right)\right)+d_{Y}\left(h^{S}\left(\bx^{S}\right),h^{T}\left(\bx^{T}\right)\right),
\]
where the weight $\bw\left(\bx^{S},\bx^{T}\right)$ is estimated based
on $s\left(\bx^{S},\bx^{T}\right)$. 

\vspace{0.5 em}
\noindent
\textbf{Entropic regularized version of shifting term:} Since computing
directly the shifting term $\mathcal{W}_{d}\left(\mathbb{P}_{h^{S}}^{S},\mathbb{P}_{h^{T}}^{T}\right)$
is expensive, we use the entropic regularized version of $\mathcal{W}_{d}\left(\mathbb{P}_{h^{S}}^{S},\mathbb{P}_{h^{T}}^{T}\right)$
instead, which we denote $\mathcal{W}_{d}^{\epsilon}\left(\mathbb{P}_{h^{S}}^{S},\mathbb{P}_{h^{T}}^{T}\right)$
where $\epsilon$ is a positive regularized term (Detailed definition
and discussion of the entropic regularized Wasserstein metric is in
Appendix A). The dual-form~\cite{stochastic_ws} of that entropic
regularized term with respect to the ground metric $\bar{d}$ admits
the following form:{\footnotesize{}
\begin{equation}
\max_{\phi}\Biggl\{\frac{1}{N_{S}}\sum_{i=1}^{N_{S}}\phi\left(g\left(\bx_{i}^{S}\right)\right)-\frac{\epsilon}{N_{T}}\sum_{j=1}^{N_{T}}\Biggl[\log\Biggl(\frac{1}{N_{S}}\sum_{i=1}^{N_{S}}\exp\Biggl\{\frac{\phi\left(g\left(\bx_{i}^{S}\right)\right)-\bar{d}\left(\bz_{i}^{S},\bz_{j}^{T}\right)}{\epsilon}\Biggr\}\Biggr)\Biggr]\Biggr\},\label{eq:entropic_WS-1}
\end{equation}
}where $\phi$ is a neural net named the Kantorovich potential network,
$\left\{ \left(\bx_{i}^{S},y_{i}^{S}\right)\right\} _{i=1}^{N_{S}}$and
$\left\{ \bx_{i}^{T}\right\} _{i=1}^{N_{T}}$ are source and target
data, $\bz_{i}^{S}=(\bx_{i}^{S},h^{S}(\bx_{i}^{S}))$, $\bz_{j}^{T}=(\bx_{j}^{T},h^{T}(\bx_{j}^{T}))$,
and $\bw_{ij}=\bw\left(\bx_{i}^{S},\bx_{j}^{T}\right)$.

\vspace{0.5 em}
\noindent
\textbf{Evaluating the weights $\bw_{ij}$:} The weights $\bw_{ij}$
are evaluated based on the similarity scores $s_{ij}:=s\left(\bx_{i}^{S},\bx_{j}^{T}\right)$.
Basically, we train from scratch or fine-tune a pre-trained deep net
(e.g., ResNet~\cite{he2016resnet}) using source dataset with labels
and compute cosine similarity of latent representations $\boldsymbol{r}_{j}^{T}$
and $\boldsymbol{r}_{i}^{S}$ of $\bx_{j}^{T}$ and $\bx_{i}^{S}$
as $s_{ij}=s\left(\bx_{i}^{S},\bx_{j}^{T}\right)=\text{cosine-sim}\left(\boldsymbol{r}_{i}^{S},\boldsymbol{r}_{j}^{T}\right).$

To ease the computation, we estimate the weights $\bw_{ij}$ according
to source and target batches. Specifically, we consider a balanced
source batch of $Mb$ source examples (i.e., $M$ is the number of
classes and $b$ is source batch size for each class). For a target
example $\bx_{j}^{T}$ in the target batch, we sort the similarity
array $\left[s_{ij}\right]_{i=1}^{Mb}$ in an ascending order. Ideally,
we expect that the similarity scores of the target example $\bx_{j}^{T}$
and the source examples $\bx_{i}^{S}$ with the same class as $\bx_{j}^{T}$
(i.e., totally we have $b$ of theirs) are higher than other similarity
scores in the current source batch. Therefore, we find $\mu_{i}$
as the $\frac{M-1}{M}$-percentile of the ascending similarity array
$\left[s_{ij}\right]_{i=1}^{Mb}$ and compute the weights as $\bw_{ij}=\exp\left\{ \frac{\bs_{ij}-\mu_{i}}{\tau}\right\} $
with a temperature variable $\tau$. It is worth noting that this
weight evaluation strategy assists us in \emph{sharpening} and \emph{contrasting}
the weights for the pairs in the similar and different classes. More
specifically, for the pairs in the same classes, $\bs_{ij}$ tend
to be bigger than $\mu_{i}$, whilst for the pairs in different classes,
$\bs_{ij}$ tend to be smaller than $\mu_{i}$. Hence, with the support
of exponential form and temperature variable $\tau$, $\bw_{ij}$
for the pairs in the same classes tend to be higher than those for
the pairs in different classes. More specifically, for the pairs in the same classes, $s_{ij}$ tend
to be bigger than $\mu_i$, whilst for the pairs in different classes,
$s_{ij}$ tend to be smaller than $\mu_i$. Hence, with the support of
exponential form and temperature variable $\tau$, $w_{ij}$ for the
pairs in the same classes tend to be higher than those for
the pairs in different classes.

\paragraph{Comparing to DeepJDOT \cite{damodaran2018deepjdot}: }

The $\mathcal{W}_{d}\left(\mathbb{P}_{h^{S}}^{S},\mathbb{P}_{h^{T}}^{T}\right)$
with $d=\lambda d_{X}+d_{Y}$ was investigated in DeepJDOT \cite{damodaran2018deepjdot}. However,
ours is different from that work in some aspects: (i) \emph{similarity
based dynamic weighting}, (ii) \emph{clustering loss} for enforcing
clustering assumption for target classifier, and (iii) \emph{entropic
dual form} for training rather than Sinkhorn as in \cite{damodaran2018deepjdot}. More analysis of  LDROT can be found in Section \ref{more_analysis}.

\begin{table}[t!]
\centering{} 
 \caption{Classification accuracy (\%) on Office-31 dataset for unsupervised
DA (ResNet-50).\label{tab:office31-dataset-1}}
\resizebox{0.6\columnwidth}{!}{\centering\setlength{\tabcolsep}{2pt}
\begin{tabular}{cccccccc}
\hline 
Method  & A$\rightarrow$W  & A$\rightarrow$D  & D$\rightarrow$W  & W$\rightarrow$D  & D$\rightarrow$A  & W$\rightarrow$A  & Avg\tabularnewline
\hline 
ResNet-50 \cite{he2016resnet}  & 70.0  & 65.5  & 96.1  & 99.3  & 62.8  & 60.5  & 75.7\tabularnewline
DeepCORAL \cite{sun2016coral}  & 83.0  & 71.5  & 97.9  & 98.0  & 63.7  & 64.5  & 79.8\tabularnewline
DANN \cite{ganin2016domain}  & 81.5  & 74.3  & 97.1  & 99.6  & 65.5  & 63.2  & 80.2\tabularnewline
ADDA \cite{tzeng2017ADDA}  & 86.2  & 78.8  & 96.8  & 99.1  & 69.5  & 68.5  & 83.2\tabularnewline
CDAN \cite{long2018cdan}  & 94.1  & 92.9  & 98.6  & \textbf{100.0}  & 71.0  & 69.3  & 87.7\tabularnewline
TPN \cite{pan2019tpn}  & 91.2  & 89.9  & 97.7  & 99.5  & 70.5  & 73.5  & 87.1\tabularnewline
SAFN \cite{xu2019sfan}  & 90.1  & 90.7  & 98.6  & 99.8  & 73.0  & 70.2  & 87.1\tabularnewline
rRevGrad+CAT \cite{deng2019cluster}  & 94.4  & 90.8  & 98.0  & \textbf{100.0}  & 72.2  & 70.2  & 87.6\tabularnewline
DeepJDOT \cite{damodaran2018deepjdot}  & 88.9  & 88.2  & 98.5  & 99.6  & 72.1  & 70.1  & 86.2\tabularnewline
ETD \cite{li2020enhanceOT}  & 92.1  & 88.0  & \textbf{100.0}  & \textbf{100.0}  & 71.0  & 67.8  & 86.2\tabularnewline
RWOT \cite{xu2020reliable}  & 95.1  & 94.5  & 99.5  & \textbf{100.0}  & 77.5  & 77.9  & 90.8\tabularnewline
\hline 
\textbf{LDROT}  & \textbf{95.6}  & \textbf{98.0}  & 98.1  & \textbf{100.0}  & \textbf{85.6}  & \textbf{84.9}  & \textbf{93.7}\tabularnewline
\hline 
\end{tabular}}
 
\end{table}
\begin{figure}[t!]
\begin{centering}
\subfloat[{{Closed-set setting with\protect \protect \\
 ~~$\mathcal{Y}^{S}=\mathcal{Y}^{T}=[9]$.}}]{\includegraphics[width=0.5\textwidth]{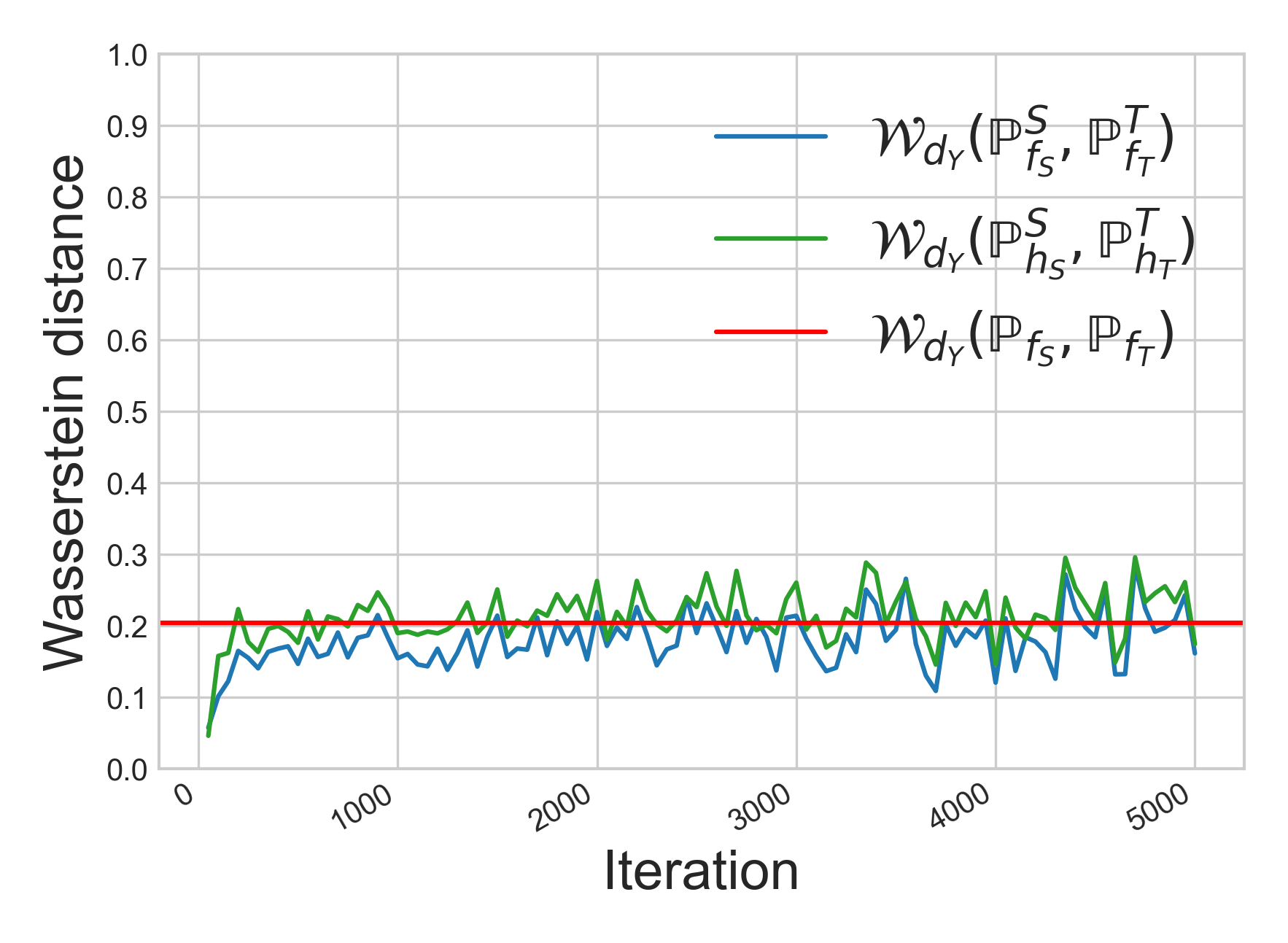}

}\subfloat[{{Open-set setting with \protect \protect \\
 ~~$\mathcal{Y}^{S}=[6]$, $\mathcal{Y}^{T}=[9]$.}}]{\includegraphics[width=0.5\textwidth]{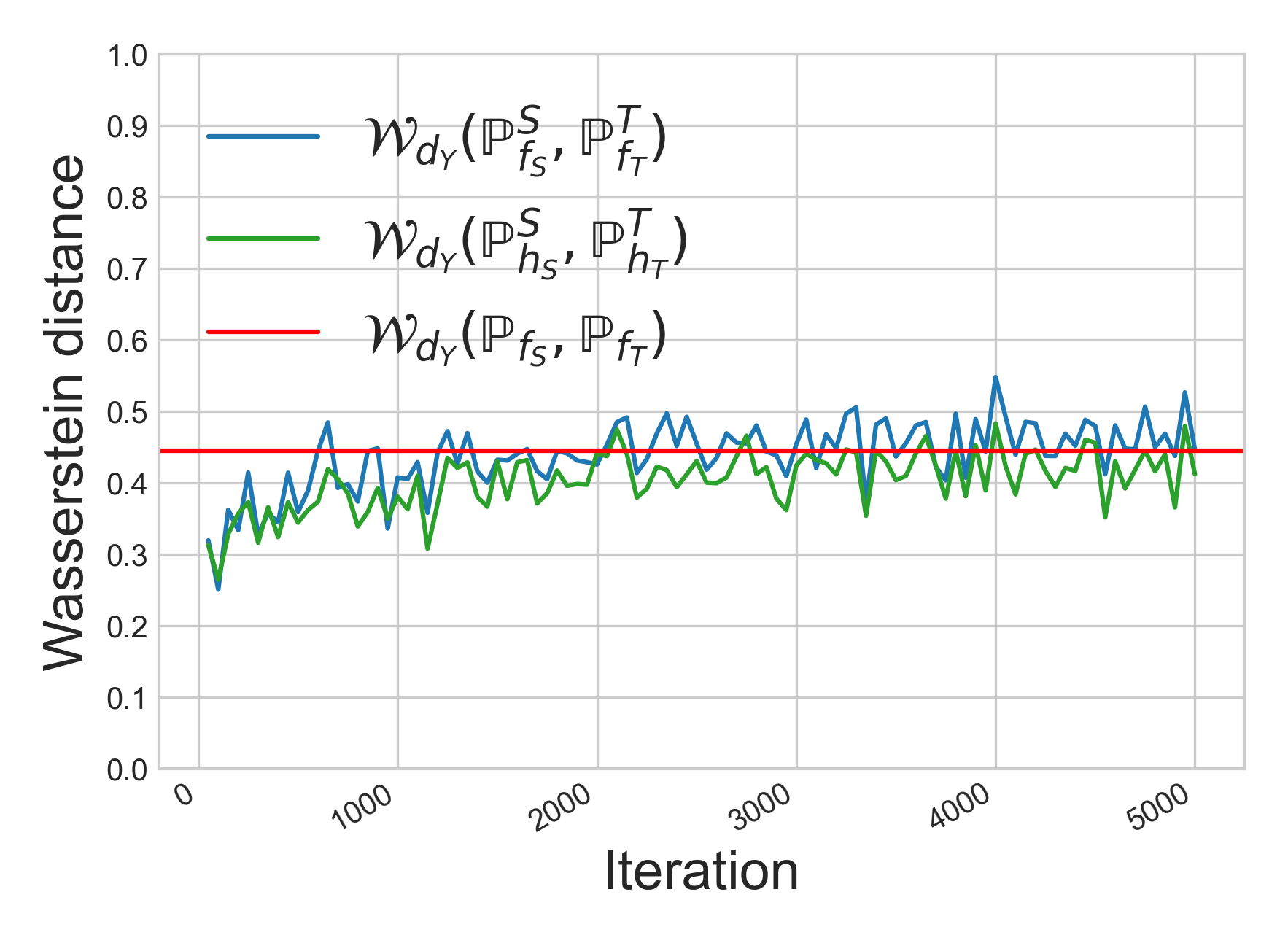}

} \\
\subfloat[{{Partial-set setting with \protect \protect \\
 ~~$\mathcal{Y}^{S}=[9]$, $\mathcal{Y}^{T}=[6]$.}}]{\includegraphics[width=0.5\textwidth]{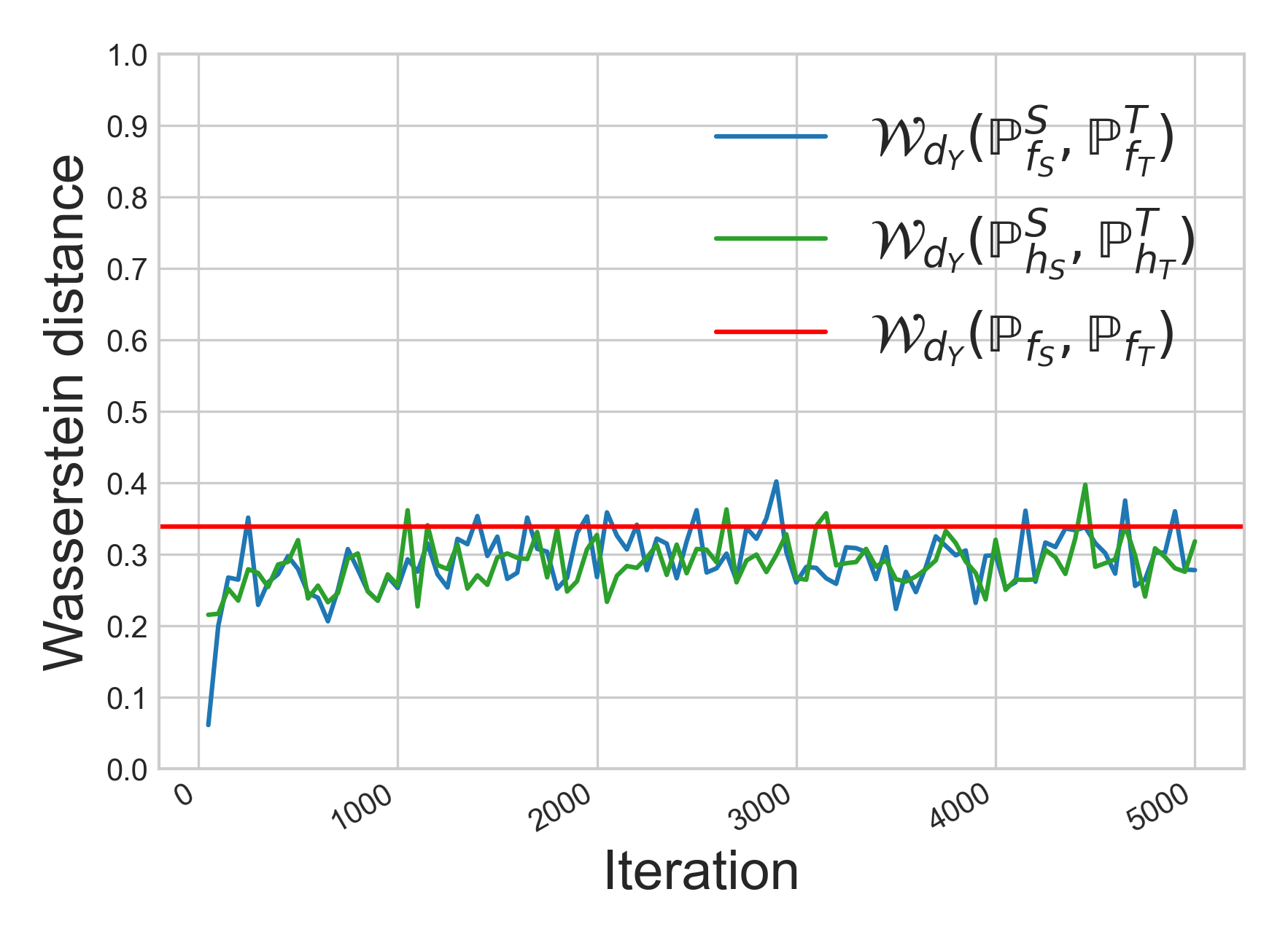}

}\subfloat[{{Universal setting with $\mathcal{Y}^{S}=[3]\cup\left\{ 4,5,6\right\} $, \protect \protect \\
 ~~$\mathcal{Y}^{T}=[3]\cup\left\{ 7,8,9\right\} $. }}]{\includegraphics[width=0.5\textwidth]{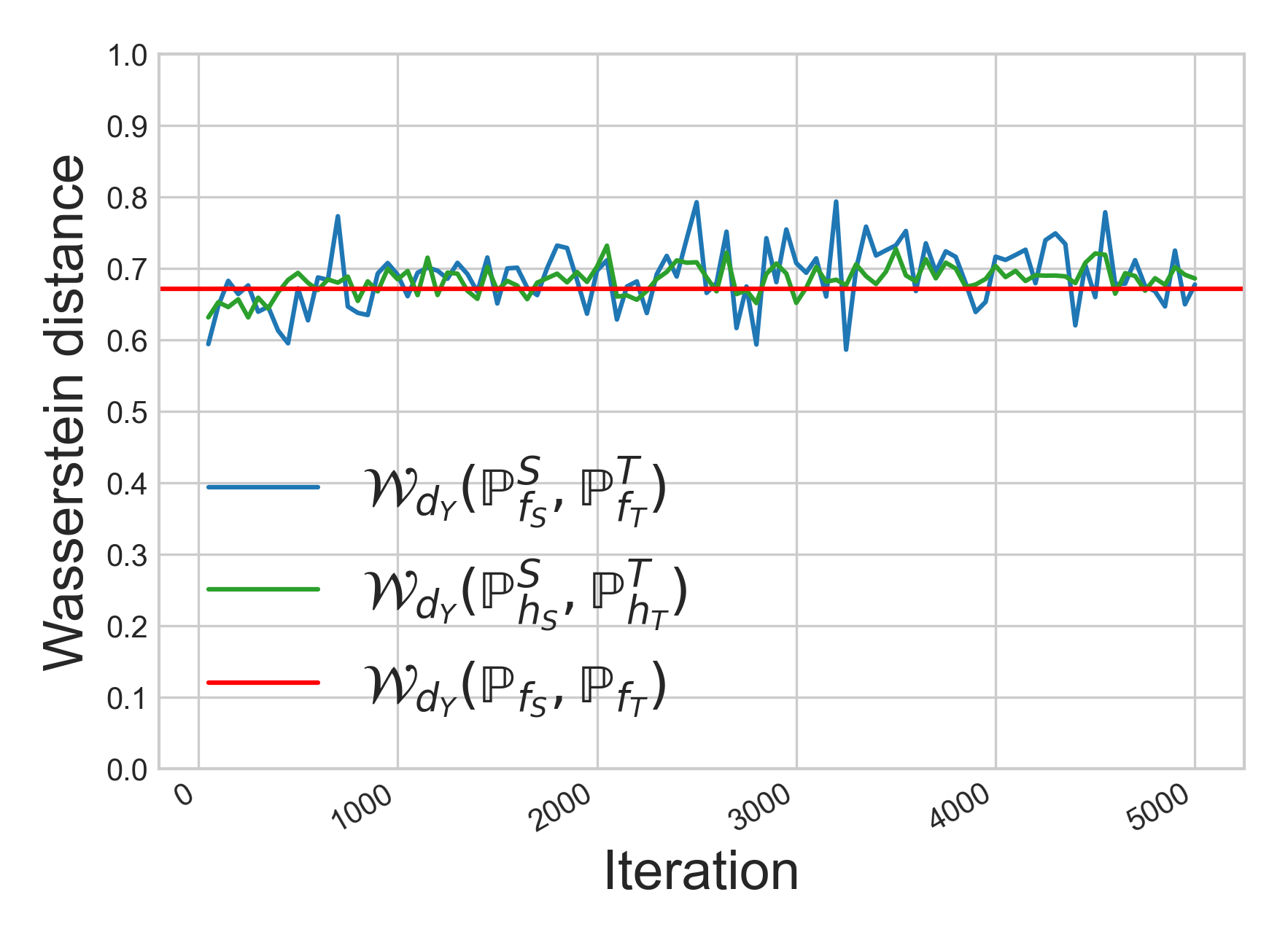}

}
\par
\end{centering}
\centering{} 
 \caption{Label shift estimation for various settings of DA when the source
data set SVHN and the target data set is MNIST. Here, we denote $[C]:=\left\{ 0,1,...,C\right\} $
for a positive integer number $C$.\label{fig:label_shift_estimation}}
\end{figure}

\section{Experiments
}

\subsection{Experiments of theoretical part \label{subsec:label_shift_estimate}}

\textbf{Label shift estimation:} In this experiment, we show how to
evaluate the label shift if we know the labeling mechanisms of the
source and target domains. We consider SVHN as the source domain and
MNIST as the target domain. These two datasets have ten categorical
labels which stand for the digits in $\mathcal{Y}=\left\{ 0,1,...,9\right\} $.
Additionally, for each source or target example $\bx$, the ground-truth
label of $\bx$ is a categorical label in $\mathcal{Y}=\left\{ 0,1,...,9\right\} $.
Since these labels have good separation, from part (iii) of Proposition~\ref{proposition:properties_labelshift},
we can choose $f^{S}\left(\bx\right)$ and $f^{T}\left(\bx\right)$
as one-hot vectors on the simplex $\mathcal{Y}_{\simplex}$. Therefore,
$\mathbb{P}_{f^{S}}$ and $\mathbb{P}_{f^{T}}$ are two discrete distributions
over one-hot vectors representing the categorical labels, wherein
each categorical label $y\in\left\{ 0,...,9\right\} $ corresponds
to the one-hot vector $\bone_{y+1}=\left[0,..,0,1_{y+1},0,...,0\right]$.
The label shift between SVHN and MNIST is estimated by either $\mathcal{W}_{d_{Y}}\left(\mathbb{P}_{f^{S}}^{S},\mathbb{P}_{f^{T}}^{T}\right)$
or $\mathcal{W}_{d_{Y}}\left(\mathbb{P}_{f^{S}},\mathbb{P}_{f^{T}}\right)$,
where $d_{Y}$ is chosen as $L^{1}$ distance. More specifically,
$\mathcal{W}_{d_{Y}}\left(\mathbb{P}_{f^{S}},\mathbb{P}_{f^{T}}\right)$
is evaluated accurately via linear programming\footnote{\url{https://pythonot.github.io/all.html}},
while estimating $\mathcal{W}_{d_{Y}}\left(\mathbb{P}_{f^{S}}^{S},\mathbb{P}_{f^{T}}^{T}\right)$
using the entropic regularized dual form $\mathcal{W}_{d_{Y}}^{\epsilon}\left(\mathbb{P}_{f^{S}}^{S},\mathbb{P}_{f^{T}}^{T}\right)$
with $\epsilon=0.1$~\cite{stochastic_ws}. Moreover, to visualize
the precision when using a probabilistic labeling mechanism to estimate
the label shift, we train two probabilistic labeling functions $h^{S}$
and $h^{T}$ by minimizing the cross-entropy loss with respect to
$f^{S}$ and $f^{T}$respectively and subsequently estimate $\mathcal{W}_{d_{Y}}\left(\mathbb{P}_{h^{S}},\mathbb{P}_{h^{T}}\right)$
using the entropic regularized dual form $\mathcal{W}_{d_{Y}}^{\epsilon}\left(\mathbb{P}_{h^{S}}^{S},\mathbb{P}_{h^{T}}^{T}\right)$
with $\epsilon=0.1$. Note that the prediction probabilities of $h^{S}$
and $h^{T}$ are now the points on the simplex $\mathcal{Y}_{\simplex}$.

We compute the label shift for four DA settings including the closed-set,
partial-set, open-set, and universal settings. As shown in Figure
\ref{fig:label_shift_estimation}, for all DA settings, the blue lines
estimating $\mathcal{W}_{d_{Y}}\left(\mathbb{P}_{f^{S}}^{S},\mathbb{P}_{f^{T}}^{T}\right)$
and the green lines estimating $\mathcal{W}_{d_{Y}}\left(\mathbb{P}_{h^{S}},\mathbb{P}_{h^{T}}\right)$
along with batches tend to approach the red lines evaluating $\mathcal{W}_{d_{Y}}\left(\mathbb{P}_{f^{S}},\mathbb{P}_{f^{T}}\right)$
accurately, which illustrates the result of part (iii) in Proposition~\ref{proposition:properties_labelshift}.
We also observe that the label shifts of partial-set, open-set, and
universal settings are higher than the closed-set setting as discussed
in Theorem \ref{thm:label_shift_setting}.

\vspace{0.5 em}
\noindent
\textbf{Implication on target performance:} In this experiment, we
demonstrate that the theoretical finding of Theorem \ref{thm:bound_py_diff}
indicating that forcing learning domain-invariant representations
hurts the target performance. We train a classifier $h^{ST}=g\circ h$,
where $g$ is a feature extractor and $h$ is a classifier on top
of latent representations by solving $\min_{g,h}\{\mathcal{L}\left(h^{ST},f^{S},\mathbb{P}^{S}\right)+0.1\times\mathcal{W}_{L^{1}}^{\epsilon}\left(g\#\mathbb{P}^{S},g\#\mathbb{P}^{T}\right)\}$
where $\mathcal{W}_{L^{1}}^{\epsilon}\left(g\#\mathbb{P}^{S},g\#\mathbb{P}^{T}\right)$
is used to estimate $\mathcal{W}_{L^{1}}\left(g\#\mathbb{P}^{S},g\#\mathbb{P}^{T}\right)$
for learning domain-invariant representations on the latent space.
We conduct the experiments on the pairs A$\rightarrow$W (Office-31)
and P$\rightarrow$I (ImageCLEF-DA) in which we measure the WS data
shift on the latent space $\mathcal{W}_{L^{1}}^{\epsilon}\left(g\#\mathbb{P}^{S},g\#\mathbb{P}^{T}\right)$,
and the source and target accuracies. As shown in Figure \ref{fig:target_performance_hurt},
along with the training process, while the WS data shift on the latent
space consistently decreases (i.e., the latent representations become
more domain-invariant), the source accuracies get saturated, but the
target accuracies get hurt gradually. 
\begin{figure}[!t]
 \centering{}\includegraphics[width=0.5\textwidth]{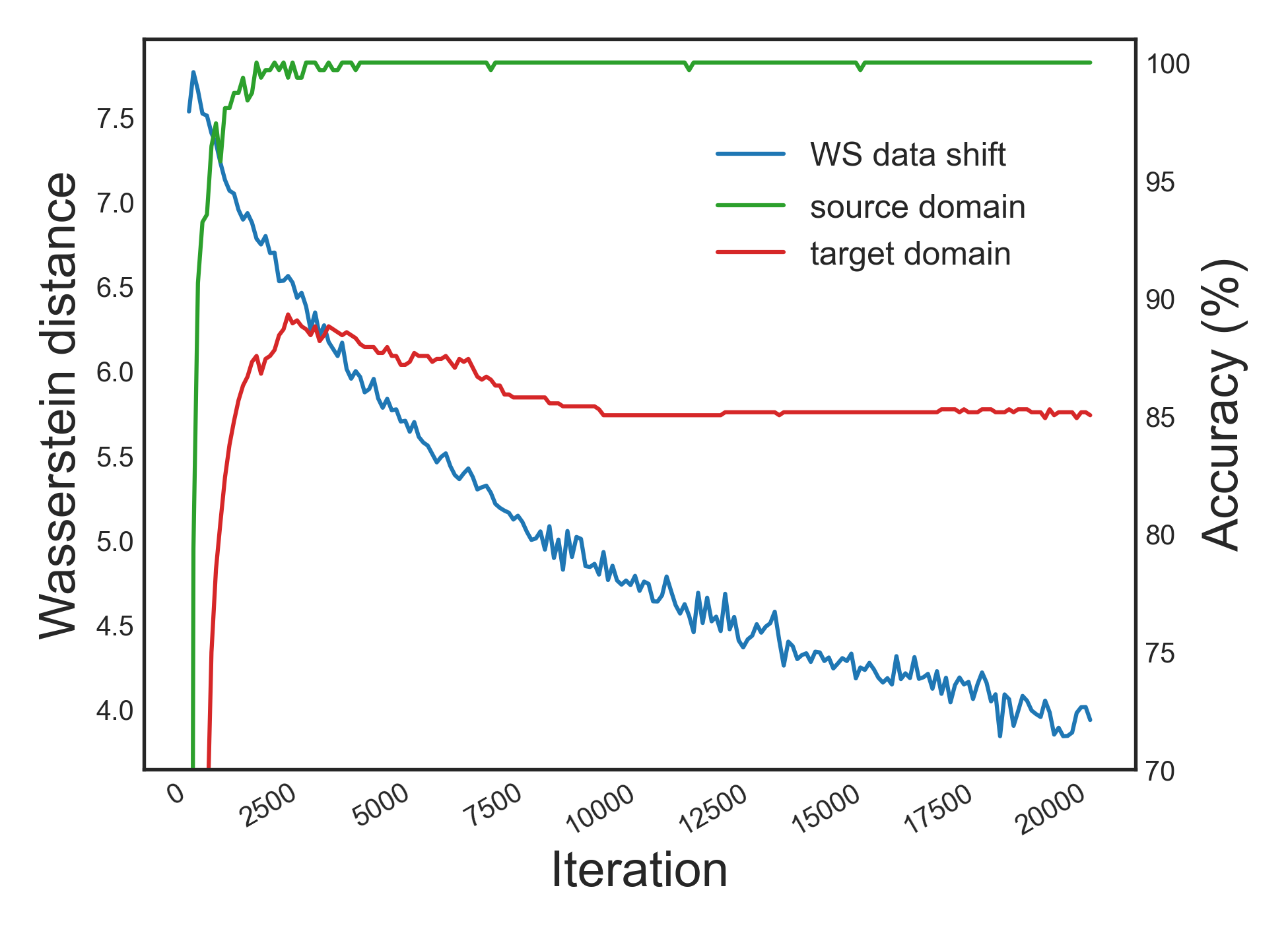}~~~~~~~~~\includegraphics[width=0.5\textwidth]{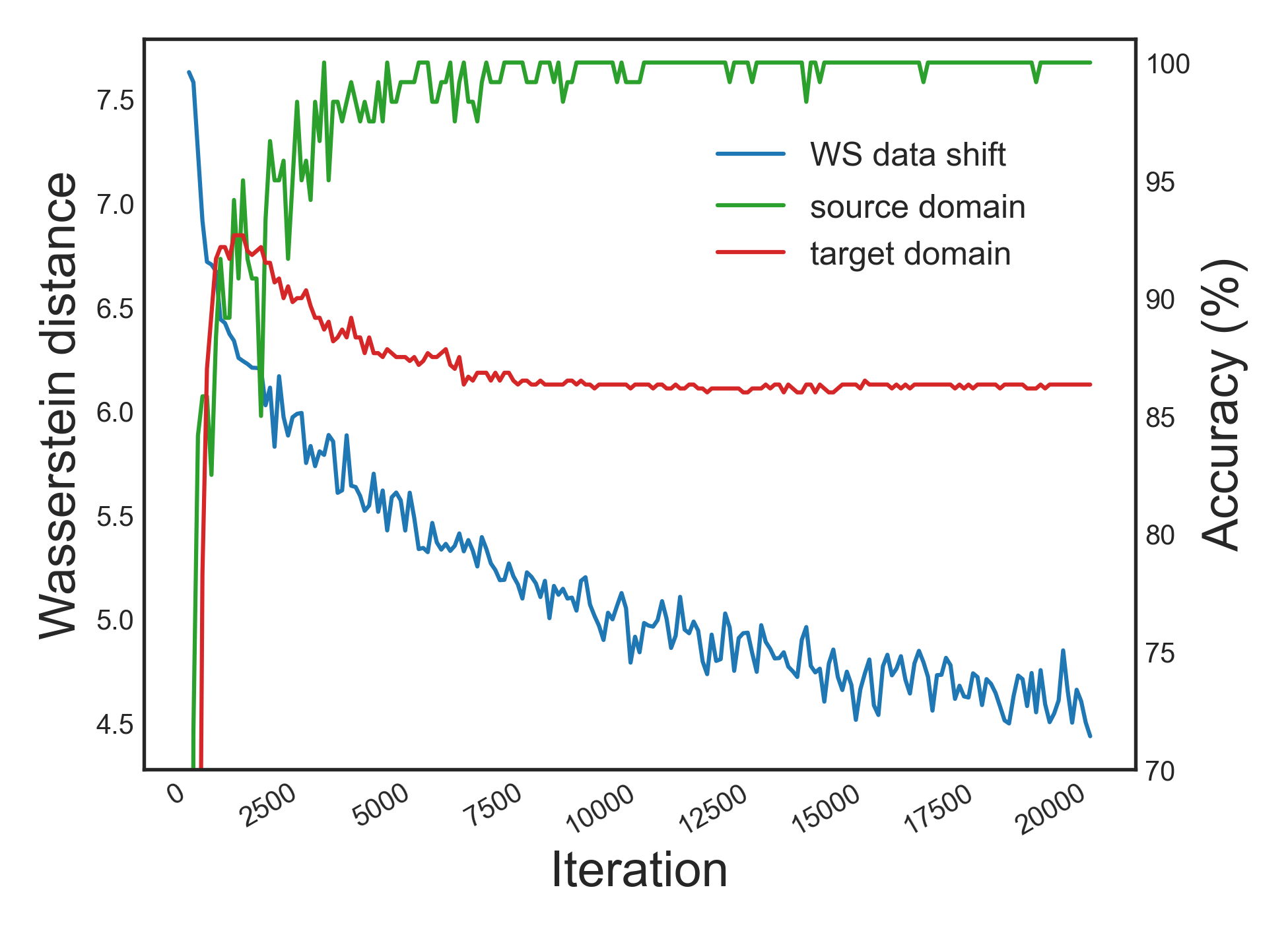}
 \caption{Illustration on target performance to show that forcing learning domain-invariant
representations can hurt the target performance. Left: A$\rightarrow$W
(Office-31). Right: P$\rightarrow$I (ImageCLEF-DA). \label{fig:target_performance_hurt}}
 
\end{figure}

\subsection{Experiments of LDROT on real-world datasets}

We conduct the experiments on the real-world datasets: Digits, Office-31,
Office-Home, and ImageCLEF-DA to compare our LDROT to the state-of-the-art
baselines, especially OT-based ones DeepJDOT \cite{damodaran2018deepjdot},
SWD \cite{chenyu2019swd}, DASPOT \cite{yujia2019onscalable}, ETD
\cite{li2020enhanceOT}, and RWOT \cite{xu2020reliable}. Due to
the space limit, we show the results for Office-31 and Office-Home
in Tables \ref{tab:office31-dataset-1} and \ref{tab:office-home dataset},
while \emph{other results, parameter settings, and network architectures}
can be found in Appendix \ref{sec:additional_exps}. The experimental
results indicate that our proposed method outperforms the baselines.
\begin{table}[t!]
\centering{}\caption{Classification accuracy (\%) on Office-Home dataset for unsupervised
DA (ResNet-50).\label{tab:office-home dataset}}
\resizebox{0.95\textwidth}{!}{\centering\setlength{\tabcolsep}{2pt}
\begin{tabular}{cccccccccccccc}
\hline 
Method  & Ar$\rightarrow$Cl  & Ar$\rightarrow$Pr  & Ar$\rightarrow$Rw  & Cl$\rightarrow$Ar  & Cl$\rightarrow$Pr  & Cl$\rightarrow$Rw  & Pr$\rightarrow$Ar  & Pr$\rightarrow$Cl  & Pr$\rightarrow$Rw  & Rw$\rightarrow$Ar  & Rw$\rightarrow$Cl  & Rw$\rightarrow$Pr  & Avg\tabularnewline
\hline 
ResNet-50 \cite{he2016resnet}  & 34.9  & 50.0  & 58.0  & 37.4  & 41.9  & 46.2  & 38.5  & 31.2  & 60.4  & 53.9  & 41.2  & 59.9  & 46.1\tabularnewline
DANN \cite{ganin2016domain}  & 43.6  & 57.0  & 67.9  & 45.8  & 56.5  & 60.4  & 44.0  & 43.6  & 67.7  & 63.1  & 51.5  & 74.3  & 56.3\tabularnewline
CDAN \cite{long2018cdan}  & 50.7  & 70.6  & 76.0  & 57.6  & 70.0  & 70.0  & 57.4  & 50.9  & 77.3  & \textbf{70.9}  & 56.7  & 81.6  & 65.8\tabularnewline
TPN \cite{pan2019tpn}  & 51.2  & 71.2  & 76.0  & 65.1  & 72.9  & 72.8  & 55.4  & 48.9  & 76.5  & \textbf{70.9}  & 53.4  & 80.4  & 66.2\tabularnewline
SAFN \cite{xu2019sfan}  & 52.0  & 71.7  & 76.3  & 64.2  & 69.9  & 71.9  & 63.7  & 51.4  & 77.1  & \textbf{70.9}  & 57.1  & 81.5  & 67.3\tabularnewline
DeepJDOT \cite{damodaran2018deepjdot}  & 48.2  & 69.2  & 74.5  & 58.5  & 69.1  & 71.1  & 56.3  & 46.0  & 76.5  & 68.0  & 52.7  & 80.9  & 64.3\tabularnewline
ETD \cite{li2020enhanceOT}  & 51.3  & 71.9  & \textbf{85.7}  & 57.6  & 69.2  & 73.7  & 57.8  & 51.2  & 79.3  & 70.2  & 57.5  & 82.1  & 67.3\tabularnewline
RWOT \cite{xu2020reliable}  & 55.2  & 72.5  & 78.0  & 63.5  & 72.5  & 75.1  & 60.2  & 48.5  & 78.9  & 69.8  & 54.8  & 82.5  & 67.6\tabularnewline
\hline 
\textbf{LDROT}  & \textbf{57.4}  & \textbf{79.6}  & 82.5  & \textbf{67.2}  & \textbf{79.8}  & \textbf{80.7}  & \textbf{66.5}  & \textbf{53.3}  & \textbf{82.5}  & \textbf{70.9}  & \textbf{57.4}  & \textbf{84.8}  & \textbf{71.9}\tabularnewline
\hline 
\end{tabular}}
 
\end{table}

\subsection{Ablation Study for LDROT}
This ablation study investigates the significant meanings of the shifting term $\mathcal{L}^{shift}=W_d(\mathbb{P}^{S}_{h^S},\mathbb{P}^{T}_{h^T})$. We observe the values of $\mathcal{L}^{shift}$ together with training accuracy of the source (\textbf{CI}) and target (\textbf{Rw}) domains in Figure~\ref{figure:ablation}. During
training, $\mathcal{L}^{shift}$ smoothly decreases, while the source and
target training accuracies increases. This implies that $h^S$ gradually improves on the source domain, while the label
shift $\mathcal{L}^{shift}$ between $h^T$ and $h^S$ gradually reduces, guiding $h^T$ to improve on the target domain. This observation demonstrates the rational of our label shift quantity $W_d(\mathbb{P}^{S}_{h^S},\mathbb{P}^{T}_{h^T})$. Due to the space limitation, other ablation
studies are presented in Appendix~\ref{subsec:Ablation-studies}.

\begin{figure}[!t]
\centering{}
\includegraphics[width=0.6\textwidth]{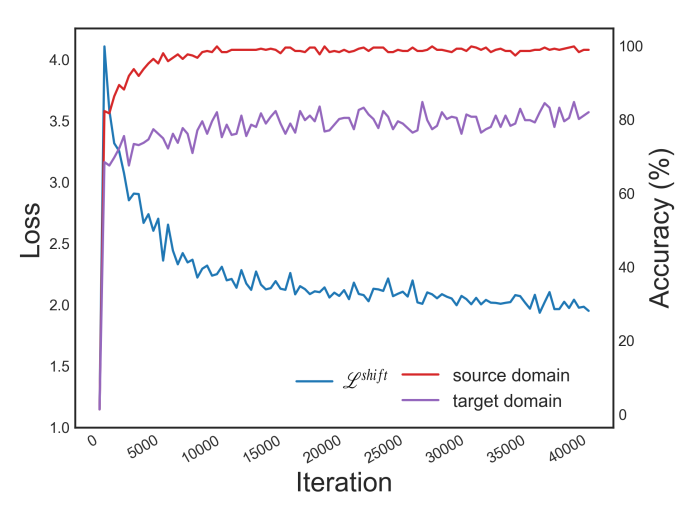}
\caption{The shifting loss $\mathcal{L}^{shift}$ and accuracy (\%) of transfer task \textbf{CI}$\to$\textbf{Rw} of \textit{Office-Home} during training.}
\label{figure:ablation}
\end{figure}
\section{Conclusion
}

In this paper, we study label shift between the source and target
domains in a general DA setting. Our main workaround is to consider
valid transformations transporting the target to source domains allowing
us to align the source and target examples and rigorously define the
label shift between two domains. We then connect the proposed label
shift to optimal transport theory and develop further theory to inspect
the properties of DA under various DA settings (e.g., closed-set,
partial-set, open-set, or universal setting). Furthermore inspired
from theory development, we propose\emph{ }\textbf{\emph{L}}\emph{abel
and }\textbf{\emph{D}}\emph{ata Shift }\textbf{\emph{R}}\emph{eduction
via }\textbf{\emph{O}}\emph{ptimal }\textbf{\emph{T}}\emph{ransport}
(LDROT) which can mitigate data and label shifts simultaneously. We
conduct comprehensive experiments to verify our theoretical findings
and compare LDROT against state-of-the-art baselines to demonstrate
its merits. 

\bibliography{nips21}
\bibliographystyle{abbrv}




\appendix

\clearpage
\appendix

\begin{center}
{\bf{\LARGE{Supplement to ``On Label Shift in Domain Adaptation \\
 via Wasserstein Distance''}}}
\end{center}

In this appendix, we collect several proofs and remaining materials
that are deferred from the main paper. 
\begin{itemize}
\item In Appendix \ref{sec:Notations}, we present notations and definitions
that are deferred from the main text including optimal transport and
entropic regularized Wasserstein distance. 
\item In Appendix \ref{sec:proof_key_results}, we present proofs of all
the key results. 
\item In Appendix \ref{sec:proof_remaining_results}, we present proofs
of the remaining results, including the derivation of the dual form
of entropic regularized optimal transport. 
\item In Appendix \ref{sec:additional_exps}, we provide training specification
and additional experimental results. 
\end{itemize}

\section{Notations and definitions \label{sec:Notations}}

In this appendix, we provide notations, notions, and definitions that
are used in the main text. 

\subsection{Optimal Transport}

Given two probability measures $\left(\mathcal{X},\mathbb{P}\right)$
and $\left(\mathcal{Y},Q\right)$ and a cost function or ground metric
$d\left(\bx,\by\right)$, under the conditions stated in the below
theorem (cf. Theorems 1.32 and 1.33 \cite{santambrogio2015optimal}),
the \emph{primal form} of Wasserstein (WS) distance \cite{santambrogio2015optimal}
is defined as: 
\begin{align}
\text{\ensuremath{\mathcal{W}}}_{d}\left(\mathbb{P},\mathbb{Q}\right)= & \inf_{T:T\#\mathbb{P}=\mathbb{Q}}\mathbb{E}_{\bx\sim\mathbb{P}}\left[d\left(\bx,T\left(\bx\right)\right)\right],\label{eq:primal_form-1}\\
\text{\ensuremath{\mathcal{W}}}_{d}\left(\mathbb{P},\mathbb{Q}\right)= & \inf_{\gamma\in\Gamma\left(\mathbb{P},\mathbb{Q}\right)}\mathbb{E}_{\left(\bx,\by\right)\sim\gamma}\left[d\left(\bx,\by\right)\right],\label{eq:primal_form}
\end{align}
where $\Gamma\left(\mathbb{P},\mathbb{Q}\right)$ specifies the set
of joint distributions over $\mathcal{X}\times\mathcal{Y}$ which
admits $\mathbb{P}$ and $\mathbb{Q}$ as marginals. The first definition
is known as Monge problem (MP), while the second one is known as Kantorovich
problem (KP). We now restate the sufficient conditions for which (MP)
and (KP) are equivalent (cf. Theorems 1.32 and 1.33 \cite{santambrogio2015optimal}). 
\begin{thm}
\label{thm:equi_WS-1}If $\mathcal{X}$ and $\mathcal{Y}$ are compact,
Polish metric spaces, $\mathbb{P}$ and $\mathbb{Q}$ are atomless,
and $d$ is a lower semi-continuous function, then (KP) is equivalent
to (MP) in the sense that two infima are equal. 
\end{thm}
In addition, under some mild conditions as stated in Theorem 5.10
in \cite{villani2008optimal}, we can replace the primal form by
its corresponding dual form 
\begin{equation}
\mathcal{W}_{d}\left(\mathbb{P},\mathbb{Q}\right)=\max_{\phi\in\mathcal{L}_{1}\left(\Omega,\mathbb{P}\right)}\left\{ \mathbb{E}_{\mathbb{Q}}\left[\phi^{c}\left(\bx\right)\right]+\mathbb{E}_{\mathbb{P}}\left[\phi\left(\by\right)\right]\right\} ,\label{eq:dual_form}
\end{equation}
where $\mathcal{L}_{1}\left(\Omega,\mathbb{P}\right):=\left\{ \psi:\int_{\Omega}\left|\psi\left(\by\right)\right|\mathrm{d}\mathbb{P}\left(\by\right)<\infty\right\} $
and $\phi^{c}$ is the $c$-transform of function $\phi$ defined
as $\phi^{c}\left(\bx\right):=\min_{\by}\left\{ d\left(\bx,\by\right)-\phi\left(\by\right)\right\} .$

\subsection{Entropic Regularized Duality \label{subsec:Entropic-Regularized-Duality}}

To enable the application of optimal transport in machine learning
and deep learning, Genevay et al.~\cite{stochastic_ws}developed an entropic regularized
dual form, which had been shown to have favorable computational complexities~\cite{Lin-2019-Efficient, Lin-2019-Efficiency, Lin-2020-Revisiting}. First, they proposed to add an
entropic regularization term to the primal form in (\ref{eq:primal_form})
\begin{equation}
\mathcal{W}_{d}^{\epsilon}\left(\mathbb{P},\mathbb{Q}\right):=\min_{\gamma\in\Gamma\left(\mathbb{Q},\mathbb{P}\right)}\left\{ \mathbb{E}_{\left(\bx,\by\right)\sim\gamma}\left[d\left(\bx,\by\right)\right]+\epsilon D_{KL}\left(\gamma\Vert\mathbb{P}\otimes\mathbb{Q}\right)\right\} ,\label{eq:entropic_primal}
\end{equation}
where $\epsilon$ is the regularization rate, $D_{KL}\left(\cdot\Vert\cdot\right)$
is the Kullback-Leibler (KL) divergence, and $\mathbb{P}\otimes\mathbb{Q}$
represents the specific coupling in which $\mathbb{Q}$ and $\mathbb{P}$
are independent. Note that when $\epsilon\goto0$, $\mathcal{W}_{d}^{\epsilon}\left(\mathbb{P},\mathbb{Q}\right)$
approaches $\mathcal{W}_{d}\left(\mathbb{P},\mathbb{Q}\right)$ and
the optimal transport plan $\gamma_{\epsilon}^{*}$ of (\ref{eq:entropic_primal})
also weakly converges to the optimal transport plan $\gamma^{*}$
of (\ref{eq:primal_form}). In practice, we set $\epsilon$ to be
a small positive number, hence $\gamma_{\epsilon}^{*}$ is very close
to $\gamma^{*}$. Second, using the Fenchel-Rockafellar theorem, they
obtained the following dual form w.r.t. the potential $\phi$

\begin{align}
\mathcal{W}_{d}^{\epsilon}\left(\mathbb{P},\mathbb{Q}\right) & =\max_{\phi}\left\{ \int\phi_{\epsilon}^{c}\left(\bx\right)\mathrm{d}\mathbb{Q}\left(\bx\right)+\int\phi\left(\by\right)\mathrm{d}\mathbb{P}\left(\by\right)\right\} =\max_{\phi}\left\{ \mathbb{E}_{\mathbb{Q}}\left[\phi_{\epsilon}^{c}\left(\bx\right)\right]+\mathbb{E}_{\mathbb{P}}\left[\phi\left(\by\right)\right]\right\} ,\label{eq:entropic_dual}
\end{align}
where $\phi_{\epsilon}^{c}\left(\bx\right):=-\epsilon\log\left(\mathbb{E}_{\mathbb{P}}\left[\exp\left\{ \frac{-d\left(\bx,\by\right)+\phi\left(\by\right)}{\epsilon}\right\} \right]\right)$.

\subsection{Preliminaries}

\paragraph{Notions.}

For a positive integer $n$ and a real number $p\in[1,\infty)$, $[n]$
indicates the set $\{1,2,\ldots,n\}$ while $\|x\|_{p}$ denotes the
$l_{p}$-norm of a vector $x\in\mathbb{R}^{n}$. Let $\mathcal{Y}^{S}$
and $\mathcal{Y}^{T}$ be the label sets of the source and target
domains that have $M^{S}:=\left|\mathcal{Y}^{S}\right|$ and $M^{T}:=\left|\mathcal{Y}^{T}\right|$
elements, respectively. Meanwhile, $\mathcal{Y}=\mathcal{Y}^{S}\cup\mathcal{Y}^{T}$
stands for the label set of both domains which has the cardinality
of $M:=\left|\mathcal{Y}\right|$. Subsequently, we denote $\mathcal{Y}_{\Delta}$,
$\mathcal{Y}_{\Delta}^{S}$, and $\mathcal{Y}_{\Delta}^{T}$ as the
simplices corresponding to $\mathcal{Y},\mathcal{Y}^{S}$, and $\mathcal{Y}^{T}$
respectively. Finally, let $f^{S}\left(\cdot\right)\in\mathcal{Y}_{\simplex}$
and $f^{T}\left(\cdot\right)\in\mathcal{Y}_{\simplex}$ be the labeling
functions of the source and target domains, respectively, by filling
zeros for the missing labels.

We first examine a general supervised learning setting. Consider a
hypothesis $h$ in a hypothesis class $\mathcal{H}$ and a labeling
function $f$ (i.e., $f\left(\cdot\right)\in\mathcal{Y}_{\triangle}$
and $h\left(\cdot\right)\in\mathcal{Y}_{\triangle}$ where $\mathcal{Y}_{\triangle}:=\left\{ \bpi\in\mathbb{R}^{M}:\norm{\bpi}_{1}=1\,\text{and \ensuremath{\bpi\geq\bzero}}\right\} $
with the number of classes $M$). Let $d_{\mathcal{Y}}$ be a metric
over $\mathcal{Y}_{\triangle}$. We further define the general loss
of the hypothesis $h$ w.r.t. the data distribution $\mathbb{P}$
and the labeling function $f$ as: $\mathcal{L}\left(h,f,\mathbb{P}\right):=\int d_{\mathcal{Y}}\left(h\left(\bx\right),f\left(\bx\right)\right)\mathrm{d}\mathbb{P}\left(\bx\right).$

Next we consider a domain adaptation setting in which we have a source
space $\mathcal{X}^{S}$ endowed with a distribution $\mathbb{P}^{S}$
and the density function $p^{s}\left(\bx\right)$ and a target space
$\mathcal{X}^{T}$ endowed with a distribution $\mathbb{P}^{T}$and
the density function $p^{t}\left(\bx\right)$. Let $f^{S}\left(\cdot\right)\in\mathcal{Y}_{\simplex}$
and $f^{T}\left(\cdot\right)\in\mathcal{Y}_{\simplex}$ be the labeling
functions of the source and target domains respectively. It appears
that $p^{S}\left(\bx,y\right)=p^{S}\left(\bx\right)f^{S}\left(\bx,y\right)$
and $p^{T}\left(\bx,y\right)=p^{T}\left(\bx\right)f^{T}\left(\bx,y\right)$
are the source and target joint distributions of pairs $\left(\bx,y\right)$
respectively. Note that for a categorical label $y\in\left\{ 1,...,M\right\} $,
$f^{S}\left(\bx,y\right)$ and $f^{T}\left(\bx,y\right)$ represent
the $y-$th element of the prediction probabilities $f^{S}\left(\bx\right)$
and $f^{T}\left(\bx\right)$.

\section{Proofs of all the key results \label{sec:proof_key_results}}

In this appendix, we provide useful lemmas and proofs for main results
in the paper. 

\subsection{Useful Lemmas}
\begin{lem}
\label{lem:my_gluing} If $\gamma\in\Gamma\left(\mathbb{P}_{f^{S}},\mathbb{P}_{f^{T}}\right)$,
there exists $\gamma'\in\Gamma\left(\mathbb{P}^{S},\mathbb{P}^{T}\right)\text{ such that }\left(f^{S},f^{T}\right)_{\#}\gamma'=\gamma$. 
\end{lem}

\begin{proof}
Let denote $\gamma^{S}$ as the joint distribution of the samples
$\left(\bx^{S},f^{S}\left(\bx^{S}\right)\right)$ where $\bx^{S}\sim\mathbb{P}^{S}$
and $\gamma^{T}$ as the joint distribution of the samples $\left(\bx^{T},f^{T}\left(\bx^{T}\right)\right)$
where $\bx^{T}\sim\mathbb{P}^{T}$. It is obvious that $\gamma^{S}$
is a joint distribution of $\mathbb{P}^{S}$ and $\mathbb{P}^{f^{S}}$
and $\gamma^{T}$ is a joint distribution of $\mathbb{P}^{T}$ and
$\mathbb{P}^{f^{T}}$. According to the gluing lemma (see Lemma 5.5
in \cite{santambrogio2015optimal}), there exists a joint distribution
$\mu$ such that for any draw $\left(\bx^{S},\btau^{S},\btau^{T},\bx^{T}\right)\sim\mu$
then $\left(\bx^{S},\btau^{S}\right)\sim\gamma^{S}$, $\left(\btau^{S},\btau^{T}\right)\sim\gamma$,
and $\left(\bx^{T},\btau^{T}\right)\sim\gamma^{T}$.

Let $\gamma'$ be the distribution of samples $\left(\bx^{S},\bx^{T}\right)$
(i.e., the projection of $\mu$ onto the first and fourth dimensions).
This follows that $\gamma'$ is a joint distribution of $\mathbb{P}^{S}$
and $\mathbb{P}^{T}$ (i.e., $\gamma'\in\Gamma\left(\mathbb{P}^{S},\mathbb{P}^{T}\right)$).
In addition, since $\left(\bx^{S},\btau^{S}\right)\sim\gamma^{S}$,
$\btau^{S}=f^{S}\left(\bx^{S}\right)$, since $\left(\bx^{T},\btau^{T}\right)\sim\gamma^{T}$,
$\btau^{T}=f^{T}\left(\bx^{T}\right)$, and $\left(\btau^{S},\btau^{T}\right)\sim\gamma$.
Therefore, we reach $\left(f^{S},f^{T}\right)_{\#}\gamma'=\gamma$.

We note that in the above proof, we employ a general form of the gluing
lemma for 4 distributions and spaces. The proof is mainly based on
the gluing lemma for 3 distributions and spaces and trivial. 
\end{proof}
\begin{lem}
\label{lem:proper_metric} Let $d_{Z}$ be defined with respect to
the family $\mathcal{H}$ as follows: 
\begin{align*}
d_{Z}(z_{1},z_{2})=\mathop{\sup}\limits _{h\in\mathcal{H}}d_{Y}(h(z_{1}),h(z_{2})),
\end{align*}
where $z_{1}$ and $z_{2}$ lie on the latent space $\mathcal{Z}$.
For any $\bz_{1}$ and $\bz_{2}$, if $h\left(\bz_{1}\right)=h\left(\bz_{2}\right),\forall h\in\mathcal{H}$
leads to $\bz_{1}=\bz_{2}$, then $d_{Z}$ is a proper metric. 
\end{lem}

\begin{proof}
First, $d_{Z}\left(\bz_{1},\bz_{2}\right)\geq0$ and $d_{Z}\left(\bz_{1},\bz_{2}\right)=0$
means $h\left(\bz_{1}\right)=h\left(\bz_{2}\right),\forall h\in\mathcal{H}$,
which leads to $\bz_{1}=\bz_{2}$. Second, it is obvious that $d_{Z}\left(\bz_{1},\bz_{2}\right)=d_{Z}\left(\bz_{2},\bz_{1}\right),\forall\bz_{1},\bz_{2}$.

Given any $\bz_{1},\bz_{2},\bz_{3}$, we have 
\begin{align*}
d_{Z}\left(\bz_{1},\bz_{3}\right) & =\sup_{h\in\mathcal{H}}d_{Y}\left(h\left(\bz_{1}\right),h\left(\bz_{3}\right)\right)\leq\sup_{h\in\mathcal{H}}\left(d_{Y}\left(h\left(\bz_{1}\right),h\left(\bz_{2}\right)\right)+d_{Y}\left(h\left(\bz_{2}\right),h\left(\bz_{3}\right)\right)\right)\\
 & \leq\sup_{h\in\mathcal{H}}d_{Y}\left(h\left(\bz_{1}\right),h\left(\bz_{2}\right)\right)+\sup_{h\in\mathcal{H}}d_{Y}\left(h\left(\bz_{2}\right),h\left(\bz_{3}\right)\right)\\
 & =d_{Z}\left(\bz_{1},\bz_{2}\right)+d_{Z}\left(\bz_{2},\bz_{3}\right).
\end{align*}
Therefore, $d_{Z}$ is a proper metric. 
\end{proof}

\subsection{Proof and Corollary of Proposition 2}
\begin{proof}
\textbf{(i)} First, we will prove that $\mathcal{W}_{d_{Y}}\left(\mathbb{P}_{f^{S}}^{S},\mathbb{P}_{f^{T}}^{T}\right)\geq\mathcal{L}S\left(S,T\right)$.
Let $H:\text{supp}\left(\mathbb{P}_{f^{T}}^{T}\right)\goto\text{supp}\left(\mathbb{P}_{f^{S}}^{S}\right)$
be such that $H_{\#}\mathbb{P}_{f^{T}}^{T}=\mathbb{P}_{f^{S}}^{S}$
where $\support$ indicates the support of a distribution. We can
express $H$ as 
\[
H\left(\bx,f^{T}\left(\bx\right)\right):=\left(H_{1}\left(\bx,f^{T}\left(\bx\right)\right),H_{2}\left(\bx,f^{T}\left(\bx\right)\right)\right),
\]
with $H_{1}\left(\bx,f^{T}\left(\bx\right)\right)\in\mathcal{X}^{S}$
and $H_{2}\left(\bx,f^{T}\left(\bx\right)\right)\in\mathcal{Y}_{\Delta}$.
Define $K\left(\bx\right):=H_{1}\left(\bx,f^{T}\left(\bx\right)\right)$.
We claim that $K_{\#}\mathbb{P}^{T}=\mathbb{P}^{S}$. Observe first
that for any $U_{S}\subset\mathcal{X}^{S}\times\mathcal{Y}_{\triangle}$,
we have $\mathbb{P}_{f^{S}}^{S}\left(U_{S}\right)=\mathbb{P}^{S}\left(V_{S}\right)$
where $V_{S}\coloneqq\left\{ \bx\in\mathcal{X}^{S}\mid\left(\bx,f^{S}\left(\bx\right)\right)\in U_{S}\right\} $.
Next, let $V_{S}\subset\mathcal{X}^{S}$ be any measurable set and
denote $U_{S}\coloneqq V_{S}\times\mathcal{Y}_{\triangle}$. Then
by using the observation above and the fact $H_{\#}\mathbb{P}_{f^{T}}^{T}=\mathbb{P}_{f^{S}}^{S}$,
we obtain 
\[
\mathbb{P}^{S}\left(V_{S}\right)=\mathbb{P}_{f^{S}}^{S}\left(U_{S}\right)=\mathbb{P}_{f^{T}}^{T}\left(H^{-1}\left(U_{S}\right)\right)=\mathbb{P}_{f^{T}}^{T}\left(K^{-1}\left(V_{S}\right)\times\mathcal{Y}_{\triangle}\right)=\mathbb{P}^{T}\left(K^{-1}\left(V_{S}\right)\right),
\]
or equivalently, $K_{\#}\mathbb{P}^{T}=\mathbb{P}^{S}$. It follows
from the fact $H_{\#}\mathbb{P}_{f^{T}}^{T}=\mathbb{P}_{f^{S}}^{S}$
that $H_{2}\left(\bx,f^{T}\left(\bx\right)\right)=f^{S}\left(K\left(\bx\right)\right)$,
which gives 
\begin{align*}
\mathcal{W}_{d_{Y}}\left(\mathbb{P}_{f^{S}}^{S},\mathbb{P}_{f^{T}}^{T}\right) & =\inf_{H:H_{\#}\mathbb{P}_{f^{T}}^{T}=\mathbb{P}_{f^{S}}^{S}}\mathbb{E}_{\left(\bx,f^{T}\left(\bx\right)\right)\sim\mathbb{P}_{f^{T}}^{T}}\left[d_{Y}\left(f^{T}\left(\bx\right),H_{2}\left(\bx,f^{T}\left(\bx\right)\right)\right)\right]\\
 & \geq\inf_{K:K_{\#}\mathbb{P}^{T}=\mathbb{P}^{S}}\mathbb{E}_{\bx\sim\mathbb{P}^{T}}\left[d_{Y}\left(f^{T}\left(\bx\right),f^{S}\left(K\left(\bx\right)\right)\right)\right]\\
 & =\mathcal{L}S(S,T).
\end{align*}
In order to prove the reverse inequality, let us consider any maps
$K$ satisfying $K_{\#}\mathbb{P}^{T}=\mathbb{P}^{S}$. Define a map
$H:\text{supp}\left(\mathbb{P}_{f^{T}}^{T}\right)\goto\text{supp}\left(\mathbb{P}_{f^{S}}^{S}\right)$
as $H\left(\bx,f^{T}\left(\bx\right)\right)\coloneqq\left(K(\bx),f^{S}\left(K\left(\bx\right)\right)\right)$,
we will show that $H_{\#}\mathbb{P}_{f^{T}}^{T}=\mathbb{P}_{f^{S}}^{S}$.
Indeed, let $U_{S}\subset\mathcal{X}^{S}\times\mathcal{Y}_{\triangle}$
be any measurable sets and take $V_{S}\coloneqq\left\{ \bx\in\mathcal{X}^{S}\mid\left(\bx,f^{S}\left(\bx\right)\right)\in U_{S}\right\} $.
Then, as $$H^{-1}\left(U_{S}\right)=\left\{ \left(\bx,f^{S}\left(\bx\right)\right)\mid K\left(\bx\right)\in V_{S}\right\} =\left\{ \left(\bx,f^{S}\left(\bx\right)\right)\mid\bx\in K^{-1}\left(V_{S}\right)\right\},$$
we have 
\[
\mathbb{P}_{f^{T}}^{T}\left(H^{-1}\left(U_{S}\right)\right)=\mathbb{P}^{T}\left(K^{-1}\left(V_{S}\right)\right)=\mathbb{P}^{S}\left(V_{S}\right)=\mathbb{P}_{f^{S}}^{S}\left(U_{S}\right),
\]
which means $H_{\#}\mathbb{P}_{f^{T}}^{T}=\mathbb{P}_{f^{S}}^{S}$.
As a result, 
\[
\mathcal{W}_{d_{Y}}\left(\mathbb{P}_{f^{S}}^{S},\mathbb{P}_{f^{T}}^{T}\right)\leq\inf_{K:K_{\#}\mathbb{P}^{T}=\mathbb{P}^{S}}\mathbb{E}_{\bx\sim\mathbb{P}_{T}}\left[d_{Y}\left(f^{T}\left(\bx\right),f^{S}\left(K\left(\bx\right)\right)\right)\right]=\mathcal{L}S(S,T).
\]
By combining the above two inequalities, we obtain the desired equality
\[
\mathcal{W}_{d_{Y}}\left(\mathbb{P}_{f^{S}}^{S},\mathbb{P}_{f^{T}}^{T}\right)=\mathcal{L}S(S,T).
\]
\textbf{(ii)} First, let $\gamma\in\Gamma\left(\mathbb{P}_{f^{S}},\mathbb{P}_{f^{T}}\right)$.
According to Lemma \ref{lem:my_gluing}, there exists $\gamma'\in\Gamma\left(\mathbb{P}^{S},\mathbb{P}^{T}\right)$ such that $\left(f^{S},f^{T}\right)_{\#}\gamma'=\gamma$.
Then, we find that 
\begin{align*}
\mathbb{E}_{\left(\by^{S},\by^{T}\right)\sim\gamma}\left[d_{Y}\left(\by^{S},\by^{T}\right)\right] & =\mathbb{E}_{\left(\bx^{S},\bx^{T}\right)\sim\gamma'}\left[d_{Y}\left(f^{S}\left(\bx^{S}\right),f^{T}\left(\bx^{T}\right)\right)\right]\\
 & \geq\inf_{\gamma'\in\Gamma\left(\mathbb{P}^{S},\mathbb{P}^{T}\right)}\mathbb{E}_{\left(\bx^{S},\bx^{T}\right)\sim\gamma'}\left[d_{Y}\left(f^{S}\left(\bx^{S}\right),f^{T}\left(\bx^{T}\right)\right)\right]\\
 & =\mathcal{W}_{d_{Y}}\left(\mathbb{P}_{f^{S}}^{S},\mathbb{P}_{f^{T}}^{T}\right)=\mathcal{L}S\left(S,T\right).
\end{align*}
Therefore, we arrive at 
\begin{align*}
\mathcal{W}_{d_{Y}}\left(\mathbb{P}_{f^{S}},\mathbb{P}_{f^{T}}\right)=\inf_{\gamma\in\Gamma\left(\mathbb{P}_{f^{S}},\mathbb{P}_{f^{T}}\right)}\mathbb{E}_{\left(\by^{S},\by^{T}\right)\sim\gamma}\left[d_{Y}\left(\by^{S},\by^{T}\right)\right]\geq\mathcal{L}S\left(S,T\right).
\end{align*}
Second, let $\gamma'\in\Gamma\left(\mathbb{P}^{S},\mathbb{P}^{T}\right)$,
we denote $\gamma=\left(f^{S},f^{T}\right)_{\#}\gamma'$. We then
have 
\begin{align*}
\mathbb{E}_{\left(\bx^{S},\bx^{T}\right)\sim\gamma'}\left[d_{Y}\left(f^{S}\left(\bx^{S}\right),f^{T}\left(\bx^{T}\right)\right)\right] & =\mathbb{E}_{\left(\by^{S},\by^{T}\right)\sim\gamma}\left[d_{Y}\left(\by^{S},\by^{T}\right)\right]\\
\geq & \inf_{\gamma\in\Gamma\left(\mathbb{P}_{f^{S}},\mathbb{P}_{f^{T}}\right)}\mathbb{E}_{\left(\by^{S},\by^{T}\right)\sim\gamma}\left[d_{Y}\left(\by^{S},\by^{T}\right)\right]=\mathcal{W}_{d_{Y}}\left(\mathbb{P}_{f^{S}},\mathbb{P}_{f^{T}}\right).
\end{align*}
This follows that 
\begin{align*}
\mathcal{L}S\left(S,T\right) & =\mathcal{W}_{d_{Y}}\left(\mathbb{P}_{f^{S}}^{S},\mathbb{P}_{f^{T}}^{T}\right)=\inf_{\gamma'\in\Gamma\left(\mathbb{P}^{S},\mathbb{P}^{T}\right)}\mathbb{E}_{\left(\bx^{S},\bx^{T}\right)\sim\gamma'}\left[d_{Y}\left(f^{S}\left(\bx^{S}\right),f^{T}\left(\bx^{T}\right)\right)\right]\\
 & \geq\mathcal{W}_{d_{Y}}\left(\mathbb{P}_{f^{S}},\mathbb{P}_{f^{T}}\right).
\end{align*}
Hence, the proof is completely done. 
\end{proof}
\begin{cor}
\label{coro:min_WS_new}The following inequality holds $\mathcal{W}_{d_{Y}}\left(\mathbb{P}_{h^{S}}^{S},\mathbb{P}_{f^{S}}^{S}\right)\leq\mathcal{L}\left(h^{S},f^{S},\mathbb{P}^{S}\right).$ 
\end{cor}

\begin{proof}
According to Proposition 1, we have: 
\[
\mathcal{W}_{d_{Y}}\left(\mathbb{P}_{h^{S}}^{S},\mathbb{P}_{f^{S}}^{S}\right)=\inf_{L:L_{\#}\mathbb{P}^{S}=\mathbb{P}^{S}}\mathbb{E}_{\bx\sim\mathbb{P}^{S}}\left[d_{Y}\left(h^{S}\left(\bx\right),f^{S}\left(L\left(\bx\right)\right)\right)\right].
\]
Then by choosing L as the identity map (i.e., $L(\bx)=\bx$ for all
$\bx$), we obtain: 
\[
\mathcal{W}_{d_{Y}}\left(\mathbb{P}_{h^{S}}^{S},\mathbb{P}_{f^{S}}^{S}\right)\leq\mathbb{E}_{\bx\sim\mathbb{P}^{S}}\left[d_{Y}\left(h^{S}\left(\bx\right),f^{S}\left(L\left(\bx\right)\right)\right)\right]=\mathcal{L}\left(h^{S},f^{S},\mathbb{P}^{S}\right).
\]
\end{proof}

\subsection{Proof of Theorem 4}

First, we will show that 
\begin{align*}
\mathcal{L}S\left(S,T\right)\leq\mathcal{L}\left(h^{S},f^{S},\mathbb{P}^{S}\right)+\mathcal{L}\left(h^{T},f^{T},\mathbb{P}^{T}\right)+\mathcal{W}_{d_{Y}}\left(\mathbb{P}_{h^{S}},\mathbb{P}_{h^{T}}\right).
\end{align*}
By using the triangle inequality for the Wasserstein distance with
respect to the metric $d_{Y}$, we have 
\begin{align*}
\mathcal{L}S\left(S,T\right) & =\mathcal{W}_{d_{Y}}\left(\mathbb{P}_{f^{S}},\mathbb{P}_{f^{T}}\right)\\
 & \overset{(1)}{\leq}\mathcal{W}_{d_{Y}}\left(\mathbb{P}_{f^{S}},\mathbb{P}_{h^{S}}\right)+\mathcal{W}_{d_{Y}}\left(\mathbb{P}_{h^{S}},\mathbb{P}_{h^{T}}\right)+\mathcal{W}_{d_{Y}}\left(\mathbb{P}_{h^{T}},\mathbb{P}_{f^{T}}\right)\\
 & =\mathcal{W}_{d_{Y}}\left(\mathbb{P}_{f^{S}}^{S},\mathbb{P}_{h^{S}}^{S}\right)+\mathcal{W}_{d_{Y}}\left(\mathbb{P}_{h^{S}},\mathbb{P}_{h^{T}}\right)+\mathcal{W}_{d_{Y}}\left(\mathbb{P}_{h^{T}}^{T},\mathbb{P}_{f^{T}}^{T}\right)\\
 & \overset{(2)}{\leq}\mathcal{L}\left(h^{S},f^{S},\mathbb{P}^{S}\right)+\mathcal{L}\left(h^{T},f^{T},\mathbb{P}^{T}\right)+\mathcal{W}_{d_{Y}}\left(\mathbb{P}_{h^{S}},\mathbb{P}_{h^{T}}\right).
\end{align*}
Here we note that for $\overset{(1)}{\leq}$, we use the triangle
inequality and for $\overset{(2)}{\leq}$, we invoke Corollary \ref{coro:min_WS_new}.

It is sufficient to prove that 
\[
\mathcal{W}_{d_{Y}}\left(\mathbb{P}_{h^{S}},\mathbb{P}_{h^{T}}\right)\leq\mathcal{W}_{d_{Z}}\left(g^{S}\#\mathbb{P}^{S},g^{T}\#\mathbb{P}^{T}\right).
\]
Indeed, let $\gamma\in\Gamma\left(g^{S}\#\mathbb{P}^{S},g^{T}\#\mathbb{P}^{T}\right)$
and denote $\gamma'=h\#\gamma$. Then, we have $\gamma'\in\Gamma\left(\mathbb{P}_{h^{S}},\mathbb{P}_{h^{T}}\right)$,
and 
\begin{align*}
\mathbb{E}_{\left(\by_{1},\by_{2}\right)\sim\gamma'}\left[d_{Y}\left(\by_{1},\by_{2}\right)\right]=\mathbb{E}_{\left(z_{1},z_{2}\right)\sim\gamma}\left[d_{Y}\left(h\left(z_{1}\right),h\left(z_{2}\right)\right)\right]\leq\mathbb{E}_{\left(z_{1},z_{2}\right)\sim\gamma}\left[d_{Z}\left(z_{1},z_{2}\right)\right].
\end{align*}
Therefore, we obtain 
\begin{align*}
\mathcal{W}_{d_{Y}}\left(\mathbb{P}_{h^{S}},\mathbb{P}_{h^{T}}\right) & =\inf_{\gamma'\in\Gamma\left(\mathbb{P}_{h^{S}},\mathbb{P}_{h^{T}}\right)}\mathbb{E}_{\left(\by_{1},\by_{2}\right)\sim\gamma'}\left[d_{Y}\left(\by_{1},\by_{2}\right)\right]\\
 & \leq\mathbb{E}_{\left(\by_{1},\by_{2}\right)\sim\gamma'}\left[d_{Y}\left(\by_{1},\by_{2}\right)\right]\leq\mathbb{E}_{\left(z_{1},z_{2}\right)\sim\gamma}\left[d_{Z}\left(z_{1},z_{2}\right)\right].
\end{align*}
Finally, we reach 
\[
\mathcal{W}_{d_{Y}}\left(\mathbb{P}_{h^{S}},\mathbb{P}_{h^{T}}\right)\leq\inf_{\gamma\in\Gamma\left(g^{S}\#\mathbb{P}^{S},g^{T}\#\mathbb{P}^{T}\right)}\mathbb{E}_{\left(z_{1},z_{2}\right)\sim\gamma}\left[d_{Z}\left(z_{1},z_{2}\right)\right]=\mathcal{W}_{d_{Z}}\left(g^{S}\#\mathbb{P}^{S},g^{T}\#\mathbb{P}^{T}\right).
\]
Hence, we have proved our claim. 
\subsection{Proof of Theorem~\ref{theorem:motivation}}
Using triangle inequality, we have
\begin{align*}
    LS(h^T,f^T,\mathbb{P}^T)&=W_{d_Y}(\mathbb{P}^{T}_{h^T},\mathbb{P}^{T}_{f^T})=W_{d_Y}(\mathbb{P}_{h^T},\mathbb{P}_{f^T})\\
    &\leq W_{d_Y}(\mathbb{P}_{h^T},\mathbb{P}_{h^S})+W_{d_Y}(\mathbb{P}_{h^S},\mathbb{P}_{f^S})+W_{d_Y}(\mathbb{P}_{f^S},\mathbb{P}_{f^T})\\
    &\leq W_{d_Y}(\mathbb{P}^T_{h^T},\mathbb{P}^S_{h^S})+\mathcal{L}(h^S,f^S,\mathbb{P}^S)+LS(S,T)\\
    &=\frac{1}{2}W_{d_Y}(\mathbb{P}^T_{h^T},\mathbb{P}^S_{h^S})+\mathcal{L}(h^S,f^S,\mathbb{P}^S)+LS(S,T)+\frac{1}{2}W_{d_Y}(\mathbb{P}^T_{h^T},\mathbb{P}^S_{h^S}).
\end{align*}
We further derive 
\begin{align*}
    W_{d_Y}(\mathbb{P}^T_{h^T},\mathbb{P}^S_{h^S})&=\inf_{\gamma\in\Gamma(\mathbb{P}^S_{h^S},\mathbb{P}^T_{h^T})}\mathbb{E}_{\gamma}\Big[d_Y\big(h^S(\bx^S),h^T(\bx^T)\big)\Big]\\
    &=\inf_{\gamma\in\Gamma(\mathbb{P}^S,\mathbb{P}^T)}\mathbb{E}_{\gamma}\Big[d_Y\big(h^S(\bx^S),h^T(\bx^T)\big)\Big]\\
    &\leq\mathbb{E}_{\gamma^*}\Big[d_Y\big(h^S(\bx^S),h^T(\bx^T)\big)\Big]\\
    &\leq\mathbb{E}_{\gamma^*}\Big[d_Y\big(h^T(\bx^S),h^T(\bx^T)\big)\Big]+\mathbb{E}_{\gamma^*}\Big[d_Y\big(h^S(\bx^S),h^T(\bx^S)\big)\Big]\\
    &\leq \mathbb{E}_{\gamma^*}\Big[d_Y\big(h^T(\bx^S),h^T(\bx^T)\big)\Big]+N.
\end{align*}
We denote $A=\{(\bx^S,\bx^T):d_Y\big(h^T(\bx^S),h^T(\bx^T)\big)>\delta d_X(\bx^S,\bx^T)\}$. It appears that $\gamma^*(A)<\phi(\delta)$ and $\gamma^*(A^c)>1-\phi(\delta)$. Therefore, we have
\begin{align*}
     \mathbb{E}_{\gamma^*}\Big[d_Y\big(h^T(\bx^S),h^T(\bx^T)\big)\Big]&=\int_{A}d_Y\big(h^T(\bx^S),h^T(\bx^T)\big)d\gamma^*(\bx^S,\bx^T)+\int_{A^c}d_Y\big(h^T(\bx^S),h^T(\bx^T)\big)d\gamma^*(\bx^S,\bx^T)\\
     &\leq N\gamma^*(A)+\delta\int_{A^c}d_X(\bx^S,\bx^T)d\gamma^*(\bx^S,\bx^T)\\
     &\leq N\phi(\delta)+\delta\int_{\mathcal{X}}d_X(\bx^S,\bx^T)d\gamma^*(\bx^S,\bx^T)\\
     &\leq  N\phi(\delta)+\delta W_{d_X}(\mathbb{P}^S,\mathbb{P}^T).
\end{align*}
Finally, we reach
\begin{align*}
    LS(h^T,f^T,\mathbb{P}^T)\leq LS(S,T)+\mathcal{L}(h^S,f^S,\mathbb{P}^S)+\frac{1}{2}\delta W_{d_X}(\mathbb{P}^{S},\mathbb{P}^{T})+\frac{1}{2}W_{d_Y}(\mathbb{P}^{S}_{h^S},\mathbb{P}^{T}_{h^T})+\frac{[1+\phi(\delta)]N}{2}.
\end{align*}
As a consequence, we obtain the conclusion of the theorem.
\section{Proof of the remaining results \label{sec:proof_remaining_results}}

\subsection{Proof of Proposition 3}

Denote by $p_{Y}^{S}=(p_{Y}^{S}(y))_{y=1}^{M}$ and $p_{Y}^{T}=(p_{Y}^{T}(y))_{y=1}^{M}$
the marginal distributions of the source and target domain labels,
i.e., $p_{Y}^{S}(y)=\int_{\mathcal{X}^{S}}p^{S}(x,y)dx$ and $p_{Y}^{T}(y)=\int_{\mathcal{X}^{T}}p^{T}(x,y)dx$.
Let $\mathbb{P}_{Y}^{a}$ be the discrete measure on the vertices
of $\Delta^{M-1}$ putting mass $P_{Y}^{a}(y)$ on the one-hot representation
of $y$, $\forall\,y\in[M],a\in\{S,T\}$. For $d_{Y}(y,y')=\|y-y'\|_{p}^{p}$
when $p\geq1$, the following holds:

(i) $\mathcal{L}\left(h^{T},f^{T},\mathbb{P}^{T}\right)\leq\mathcal{L}S\left(S,T\right)+\mathcal{L}\left(h^{S},f^{S},\mathbb{P}^{S}\right)+\mathcal{W}_{d_{Y}}\left(\mathbb{P}_{h^{S}}^{S},\mathbb{P}_{h^{T}}^{T}\right)+const$,
where the constant can be viewed as a reconstruction term: $\sup_{L,K:L\#\mathbb{P}^{T}=\mathbb{P}^{S},K\#\mathbb{P}^{S}=\mathbb{P}^{T}}\mathbb{E}_{\mathbb{P}^{T}}\left[d_{Y}\left(f^{T}\left(K\left(L\left(\bx\right)\right)\right),f^{T}\left(\bx\right)\right)\right]$; 

(ii) $\mathcal{LS}(S,T)\geq\|p_{Y}^{S}-p_{Y}^{T}\|_{p}^{p}$;

(iii) In the setting that $\mathbb{P}^{S}$ and $\mathbb{P}^{T}$
are mixtures of well-separated Gaussian distributions, i.e., 
\begin{align*}
p^{a}(\bx)=\sum_{y=1}^{M}p^{a}(y)\mathcal{N}(\bx|\mu_{y}^{a},\Sigma_{y}^{a}),\quad\forall a\in\{S,T\}
\end{align*}

with $\|\mu_{y}^{a}-\mu_{y'}^{a}\|_{2}\geq D\times\max\{\|\Sigma_{y}^{a}\|_{op}^{1/2},\|\Sigma_{y'}^{a}\|_{op}^{1/2}\}\,\forall\,a\in\{S,T\},y\neq y'$,
in which $\|\cdot\|_{op}$ denotes the operator norm and $D$ is sufficiently
large, we have 
\begin{equation}
0\leq|\mathcal{LS}(S,T)-\mathcal{W}_{p}^{p}(\mathbb{P}_{Y}^{S},\mathbb{P}_{Y}^{T})|\leq\epsilon(D),
\end{equation}
where $\epsilon(D)$ is a small constant depending on $D,p_{Y}^{S},p_{Y}^{T},(\Sigma_{y}^{S},\Sigma_{y}^{T})_{y=1}^{M}$,
and it goes to 0 as $D\to\infty$. 

(iv) In the anti-causal setting, where $p^{S}(x|y)=p^{T}(x|y)$ for
all $x,y$, 
\begin{equation}
\mathcal{LS}(S,T)\leq M^{p}\|p_{Y}^{S}-p_{Y}^{T}\|_{1}+\min\{\mathbb{E}_{\mathbb{P}^{S}}\|f^{S}-f^{T}\|_{p}^{p},\mathbb{E}_{\mathbb{P}^{T}}\|f^{S}-f^{T}\|_{p}^{p}\};
\end{equation}

\begin{proof}
\textbf{(i)} Let $L$ and $K$ be two arbitrary maps such that $L\#\mathbb{P}^{T}=\mathbb{P}^{S}$
and $K\#\mathbb{P}^{S}=\mathbb{P}^{T}$. We have the following triangle
inequality: 
\begin{align*}
d_{Y}\left(h^{T}\left(\bx\right),f^{T}\left(\bx\right)\right) & \leq d_{Y}\left(h^{T}\left(\bx\right),h^{S}\left(L\left(\bx\right)\right)\right)+d_{Y}\left(h^{S}\left(L\left(\bx\right)\right),f^{S}\left(L\left(\bx\right)\right)\right)\\
 & +d_{Y}\left(f^{S}\left(L\left(\bx\right)\right),f^{T}\left(K\left(L\left(\bx\right)\right)\right)\right)+d_{Y}\left(f^{T}\left(K\left(L\left(\bx\right)\right)\right),f^{T}\left(\bx\right)\right).
\end{align*}
Therefore, we obtain 
\begin{align*}
\mathbb{E}_{\mathbb{P}^{T}}\left[d_{Y}\left(h^{T}\left(\bx\right),f^{T}\left(\bx\right)\right)\right] & \leq\mathbb{E}_{\mathbb{P}^{T}}\left[d_{Y}\left(h^{T}\left(\bx\right),h^{S}\left(L\left(\bx\right)\right)\right)\right]+\mathbb{E}_{\mathbb{P}^{T}}\left[d_{Y}\left(h^{S}\left(L\left(\bx\right)\right),f^{S}\left(L\left(\bx\right)\right)\right)\right]\\
 & +\mathbb{E}_{\mathbb{P}^{T}}\left[d_{Y}\left(f^{S}\left(L\left(\bx\right)\right),f^{T}\left(K\left(L\left(\bx\right)\right)\right)\right)\right]+\mathbb{E}_{\mathbb{P}^{T}}\left[d_{Y}\left(f^{T}\left(K\left(L\left(\bx\right)\right)\right),f^{T}\left(\bx\right)\right)\right]\\
 & \overset{(1)}{=}\mathbb{E}_{\mathbb{P}^{T}}\left[d_{Y}\left(h^{T}\left(\bx\right),h^{S}\left(L\left(\bx\right)\right)\right)\right]+\mathbb{E}_{\mathbb{P}^{S}}\left[d_{Y}\left(h^{S}\left(\bx\right),f^{S}\left(\bx\right)\right)\right]\\
 & +\mathbb{E}_{\mathbb{P}^{S}}\left[d_{Y}\left(f^{S}\left(\bx\right),f^{T}\left(K\left(\bx\right)\right)\right)\right]+\mathbb{E}_{\mathbb{P}^{T}}\left[d_{Y}\left(f^{T}\left(K\left(L\left(\bx\right)\right)\right),f^{T}\left(\bx\right)\right)\right]\\
 & =\mathbb{E}_{\mathbb{P}^{T}}\left[d_{Y}\left(h^{T}\left(\bx\right),h^{S}\left(L\left(\bx\right)\right)\right)\right]+\mathcal{L}\left(h^{S},f^{S},\mathbb{P}^{S}\right)\\
 & +\mathbb{E}_{\mathbb{P}^{S}}\left[d_{Y}\left(f^{S}\left(\bx\right),f^{T}\left(K\left(\bx\right)\right)\right)\right]+\mathbb{E}_{\mathbb{P}^{T}}\left[d_{Y}\left(f^{T}\left(K\left(L\left(\bx\right)\right)\right),f^{T}\left(\bx\right)\right)\right].
\end{align*}
Note that, the derivation in $\overset{(1)}{=}$ is due to $L\#\mathbb{P}^{T}=\mathbb{P}^{S}$,
hence gaining 
\begin{align*}
\mathbb{E}_{\mathbb{P}^{T}}\left[d_{Y}\left(h^{S}\left(L\left(\bx\right)\right),f^{S}\left(L\left(\bx\right)\right)\right)\right] & =\mathbb{E}_{\mathbb{P}^{S}}\left[d_{Y}\left(h^{S}\left(\bx\right),f^{S}\left(\bx\right)\right)\right],\\
\mathbb{E}_{\mathbb{P}^{T}}\left[d_{Y}\left(f^{S}\left(L\left(\bx\right)\right),f^{T}\left(K\left(L\left(\bx\right)\right)\right)\right)\right] & =\mathbb{E}_{\mathbb{P}^{S}}\left[d_{Y}\left(f^{S}\left(\bx\right),f^{T}\left(K\left(\bx\right)\right)\right)\right].
\end{align*}
As a consequence, we find that 
\begin{align*}
\mathcal{L}\left(h^{T},f^{T},\mathbb{P}^{T}\right) & \leq\inf_{L,K:L\#\mathbb{P}^{T}=\mathbb{P}^{S},K\#\mathbb{P}^{S}=\mathbb{P}^{T}}\biggl\{\mathbb{E}_{\mathbb{P}^{T}}\left[d_{Y}\left(h^{T}\left(\bx\right),h^{S}\left(L\left(\bx\right)\right)\right)\right]+\mathcal{L}\left(h^{S},f^{S},\mathbb{P}^{S}\right)\\
 & \,\,\,\,\,\,\,\,\,\,\,\,\,\,\,\,\,+\mathbb{E}_{\mathbb{P}^{S}}\left[d_{Y}\left(f^{S}\left(\bx\right),f^{T}\left(K\left(\bx\right)\right)\right)\right]+\mathbb{E}_{\mathbb{P}^{T}}\left[d_{Y}\left(f^{T}\left(K\left(L\left(\bx\right)\right)\right),f^{T}\left(\bx\right)\right)\right]\biggr\}\\
 & \leq\inf_{L:L\#\mathbb{P}^{T}=\mathbb{P}^{S}}\mathbb{E}_{\mathbb{P}^{T}}\left[d_{Y}\left(h^{T}\left(\bx\right),h^{S}\left(L\left(\bx\right)\right)\right)\right]+\inf_{K:K\#\mathbb{P}^{S}=\mathbb{P}^{T}}\mathbb{E}_{\mathbb{P}^{S}}\left[d_{Y}\left(f^{S}\left(\bx\right),f^{T}\left(K\left(\bx\right)\right)\right)\right]\\
 & +\mathcal{L}\left(h^{S},f^{S},\mathbb{P}^{S}\right)+\sup_{L,K:L\#\mathbb{P}^{T}=\mathbb{P}^{S},K\#\mathbb{P}^{S}=\mathbb{P}^{T}}\mathbb{E}_{\mathbb{P}^{T}}\left[d_{Y}\left(f^{T}\left(K\left(L\left(\bx\right)\right)\right),f^{T}\left(\bx\right)\right)\right]\\
 & =\mathcal{W}_{d_{Y}}\left(\mathbb{P}_{h^{S}}^{S},\mathbb{P}_{h^{T}}^{T}\right)+\mathcal{LS}\left(S,T\right)+\mathcal{L}\left(h^{S},f^{S},\mathbb{P}^{S}\right)\\
 & +\sup_{L,K:L\#\mathbb{P}^{T}=\mathbb{P}^{S},K\#\mathbb{P}^{S}=\mathbb{P}^{T}}\mathbb{E}_{\mathbb{P}^{T}}\left[d_{Y}\left(f^{T}\left(K\left(L\left(\bx\right)\right)\right),f^{T}\left(\bx\right)\right)\right].
\end{align*}

\textbf{(ii)} We will show that for all transformation $L$ satisfying
$L\#\mathbb{P}^{T}=\mathbb{P}^{S}$, 
\begin{equation}
\mathbb{E}_{\mathbb{P}^{T}}\norm{f^{T}(\bx)-f^{S}(L(\bx))}_{p}^{p}\geq\norm{p_{Y}^{S}-p_{Y}^{T}}_{p}^{p},
\end{equation}
and then take the infimum of the LHS, which directly leads to the
conclusion. Indeed, by applying Jensen inequality, we find that 
\begin{align*}
\mathbb{E}_{\mathbb{P}^{T}}\norm{f^{T}(\bx)-f^{S}(L(\bx))}_{p}^{p} & =\sum_{y=1}^{M}\mathbb{E}_{\mathbb{P}^{T}}|p^{T}(y|\bx)-p^{S}(y|L(\bx))|^{p}\\
 & \geq\sum_{y=1}^{M}\left|\mathbb{E}_{\mathbb{P}^{T}}(p^{T}(y|\bx)-p^{S}(y|L(\bx)))\right|^{p}\\
 & =\sum_{y=1}^{M}\left|\mathbb{E}_{\mathbb{P}^{T}}p^{T}(y|\bx)-\mathbb{E}_{\mathbb{P}^{S}}p^{S}(y|\bx)\right|^{p}\\
 & =\sum_{y=1}^{M}\left|p_{Y}^{T}(y)-p_{Y}^{S}(y)\right|^{p}\\
 & =\norm{p_{Y}^{S}-p_{Y}^{T}}_{p}^{p}.
\end{align*}
We have thus proved our claim.

\textbf{(iii)} Consider $y$'s as one-hot vectors, i.e., vertices
of the simplex. By the fact that Wasserstein distances on simplex
are no greater than $M$, we have 
\begin{align*}
|\mathcal{W}_{p}^{p}(\mathbb{P}_{f^{S}},\mathbb{P}_{f^{T}})-\mathcal{W}_{p}^{p}(\mathbb{P}_{Y}^{S},\mathbb{P}_{Y}^{T})| & =\left|\sum_{i=0}^{p-1}\mathcal{W}_{p}^{i}(\mathbb{P}_{f^{S}},\mathbb{P}_{f^{T}})\mathcal{W}_{p}^{p-1-i}(\mathbb{P}_{Y}^{S},\mathbb{P}_{Y}^{T})\right||\mathcal{W}_{p}(\mathbb{P}_{f^{S}},\mathbb{P}_{f^{T}})-\mathcal{W}_{p}(\mathbb{P}_{Y}^{S},\mathbb{P}_{Y}^{T})|\\
 & \leq pM^{p-1}|\mathcal{W}_{p}(\mathbb{P}_{f^{S}},\mathbb{P}_{f^{T}})-\mathcal{W}_{p}(\mathbb{P}_{Y}^{S},\mathbb{P}_{Y}^{T})|.
\end{align*}
Besides, by triangle inequalities, 
\begin{align}
|\mathcal{W}_{p}(\mathbb{P}_{f^{S}},\mathbb{P}_{f^{T}})-\mathcal{W}_{p}(\mathbb{P}_{Y}^{S},\mathbb{P}_{Y}^{T})|\leq\mathcal{W}_{p}(\mathbb{P}_{f^{S}},\mathbb{P}_{Y}^{S})+\mathcal{W}_{p}(\mathbb{P}_{f^{T}},\mathbb{P}_{Y}^{T}).\label{eq:key_inequality}
\end{align}
Thus, we only need to prove the claimed bounds for $W_{p}(\mathbb{P}_{f^{S}},\mathbb{P}_{Y}^{S})$
and $W_{p}(\mathbb{P}_{f^{T}},\mathbb{P}_{Y}^{T})$. Because the proofs
are similar for the source and target, in the followings, we drop
the superscript $S,T$ for the ease of notations. We first show that
the mass of $\mathbb{P}_{f}$ concentrates near the vertices, i.e.,
there exists $(\alpha_{y},\epsilon_{y})_{y\in[M]}$ being small numbers
depends on $D$ such that, for $Z\sim f(X)$, 
\begin{align}
0\leq p_{Y}(y)-\mathbb{P}_{f}(\norm{Z-y}_{p}^{p}<\alpha_{y})\leq\epsilon_{y}\quad\forall y\in[M].\label{eq:claim_1}
\end{align}
Indeed, for all $y$, let $B_{y}=\{x:\|\Sigma_{y}^{-1/2}(x-\mu_{y})\|\leq\sqrt{D}\}$.
Denote the dimension of $X$ to be $d$ and $\chi_{d}^{2}$ the Chi-square
distribution with $d$ degree of freedom. We have the following tail
bound: 
\[
\mathbb{P}(X\in B_{y}|y)=P(\chi_{d}^{2}\leq D)\geq1-e^{-(D-2d)/4}.
\]
Hence, we obtain that 
\[
\mathbb{P}(X\in B_{y})=\sum_{y=1}^{M}\mathbb{P}(X\in B_{y}|y)p_{Y}(y)\geq\mathbb{P}(\chi_{d}^{2}\leq D)p_{Y}(y)\geq p_{Y}(y)-\epsilon_{y},
\]
where $\epsilon_{y}=e^{-(D-2d)/4}$. Besides, if $x\in B_{y}$ then
for any $y'\neq y$, by triangle inequalities and the definition of
the operator norm, 
\begin{align*}
\|\Sigma_{y'}^{-1/2}(x-\mu_{y'})\|_{2} & \geq\|\Sigma_{y'}\|_{op}^{-1/2}\|(x-\mu_{y'})\|_{2}\\
 & \geq\|\Sigma_{y'}\|_{op}^{-1/2}(\|(\mu_{y}-\mu_{y'})\|_{2}-\|(x-\mu_{y})\|_{2})\\
 & \geq\|\Sigma_{y'}\|_{op}^{-1/2}(\|(\mu_{y}-\mu_{y'})\|_{2}-\|(\mu_{y}-\mu_{y'})\|_{2}/\sqrt{D})\\
 & \geq D-\sqrt{D},
\end{align*}
where the above inequalites are due to our assumption and the fact
that 
\[
\|x-\mu_{y}\|_{2}\leq\|\Sigma_{y}\|_{op}^{1/2}\|\Sigma_{y}^{-1/2}(x-\mu_{y})\|_{2}\leq\|\Sigma_{y}\|_{op}^{1/2}\sqrt{D}\leq\|\mu_{y}-\mu_{y'}\|_{2}/\sqrt{D}.
\]
Hence, for all $x\in B_{y}$ and $y'\neq y$, we have 
\[
\dfrac{p(x|y')}{p(x|y)}=\dfrac{|\Sigma_{y}|^{1/2}}{|\Sigma_{y'}|^{1/2}}e^{(\|\Sigma_{y}^{-1/2}(x-\mu_{y})\|_{2}^{2}-\|\Sigma_{y'}^{-1/2}(x-\mu_{y'})\|_{2}^{2})/2}\lesssim e^{D-(D-\sqrt{D})^{2}}.
\]
Combining the above inequality with Bayes' rule leads to 
\[
p(y|x)=\dfrac{1}{1+\sum_{y'\neq y}\frac{p(y')p(x|y')}{p(y)p(x|y)}}\geq1-\sum_{y'\neq y}\frac{p(y')p(x|y')}{p(y)p(x|y)}\geq1-\gamma_{y},
\]
where $\gamma_{y}\lesssim e^{-D(D-2\sqrt{D})}$. This means the difference
between labeling function at $x\in B_{y}$ and $y$ is bounded as
follows 
\[
\|f(x)-y\|_{p}^{p}=(1-p(y|x))^{p}+\sum_{y'\neq y}p(y'|x)^{p}\leq2(1-p(y|x))^{p}\leq2\gamma_{y}^{p}.
\]
Choosing $\alpha_{y}=2\gamma_{y}^{p}$, by the fact that $x\in B_{y}$
implies $\|f(x)-y\|_{p}^{p}\leq\alpha_{y}$, we have 
\[
\mathbb{P}_{f}(\|Z-y\|_{p}^{p}\leq\alpha_{y})\geq\mathbb{P}(X\in B_{y})\geq p_{Y}(y)-\epsilon_{y}.
\]
Putting the above results together, we find that 
\[
p_{Y}(y)-\mathbb{P}_{f}(\|Z-y\|_{p}^{p}\leq\alpha_{y})\leq\epsilon_{y}.
\]
Due to the continuity of $\mathbb{P}_{f}$, we can also shrink $\alpha_{y}$
such that the inequality still holds and the left-hand side is positive
and we get our claim~\eqref{eq:claim_1}.

Now let $E_{y}=\{z:\norm{z-y}_{p}^{p}\leq\alpha_{y}\}$ and $D_{y}$
be a set containing $E_{y}$ for all $y\in[M]$ satisfying $\mathbb{P}_{f}(D_{y})=p_{Y}(y)$
and $\{D_{y}\}_{y=1}^{M}$ is a partition of $\Delta^{M-1}$. Let
$p_{f}$ be the density of $\mathbb{P}_{f}$. It can be seen that
\[
\pi(z,y)=p_{f}(z)1[z\in D_{y}]
\]
is the density function of a coupling between $\mathbb{P}_{f}$ and
$\mathbb{P}_{Y}$. Hence, we have the following inequalities: 
\begin{align*}
\mathcal{W}_{p}^{p}(\mathbb{P}_{f},\mathbb{P}_{Y}) & \leq\sum_{y=1}^{M}\int_{z\in\Delta^{M-1}}\norm{y-z}_{p}^{p}\pi(z,y)dz\\
 & \leq\sum_{y=1}^{M}\int_{z\in D_{y}}\norm{y-z}_{p}^{p}\pi(z,y)dz\\
 & \leq\sum_{y=1}^{M}\left(\int_{z\in E_{y}}\norm{y-z}_{p}^{p}\pi(z,y)+\int_{z\in D_{y}\setminus E_{y}}\norm{y-z}_{p}^{p}\pi(z,y)\right)dz\\
 & \leq\sum_{y=1}^{M}\left(\int_{z\in E_{y}}\alpha_{y}\pi(z,y)dz+\int_{z\in D_{y}\setminus E_{y}}M^{p}\pi(z,y)dz\right)dz\\
 & \leq\sum_{y=1}^{M}\alpha_{y}+M^{p}\epsilon_{y}\lesssim e^{-pD(D-2\sqrt{D})}+e^{-(D-2d)/4},
\end{align*}
which goes to 0 exponentially fast when $D$ grows to infinity. Plugging
the above inequality into equation~\eqref{eq:key_inequality}, we
obtain the conclusion of part (iii) of the proposition. 

\textbf{(iv)} Let $\pi_{y}=P(\cdot|y)$ be conditional measure of
$X$ given $Y=y$. By using the law of total probability, we have
\begin{equation}
\mathbb{P}^{S}=\sum_{y=1}^{M}p_{Y}^{S}(y)\pi_{y},\quad\mathbb{P}^{T}=\sum_{y=1}^{M}p_{Y}^{T}(y)\pi_{y}.
\end{equation}
Given the above equations, some simple algebraic transformations would
lead to 
\begin{equation}
\mathbb{P}_{f^{S}}=f^{S}\#\mathbb{P}^{S}=\sum_{y=1}^{M}p_{Y}^{S}(y)f^{S}\#\pi_{y}.
\end{equation}
Now let $q_{yy}=\min\{p_{Y}^{S}(y),p_{Y}^{T}(y)\}$ and choose $(q_{yy'})_{y\neq y'}$
such that $(q_{yy'})_{y=\overline{1,M},y'=\overline{1,M}}$ is a valid
coupling of $p_{Y}^{S}$ and $p_{Y}^{T}$. By the convexity of Wasserstein
distance, we have 
\begin{equation}
W_{p}^{p}(\mathbb{P}_{f^{S}},\mathbb{P}_{f^{T}})\leq\sum_{y,y'}q_{yy'}W_{p}^{p}(f^{S}\#\pi_{y},f^{T}\#\pi_{y'}).
\end{equation}
As the distance between two arbitrary points on the simplex $\Delta^{M-1}$
is not greater than $M$, neither is the Wasserstein distance between
any two measures on $\Delta^{M-1}$. Therefore, 
\begin{equation}
\sum_{y\neq y'}q_{yy'}W_{p}^{p}(f^{S}\#\pi_{y},f^{T}\#\pi_{y'})\leq M^{p}\sum_{y\neq y'}q_{yy'}\leq M^{p}\|p_{Y}^{S}-p_{Y}^{T}\|_{1}.\label{ineq:yy'}
\end{equation}
Besides, we find that
\begin{align*}
\sum_{y=1}^{M}q_{yy}W_{p}^{p}(f^{S}\#\pi_{y},f^{T}\#\pi_{y}) & \leq\sum_{y=1}^{M}p_{Y}^{S}(y)\inf_{\pi\in\Gamma(\pi_{y},\pi_{y})}\int_{\mathcal{X}\times\mathcal{X}}\norm{f^{S}(x)-f^{T}(x')}_{p}^{p}d\pi(x,x')\\
 & \leq\sum_{y=1}^{M}p_{Y}^{S}(y)\int_{\mathcal{X}}\norm{f^{S}(x)-f^{T}(x)}_{p}^{p}d\pi_{y}(x)\\
 & =\sum_{y=1}^{M}p_{Y}^{S}(y)\int_{\mathcal{X}}\norm{f^{S}(x)-f^{T}(x)}_{p}^{p}p(x|y)dx\\
 & =\int_{\mathcal{X}}\norm{f^{S}(x)-f^{T}(x)}_{p}^{p}p^{S}(x)dx\\
 & =\mathbb{E}_{\mathbb{P}^{S}}\norm{f^{S}-f^{T}}_{p}^{p}.
\end{align*}
Similarly, since $q_{yy}\leq p^{T}(y)$ for all $y\in[M]$, we could
also obtain 
\begin{align*}
\sum_{y=1}^{M}q_{yy}W_{p}^{p}(f^{S}\#\pi_{y},f^{T}\#\pi_{y})\leq\mathbb{E}_{\mathbb{P}^{T}}\norm{f^{S}-f^{T}}_{p}^{p}.
\end{align*}
Consequently, 
\begin{equation}
\sum_{y=1}^{M}q_{yy}W_{p}^{p}(f^{S}\#\pi_{y},f^{T}\#\pi_{y})\leq\min\{\mathbb{E}_{\mathbb{P}^{S}}\norm{f^{S}-f^{T}}_{p}^{p},\mathbb{E}_{\mathbb{P}^{T}}\norm{f^{S}-f^{T}}_{p}^{p}\}.\label{ineq:yy}
\end{equation}
Combining equations~\eqref{ineq:yy'} and \eqref{ineq:yy}, we have
the conclusion of part (iv).
\end{proof}

\subsection{Proof of Theorem 5}

Before providing the proof of Theorem 5, we first introduce a lemma
which facilitates our later arguments. 
\begin{lem}
\label{lemma:W_inequality} Let $\mu$ and $\nu$ be two probability
measures on $\mathbb{R}^{d}$. Denote by $\mu_{1}$ ($\nu_{1}$) and
$\mu_{2}$ ($\nu_{2}$) the marginal distributions of $\mu$ ($\nu$)
on the first $k$ dimensions and the last $d-k$ dimensions, respectively,
where $0\leq k\leq d$. We have 
\begin{equation}
\mathcal{W}_{p}^{p}(\mu,\nu)\geq\mathcal{W}_{p}^{p}(\mu_{1},\nu_{1})+\mathcal{W}_{p}^{p}(\mu_{2},\nu_{2})
\end{equation}
\end{lem}

\begin{proof}
Let $\pi$ be the optimal coupling of $\mu$ and $\nu$, and $(X_{1},\dots,X_{d},Y_{1},\dots,Y_{d})$
is a random vector having law $\pi$, we have $(X_{1},\dots,X_{d})\sim\mu,(Y_{1},\dots,Y_{d})\sim\nu$.
Denote by $\pi_{1}$ and $\pi_{2}$ the marginal distribution of $(X_{1},\dots,X_{k},Y_{1},\dots,Y_{k})$
and $(X_{k+1},\dots,X_{d},Y_{k+1},\dots,Y_{d})$, respectively. It
can be seen that $\pi_{1}$ is a coupling of $\mu_{1}$ and $\nu_{1}$
while $\pi_{2}$ is a coupling of $\mu_{2}$ and $\nu_{2}$. Hence,
\begin{align*}
\mathcal{W}_{p}^{p}(\mu,\nu) & =\int_{\mathbb{R}^{d}\times\mathbb{R}^{d}}\|x-y\|_{p}^{p}d\pi(x,y)\\
 & =\int_{\mathbb{R}^{k}\times\mathbb{R}^{k}}\|x_{1}-y_{1}\|_{p}^{p}d\pi_{1}(x,y)+\int_{\mathbb{R}^{d-k}\times\mathbb{R}^{d-k}}\|x_{2}-y_{2}\|_{p}^{p}d\pi_{2}(x,y)\\
 & \geq\mathcal{W}_{p}^{p}(\mu_{1},\nu_{1})+\mathcal{W}_{p}^{p}(\mu_{2},\nu_{2}),
\end{align*}
where $x_{1}\in\mathbb{R}^{k}$ is a vector including the first $k$
coordinates of $x$, whereas $x_{2}\in\mathbb{R}^{d-k}$ contains
the last $d-k$ elements of $x$, and similar definitions apply for
$y_{1}$ and $y_{2}$. 
\end{proof}
Now, we come back to the proof of Theorem 2. 

\paragraph{Proof of Theorem 5:} \paragraph{(i)}

We consider two random vectors $(X_{1},\dots,X_{M})\sim\mathbb{P}_{f^{S}}$
and $(Y_{1},\dots,Y_{M})\sim\mathbb{P}_{f^{T}}$. Recall that $\mathbb{Q}_{S}$
and $\mathbb{Q}_{T}$ are the marginal distributions of $\mathbb{P}_{f^{S}}$
and $\mathbb{P}_{f^{T}}$ on their first $C-1$ dimensions, respectively,
while $\mathbb{Q}_{S\setminus T}$ denotes the marginal of $\mathbb{P}_{f_{S}}$
on the space of variables having labels in the set $\mathcal{Y}_{S}\setminus\mathcal{Y}_{T}$.
Since $\mathcal{Y}_{T}\subset\mathcal{Y}_{S}$, we have 
\begin{enumerate}
\item $(X_{1},\ldots,X_{C-1})\sim\mathbb{Q}_{S}$, $(Y_{1},\ldots,Y_{C-1})\sim\mathbb{Q}_{T}$; 
\item $(X_{C+1},\ldots,X_{M})\sim\mathbb{Q}_{S\setminus T}$, $(Y_{C+1},\ldots,Y_{M})\sim\delta_{\mathbf{0}_{M-C}}$, 
\end{enumerate}
where $\mathbf{0}_{m}$ denotes the zero vector in $\mathbb{R}^{m}$
for $m\in\mathbb{N}$. 

Let $\mu$ be the distribution of $(X_{1},\dots,X_{C-1},X_{C+1},\dots,X_{M})$,
$\nu$ the distribution of $(Y_{1},\dots,Y_{C-1},$\\$Y_{C+1},\dots,Y_{M})$.
As the simplex $\Delta^{M-1}$ is an $M-1$ dimensional manifold in
$\mathbb{R}^{M}$, the Wasserstein distance between $\mathbb{P}_{f^{S}}$
and $\mathbb{P}_{f^{T}}$ can be written as 
\[
\mathcal{W}_{p}^{p}(\mathbb{P}_{f^{S}},\mathbb{P}_{f^{T}})=\inf_{\gamma\in\Gamma(\mu,\nu)}\int_{\mathbb{R}^{C-1}\times\mathbb{R}^{C-1}}\sum_{i=1}^{M}|x_{i}-y_{i}|^{p}d\gamma(x_{\widehat{C}},y_{\widehat{C}}),
\]
where $x_{\widehat{C}}=(x_{1},\dots,x_{C-1},x_{C+1},\dots,x_{M}),y_{\widehat{C}}=(y_{1},\dots,y_{C-1},y_{C+1},\dots,y_{M})$
and $x_{C}:=1-\sum_{k\neq C}x_{k},$\\$y_{C}:=1-\sum_{k\neq C}y_{k}$.
As $|x_{C}-y_{C}|^{p}\geq0$, it can be deduced that 
\[
\mathcal{W}_{p}^{p}(\mathbb{P}_{f^{S}},\mathbb{P}_{f^{T}})\geq\mathcal{W}_{p}^{p}(\mu,\nu).
\]
Besides, according to Lemma \ref{lemma:W_inequality}, we get 
\begin{align*}
\mathcal{W}_{p}^{p}(\mu,\nu) & \geq\mathcal{W}_{p}^{p}(\mathbb{Q}_{S},\mathbb{Q}_{T})+\mathcal{W}_{p}^{p}(\mathbb{Q}_{S\setminus T},\delta_{\mathbf{0}_{M-C}})=\mathcal{W}_{p}^{p}(\mathbb{Q}_{S},\mathbb{Q}_{T})+\mathbb{E}_{X\sim\mathbb{Q}_{S\setminus T}}\left[\norm X_{p}^{p}\right].
\end{align*}
Putting the above two inequalities together, we obtain the conclusion
that 
\begin{equation}
\mathcal{W}_{p}^{p}(\mathbb{P}_{f^{S}},\mathbb{P}_{f^{T}})\geq\mathcal{W}_{p}^{p}(\mathbb{Q}_{S},\mathbb{Q}_{T})+\mathbb{E}_{X\sim\mathbb{Q}_{S\setminus T}}\left[\norm X_{p}^{p}\right].
\end{equation}

\paragraph{(ii)}

Part (ii) is done similarly to part (i). Therefore, it is omitted.

\paragraph{(iii)}

Let $(X_{1},\ldots,X_{M})\sim\mathbb{P}_{f^{S}}$ and $(Y_{1},\ldots,Y_{M})\sim\mathbb{P}_{f^{T}}$.
Assume that $\mathcal{Y}_{S}=\{1,\ldots,C,C+1,\ldots,D\}$ and $\mathcal{Y}_{T}=\{1,\ldots,C,D+1,\ldots,M\}$.
It follows from the definitions of $\mathbb{Q}_{S},\mathbb{Q}^{T},\mathbb{Q}_{S\setminus T}$
and $\mathbb{Q}_{T\setminus S}$ that 
\begin{enumerate}
\item $(X_{1},\ldots,X_{C-1})\sim\mathbb{Q}_{S}$, $(Y_{1},\ldots,Y_{C-1})\sim\mathbb{Q}_{T}$; 
\item $(X_{C+1},\ldots,X_{D})\sim\mathbb{Q}_{S\setminus T}$, $(Y_{C+1},\ldots,Y_{D})\sim\delta_{\mathbf{0}_{D-C}}$; 
\item $(X_{D+1},\ldots,X_{M})\sim\delta_{\mathbf{0}_{M-D}}$, $(Y_{D+1},\ldots,Y_{M})\sim\mathbb{Q}_{T\setminus S}$. 
\end{enumerate}
Let $\mu$ be the distribution of $(X_{1},\dots,X_{C-1},X_{C+1},\dots,X_{M})$,
and $\nu$ be the distribution of $(Y_{1},\dots,Y_{C-1},$\\$Y_{C+1},\dots,Y_{M})$.
By using the same arguments as in part (i), we get $\mathcal{W}_{p}^{p}(\mathbb{P}_{f^{S}},\mathbb{P}_{f^{T}})\geq\mathcal{W}_{p}^{p}(\mu,\nu)$.
Next, applying Lemma \ref{lemma:W_inequality} twice, we obtain 
\begin{align*}
\mathcal{W}_{p}^{p}(\mu,\nu) & \geq\mathcal{W}_{p}^{p}(\mathbb{Q}_{S},\mathbb{Q}_{T})+\mathcal{W}_{p}^{p}(\mathbb{Q}_{S\setminus T},\delta_{\mathbf{0}_{D-C}})+\mathcal{W}_{p}^{p}(\delta_{\mathbf{0}_{M-D}},\mathbb{Q_{T\setminus S}})\\
 & =\mathcal{W}_{p}^{p}(\mathbb{Q}_{S},\mathbb{Q}_{T})+\mathbb{E}_{Y\sim\mathbb{Q}_{S\setminus T}}\left[\norm Y_{p}^{p}\right]+\mathbb{E}_{X\sim\mathbb{Q}_{T\setminus S}}\left[\norm X_{p}^{p}\right].
\end{align*}
As a consequence, we have proved our claim in part (iii). 

\subsection{Proofs of claims in paragraph "Remark on shifting term"}
\begin{lem}
\label{lem:shifting_term} The following holds: 
\begin{align*}
    \mathcal{W}_{d_{X}}(g\#\mathbb{P}^{S},g\#\mathbb{P}^{T})=\mathcal{W}_{d_{X}}(\mathbb{P}_{h^{S}}^{S},\mathbb{P}_{h^{T}}^{T}).
\end{align*}
\end{lem}

\begin{proof}
\textbf{(i)} Applying the same arguments as in Proposition 1, we have
\begin{equation}
\mathcal{W}_{d_{X}}(\mathbb{P}_{h^{S}}^{S},\mathbb{P}_{h^{T}}^{T})=\inf_{H:H\#\mathbb{P}_{h^{T}}^{T}=\mathbb{P}_{h^{S}}^{S}}\mathbb{E}_{\mathbb{P}_{h^{T}}^{T}}\left[d_{X}[g(\bx),g(H_{1}(\bx))]\right]=\inf_{L:L\#\mathbb{P}^{T}=\mathbb{P}^{S}}\mathbb{E}_{\mathbb{P}^{T}}\left[d_{X}(g(\bx),g(L(\bx)))\right],
\end{equation}
where $H((\bx,h^{T}(\bx)))=(H_{1}((\bx,h^{T}(\bx))),H_{2}((\bx,h^{T}(\bx))))$
such that $H_{1}((\bx,h^{T}(\bx)))\in\mathcal{X}^{S}$. So now we
only need to prove that 
\begin{equation}
\mathcal{W}_{d_{X}}(g\#\mathbb{P}^{S},g\#\mathbb{P}^{T})=\inf_{L:L\#\mathbb{P}^{T}=\mathbb{P}^{S}}\mathbb{E}_{\mathbb{P}^{T}}\left[d_{X}(g(\bx),g(L(\bx)))\right].\label{claim:lemishifting}
\end{equation}
Due to the equivalence of Monge and Kantorovich problem, we can write
the RHS as 
\begin{equation}
\inf_{L:L\#\mathbb{P}^{T}=\mathbb{P}^{S}}\mathbb{E}_{\mathbb{P}^{T}}\left[d_{X}(g(\bx),g(L(\bx)))\right]=\inf_{\gamma\in\Gamma(\mathbb{P}^{S},\mathbb{P}^{T})}\mathbb{E}_{(\bx^{S},\bx^{T})\sim\gamma}\left[d_{X}(g(\bx^{S}),g(\bx^{T}))\right].
\end{equation}
To prove Eq.~\eqref{claim:lemishifting}, we will show that RHS is
not less than LHS and inversely. Indeed, for any coupling $\gamma\in\Gamma(\mathbb{P}^{S},\mathbb{P}^{T})$,
we have $\gamma'=(g,g)\#\gamma$ as a coupling of $(g\#\mathbb{P}^{S},g\#\mathbb{P}^{T})$,
therefore 
\[
\mathbb{E}_{(\bx^{S},\bx^{T})\sim\gamma}\left[d_{X}(g(\bx^{S}),g(\bx^{T}))\right]=\mathbb{E}_{(\mathbf{y}^{S},\mathbf{y}^{T})\sim\gamma'}\left[d_{X}(\mathbf{y}^{S},\mathbf{y}^{T})\right]\geq\inf_{\gamma'\in\Gamma(g\#\mathbb{P}^{S},g\#\mathbb{P}^{T})}\mathbb{E}_{(\mathbf{y}^{S},\mathbf{y}^{T})\sim\gamma'}\left[d_{X}(\mathbf{y}^{S},\mathbf{y}^{T})\right].
\]
Taking the infimum with respect to $\gamma$, 
\begin{equation}
\inf_{\gamma\in\Gamma(\mathbb{P}^{S},\mathbb{P}^{T})}\mathbb{E}_{(\bx^{S},\bx^{T})\sim\gamma}\left[d_{X}(g(\bx^{S}),g(\bx^{T}))\geq\mathcal{W}_{d_{X}}(g\#\mathbb{P}^{S},g\#\mathbb{P}^{T})\right].\label{claim:RHSgeqLHS}
\end{equation}
Conversely, thanks to Lemma \ref{lem:my_gluing}, for each coupling
$\gamma'$ of $(g\#\mathbb{P}^{S},g\#\mathbb{P}^{S})$, there exists
a coupling $\gamma$ of $(\mathbb{P}^{S},\mathbb{P}^{S})$ such that
$\gamma'=(g,g)\#\gamma$, which deduces that 
\[
\mathbb{E}_{(\mathbf{y}^{S},\mathbf{y}^{T})\sim\gamma'}\left[d_{X}(\mathbf{y}^{S},\mathbf{y}^{T})\right]=\mathbb{E}_{(\bx^{S},\bx^{T})\sim\gamma}\left[d_{X}(g(\bx^{S}),g(\bx^{T}))\right]\geq\inf_{\gamma\in\Gamma(\mathbb{P}^{S},\mathbb{P}^{T})}\mathbb{E}_{(\bx^{S},\bx^{T})\sim\gamma}\left[d_{X}(g(\bx^{S}),g(\bx^{T}))\right].
\]
Taking the infimum with respect to $\gamma'$, 
\begin{equation}
\mathcal{W}_{d_{X}}(g\#\mathbb{P}^{S},g\#\mathbb{P}^{T})\geq\inf_{\gamma\in\Gamma(\mathbb{P}^{S},\mathbb{P}^{T})}\mathbb{E}_{(\bx^{S},\bx^{T})\sim\gamma}\left[d_{X}(g(\bx^{S}),g(\bx^{T}))\right].\label{claim:LHSgeqRHS}
\end{equation}
Inequalities~\eqref{claim:RHSgeqLHS} and \eqref{claim:LHSgeqRHS}
together imply equation \eqref{claim:lemishifting} and finish the proof.

\end{proof}

\section{Additional experiment results \label{sec:additional_exps}}

\subsection{More analysis about rationale of the terms used in the objective function of LDROT}\label{more_analysis}
\textbf{Comparing to DeepJDOT} \cite{damodaran2018deepjdot}. The $\mathcal{W}_{d}(\mathbb{P}_{h^{S}}^{S},\mathbb{P}_{h^{T}}^{T})$ with $d=\lambda d_X+d_Y$ was investigated in DeepJDOT \cite{damodaran2018deepjdot}. However, ours
is different from that work in some aspects: (i) \textit{similarity based dynamic weighting}, (ii) \textit{clustering loss for enforcing clustering assumption for target classifier}, and (iii) \textit{entropic dual form} for training rather than Sinkhorn as in \cite{damodaran2018deepjdot}.

This objective function consists of three losses: (i) standard loss
$\mathcal{L}^{S}$, (ii) shifting loss $\mathcal{L}^{shift}$, and
(iii) clustering loss $\mathcal{L}^{clus}$. The standard loss $\mathcal{L}^{S}$
is trained on the labeled source domain. The shifting loss aims to
reduce both data and label shift simultaneously on the latent space
by minimizing $\mathcal{W}_{d}\left(\mathbb{P}_{h^{S}}^{S},\mathbb{P}_{h^{T}}^{T}\right)$
where $d=\lambda d_{X}+d_{Y}$ for which $d_{X}$ is data distance
on the latent space and $d_{Y}$ is distance on the label simplex.
Based on the theory developed, we demonstrate that minimizing this term
helps to reduce both data shift (i.e., $\mathcal{W}_{d_{X}}\left(g_{\#}\mathbb{P}^{S},g_{\#}\mathbb{P}^{T}\right)$
and label shift (i.e., $\mathcal{W}_{d_{Y}}\left(\mathbb{P}_{h^{S}}^{S},\mathbb{P}_{h^{T}}^{T}\right)$).
Finally, the $\mathcal{L}^{clus}$ assists us in enforcing the clustering
assumption to boost the generalization of the target classifier $\bar{h}^{T}$.
By enforcing the clustering assumption, classifiers are encouraged
to preserve the cluster structure and give the same predictions for data
representations in the same cluster. It appears that when pushing
target latent toward source representations via minimizing $\mathcal{W}_{d}\left(\mathbb{P}_{h^{S}}^{S},\mathbb{P}_{h^{T}}^{T}\right)$,
source and target representations tend to group in clusters, hence
we can strengthen and boost generalization of target classifier $\bar{h}^{T}$
by enforcing it to preserve the predictions in the same clusters.

In addition, we propose a dynamic weighting for $\lambda$ using
similarities of pairs between source and target examples. For our
similarity-based weighting distance, we base on pre-trained similarities
to decide if we push more or fewer pairs of source and target latent
representations together to reduce data and label shifts more efficiently.
Definitely, if we can push groups of source and target representations
with the same labels together more efficiently, we can certainly reduce
both data and label shifts simultaneously. 

\subsection{Data preparation and pre-processing \label{subsec:data-preparation-and-preprocessing}}

\textbf{Digits}. We resize the resolution of each sample in the dataset
to $32\times32$, and normalize the value of each pixel to the range
of $[-1,1]$.

\textbf{Office-31, Office-Home, }and\textbf{ ImageCLEF-DA}. We use
2048-dimensional features extracted from ResNet-50 \cite{he2016resnet}
pretrained on ImageNet.

\subsection{Algorithm of LDROT}

We present peusocode of LDROT in Algorithm \ref{alg:ldrot}.

\begin{algorithm}[h]
\begin{algorithmic}

\REQUIRE A source batch{\small{}{} $\mathcal{B}^{S}=\left\{ \left(\bx_{i}^{S},y_{i}^{S}\right)\right\} _{i=1}^{Mb}$},
a target batch{\small{}{} $\mathcal{B}^{T}=\left\{ \left(\bx_{j}^{T},y_{j}^{T}\right)\right\} _{j=1}^{b}$},
$b$ is the batch size.

\ENSURE Classifier $\bar{h^{*}}^{S}\bar{h^{*}}^{T}$, generator $g^{*}$.

\STATE Evaluate $\left\{ \boldsymbol{r}_{i}^{S}\right\} _{i=1}^{Mb}$
and $\left\{ \boldsymbol{r}_{j}^{T}\right\} _{j=1}^{b}$ based on
$\mathcal{B}^{S}$ and $\mathcal{B}^{T}$ respectively.

\STATE Compute the weights $\bw_{ij}$ ibased on $\left\{ \boldsymbol{r}_{i}^{S}\right\} _{i=1}^{Mb}$
and $\left\{ \boldsymbol{r}_{j}^{T}\right\} _{j=1}^{b}$ .

\FOR {number of training iterations}

\FOR {$k$ steps}

\STATE Update $\phi$ according to Eq. (13).

\ENDFOR

\STATE Update $\bar{h}^{S},\bar{h}^{T}$ and $g$ according to Eq.
(10).

\ENDFOR

\end{algorithmic}

\caption{Pseudocode for training our LDROT.\label{alg:ldrot}}
\end{algorithm}

\subsection{Network architecture}

There are 2 types of the architecture described in Table \ref{tab:model-architecture},
which are small (\textbf{S}) and large (\textbf{L}) networks. We use
\textbf{L} network for \emph{Digits} and \textbf{S} network for the
other datasets. Additionally, excluding dense layers in the $\phi$
network, we add the batch normalization layers on top of convolutional
and dense layers to reduce the overfitting problem. Finally, we implement
our LDROT in Python using TensorFlow (version 1.9.0) \cite{abadi2016tensorflow},
an open-source software library for Machine Intelligence developed
by the Google Brain Team. All experiments are run on a computer with
an NVIDIA Tesla V100 SXM2 with 16 GB memory.

\begin{table}
\centering{}\caption{Small and large networks for LDROT. The Leaky ReLU (lReLU) parameter
$a$ is set to 0.1.\label{tab:model-architecture}}
\begin{tabular}{>{\raggedright}m{1.8cm}cc}
\hline 
Architecture  & S  & L\tabularnewline
\hline 
\multicolumn{1}{c}{Input size} & $2048$  & $32\times32\times3$\tabularnewline
\hline 
Generator $g$  &  & instance normalization\tabularnewline
 & $256$ dense, ReLU  & $3\times3$ conv. 64 lReLU\tabularnewline
 & dropout, $p=0.5$  & $3\times3$ conv. 64 lReLU\tabularnewline
 & Gaussian noise, $\sigma$ = 1  & $3\times3$ conv. 64 lReLU\tabularnewline
 &  & $2\times2$ max-pool, stride 2\tabularnewline
 &  & dropout, $p=0.5$\tabularnewline
 &  & Gaussian noise, $\sigma$ = 1\tabularnewline
 &  & $3\times3$ conv. 64 lReLU\tabularnewline
 &  & $3\times3$ conv. 64 lReLU\tabularnewline
 &  & $3\times3$ conv. 64 lReLU\tabularnewline
 &  & $2\times2$ max-pool, stride 2\tabularnewline
 &  & dropout, $p=0.5$\tabularnewline
 &  & Gaussian noise, $\sigma$ = 1\tabularnewline
 &  & $3\times3$ conv. 8 lReLU\tabularnewline
 &  & $2\times2$ max-pool, stride 2\tabularnewline
\hline 
Classifier $\bar{h}^{S},\bar{h}^{T}$  & \emph{\#classes} dense, softmax  & $3\times3$ conv. 8 lReLU\tabularnewline
 &  & $3\times3$ conv. 8 lReLU\tabularnewline
 &  & $3\times3$ conv. 8 lReLU\tabularnewline
 &  & global average pool\tabularnewline
 &  & \emph{\#classes} dense, softmax\tabularnewline
\hline 
$\phi$ & 1 dense, linear  & 100 dense, ReLU\tabularnewline
 &  & 1 dense, linear\tabularnewline
\hline 
\end{tabular}
\end{table}

\subsection{Implementation details}

We first present our procedure to compute the weights $\bw_{ij}$,
which derives from the feature extraction process. For \emph{Digits},
we design a network to train from scratch on labeled source examples.
Then source and target features are extracted via this pretrained
model. For the other datasets, we use extracted ResNet-50 features
\cite{he2016resnet} and design a small network to train LDROT. During
training, the features are used for first computing pairwise similarity
scores and the weights $\bw_{ij}$ after that.

For LDROT, we find that some hyper-parameters contributes substantially
to the model performance, namely $\tau$ and $\epsilon$. The temperature
parameter $\tau$, which contributes to sharpening and contrasting
the weights $\bw_{ij}$, is fixed to $0.5$. Tweaking the regularization
rate $\epsilon$ is vital for scaling $\phi_{\epsilon}^{c}\left(\bx\right)$
and we select $\epsilon=0.1$. For trade-off parameters $\alpha,\beta$,
we choose $\alpha=0.1$ and $\beta=0.5$ for all settings. We apply
Adam optimizer \cite{kingma2014adam} ($\beta_{1}=0.5,\beta_{2}=0.999$)
with Polyak averaging. The learning rate is set to $0.001$ and $0.0001$
for Digits and the other datasets respectively. Additionally, in nature,
our model solves the minimax optimization problem (see Eq. (13) in
the main paper) in which $\phi$ and $\bar{h}^{S},\bar{h}^{T},g$
are updated sequentially in each iteration with five times for $\phi$
and one time for $\bar{h}^{S},\bar{h}^{T},g$. Finally, we use the
cosine distance for $d_{X}$ and Kullback-Leibler (KL) divergence
for $d_{Y}$.

\subsection{Additional results for Digits and ImageCLEF-DA}

\begin{table}[t]
\centering{}\caption{Classification accuracy (\%) on Digits dataset for unsupervised domain
adaptation.\label{tab:digit-image-datasets}}
\begin{tabular}{ccccc}
\hline 
Method & S$\rightarrow$M  & M$\rightarrow$U  & U$\rightarrow$M  & Avg\tabularnewline
\hline 
DANN \cite{Ganin2015}  & 85.5  & 84.9  & 86.3  & 85.6\tabularnewline
ADDA \cite{tzeng2017ADDA}  & 89.2  & 85.4  & 96.5  & 90.4\tabularnewline
DeepCORAL \cite{sun2016coral}  & 88.3  & 84.1  & 93.6  & 88.7\tabularnewline
CDAN \cite{long2018cdan}  & 89.2  & 95.6  & 98.0  & 94.3\tabularnewline
TPN \cite{pan2019tpn}  & 93.0  & 92.1  & 94.1  & 93.1\tabularnewline
rRevGrad+CAT \cite{deng2019cluster}  & 98.8  & 94.0  & 96.0  & 96.3\tabularnewline
SWD\textbf{ }\cite{chenyu2019swd}  & 98.9  & 98.1  & 97.1  & 98.0\tabularnewline
DeepJDOT \cite{damodaran2018deepjdot}  & 96.7  & 95.7  & 96.4  & 96.3\tabularnewline
DASPOT\textbf{ }\cite{yujia2019onscalable}  & 96.2  & 97.5  & 96.5  & 96.7\tabularnewline
ETD \cite{li2020enhanceOT}  & 97.9  & 96.4  & 96.3  & 96.9\tabularnewline
RWOT \cite{xu2020reliable}  & 98.8  & \textbf{98.5}  & 97.5  & 98.3\tabularnewline
\hline 
\textbf{LDROT}  & \textbf{99.0}  & 98.2  & \textbf{99.1}  & \textbf{98.8}\tabularnewline
\hline 
\end{tabular}
\end{table}
\begin{table}
\centering{}\caption{Classification accuracy (\%) on ImageCLEF-DA dataset for unsupervised
domain adaptation (ResNet-50).\label{tab:imageclef-da}}
\begin{tabular}{cccccccc}
\hline 
Method  & I$\rightarrow$P  & P$\rightarrow$I  & I$\rightarrow$C  & C$\rightarrow$I  & C$\rightarrow$P  & P$\rightarrow$C  & Avg\tabularnewline
\hline 
ResNet-50 \cite{he2016resnet}  & 74.8  & 83.9  & 91.5  & 78.0  & 65.5  & 91.2  & 80.7\tabularnewline
DeepCORAL \cite{sun2016coral}  & 75.1  & 85.5  & 92.0  & 85.5  & 69.0  & 91.7  & 83.1\tabularnewline
DANN \cite{ganin2016domain}  & 75.0  & 86.0  & 96.2  & 87.0  & 74.3  & 91.5  & 85.0\tabularnewline
ADDA \cite{tzeng2017ADDA}  & 75.5  & 88.2  & 96.5  & 89.1  & 75.1  & 92.0  & 86.0\tabularnewline
CDAN \cite{long2018cdan}  & 77.7  & 90.7  & 97.7  & 91.3  & 74.2  & 94.3  & 87.7\tabularnewline
TPN \cite{pan2019tpn}  & 78.2  & 92.1  & 96.1  & 90.8  & 76.2  & 95.1  & 88.1\tabularnewline
SymNets \cite{zhang2019symsnet}  & 80.2  & 93.6  & 97.0  & 93.4  & 78.7  & 96.4  & 89.9\tabularnewline
SAFN \cite{xu2019sfan}  & 79.3  & 93.3  & 96.3  & 91.7  & 77.6  & 95.3  & 88.9\tabularnewline
rRevGrad+CAT \cite{deng2019cluster}  & 77.2  & 91.0  & 95.5  & 91.3  & 75.3  & 93.6  & 87.3\tabularnewline
DeepJDOT \cite{damodaran2018deepjdot}  & 77.5  & 90.5  & 95.0  & 88.3  & 74.9  & 94.2  & 86.7\tabularnewline
ETD \cite{li2020enhanceOT}  & 81.0  & 91.7  & \textbf{97.9}  & 93.3  & 79.5  & 95.0  & 89.7\tabularnewline
RWOT \cite{xu2020reliable}  & 81.3  & 92.9  & \textbf{97.9}  & 92.7  & 79.1  & 96.5  & 90.0\tabularnewline
\hline 
\textbf{LDROT}  & \textbf{81.7}  & \textbf{96.7}  & 97.5  & \textbf{94.2}  & \textbf{80.4}  & \textbf{96.7}  & \textbf{91.2}\tabularnewline
\hline 
\end{tabular}
\end{table}
We additionally present the experimental results for Digits and ImageCLEF-DA
datasets in Tables \ref{tab:digit-image-datasets} and \ref{tab:imageclef-da}.
It can be observed that LDROT also outperforms the baselines on these
datasets.\textbf{ }

\subsection{Ablation studies \label{subsec:Ablation-studies}}

\textbf{Effects of the label shift term $d_{Y}$ and the weights $\bw_{ij}$}:\label{dy-w}
To answer the question of \emph{how will the model performance be
affected when the weights }\textbf{\emph{$\bw_{ij}$ }}\emph{or }\textbf{\emph{$d_{Y}$
}}\emph{in Eq. (\ref{eq:entropic_WS-1}) is removed?}, we evaluate
our model on four different settings: LDROT without both the weights
\textbf{$\bw_{ij}$} and the label shift\textbf{ }$d_{Y}(\cdot,\cdot)$
(LDROT$-wd$), without the weights \textbf{$\bw_{ij}$} (LDROT$-w$),
without minimizing the label shift\textbf{ }$d_{Y}(\cdot,\cdot)$
(LDROT$-d$) and a complete model (LDROT$+wd$). The results on \emph{Office-Home
}in Table \ref{tab:effect-dy-wij} dedicate the significance of \textbf{$\bw_{ij}$
}and\textbf{ $d_{Y}(\cdot,\cdot)$}, where these components remarkably
contribute to reducing the data and label shifts with 4.1\% improvements
on average.

\begin{figure}[!t]
\centering{}\includegraphics[width=0.6\textwidth]{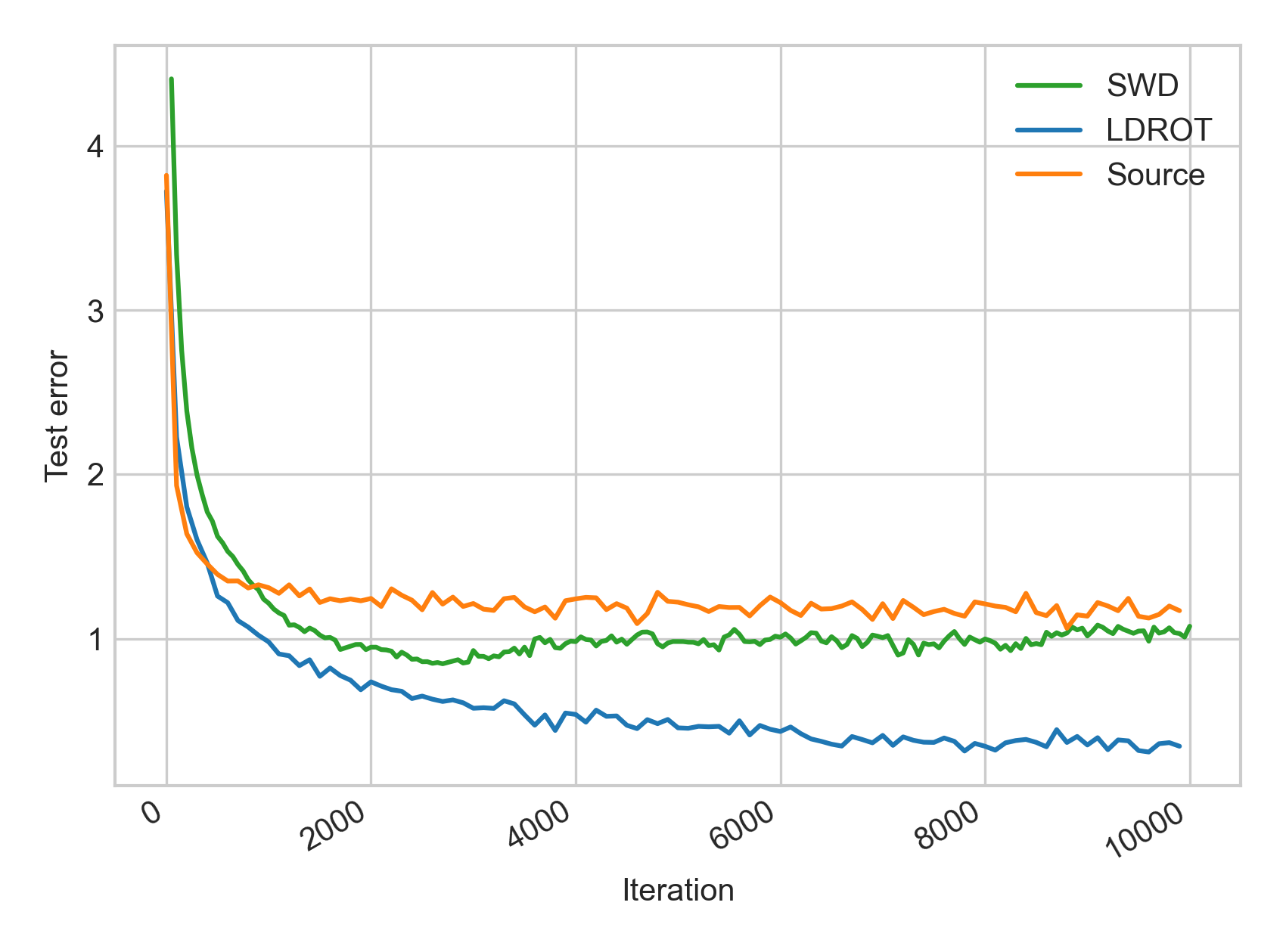}\vspace{-3mm}
 \caption{Comparision of convergence performance between LDROT and other approaches
on the transfer task \textbf{A}\textrightarrow \textbf{D}.\label{fig:convergence}}
\end{figure}

\begin{table}[h]
\begin{centering}
\caption{Accuracy (\%) of ablation study on \emph{Office-Home}.\label{tab:effect-dy-wij}}
\begin{tabular}{ccccccc}
\hline 
Method  & Ar$\rightarrow$Cl  & Cl$\rightarrow$Ar  & Cl$\rightarrow$Rw  & Pr$\rightarrow$Ar  & Rw$\rightarrow$Ar  & Avg\tabularnewline
\hline 
LDROT$-wd$  & 51.7  & 62.6  & 78.5  & 63.3  & 66.1  & 64.4\tabularnewline
LDROT$-w$  & 56.7  & 62.3  & 79.3  & 64.0  & 66.8  & 65.8\tabularnewline
LDROT$-d$  & 55.4  & 65.5  & 80.1  & 65.0  & 68.4  & 66.9\tabularnewline
LDROT$+wd$  & \textbf{57.4}  & \textbf{67.2}  & \textbf{80.7}  & \textbf{66.5}  & \textbf{70.9}  & \textbf{68.5}\tabularnewline
\hline 
\end{tabular}
\par\end{centering}
\centering{} 
\end{table}

\paragraph{Rationale of weigh strategy.}

In Figure \ref{fig:sim_scores}, we visualize the similarity scores
of a randomly selected target example $\bx_{j}^{T}$ in a batch w.r.t.
source examples $\bx_{i}^{S}$ using the source and target domains
\emph{Amazon} and \emph{Dslr} of \emph{Office-31 }dataset, respectively.
The orange points represent the similarity scores for the same class,
while the blue points represent those for different classes. It is
evident that the orange values tend to bigger than the blue ones except
in some outlier cases, hence if we choose $\mu_{i}$ as indicated,
we can separate well the orange and blue values.

\begin{figure}[h]
\centering{}\includegraphics[width=0.5\textwidth]{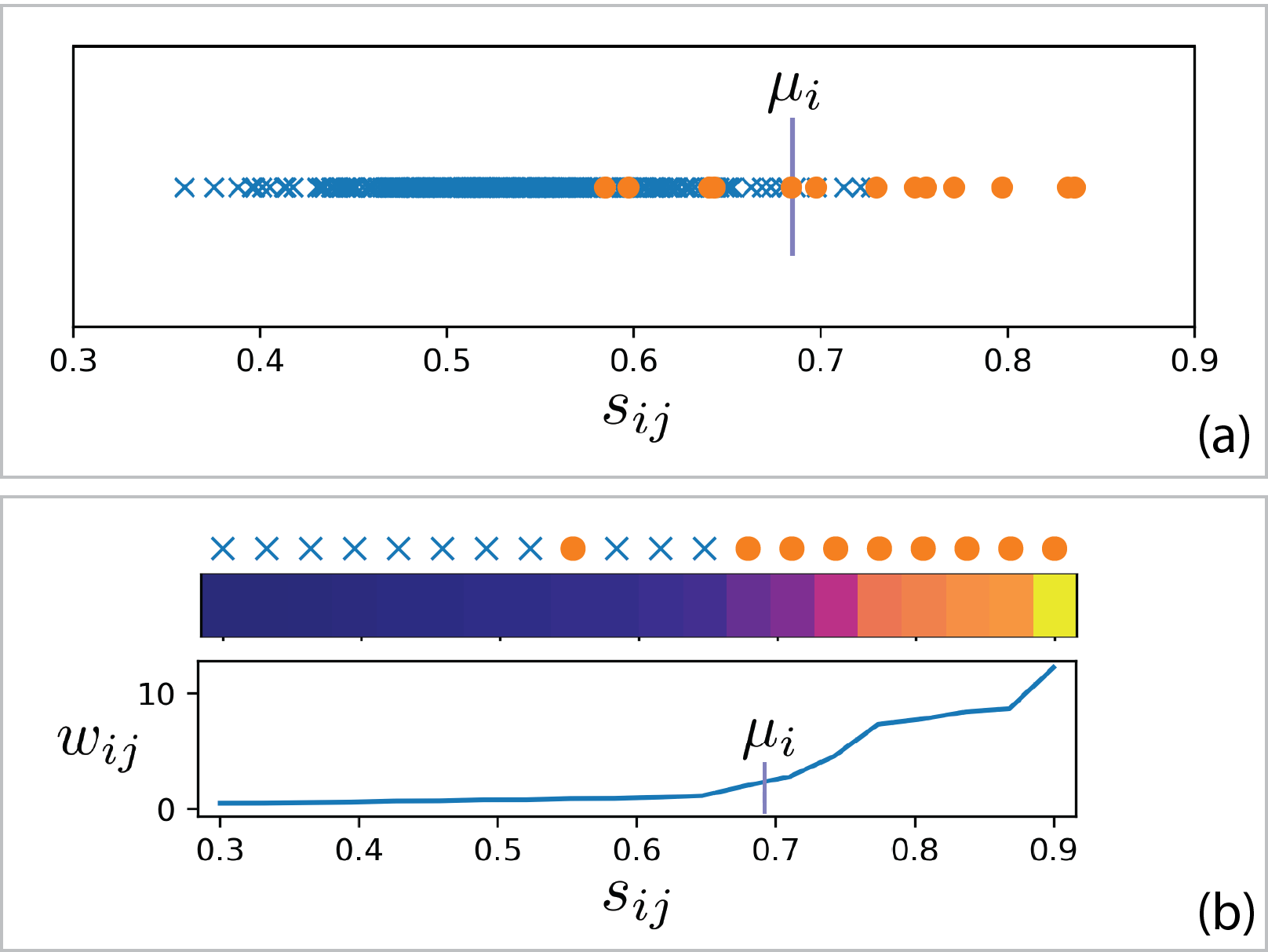}\caption{(a) 1D visualization of finding an appropriate $\mu_{i}$. It is able
to split same-label pairs (orange points) and different-label pairs
(blue points) if $\mu_{i}$ is set to $\frac{M-1}{M}$-percentile
of this array. (b) To observe the weight values $w_{ij}$ of those
pairs, we randomly picked a target example to compute similarity scores
with 20 representative source points in a batch, and then sort them
in ascending order. After computing the weights, the figures for the
same-label pairs tend to be much higher than that for different-label
pairs. A heat-map color is used to represent the weights magnitude
(the brighter means higher value).\label{fig:sim_scores}}
\end{figure}
\begin{figure}[h]
\centering \subfloat[ResNet.\label{fig:ResNet-AD}]{\includegraphics[width=0.4\textwidth]{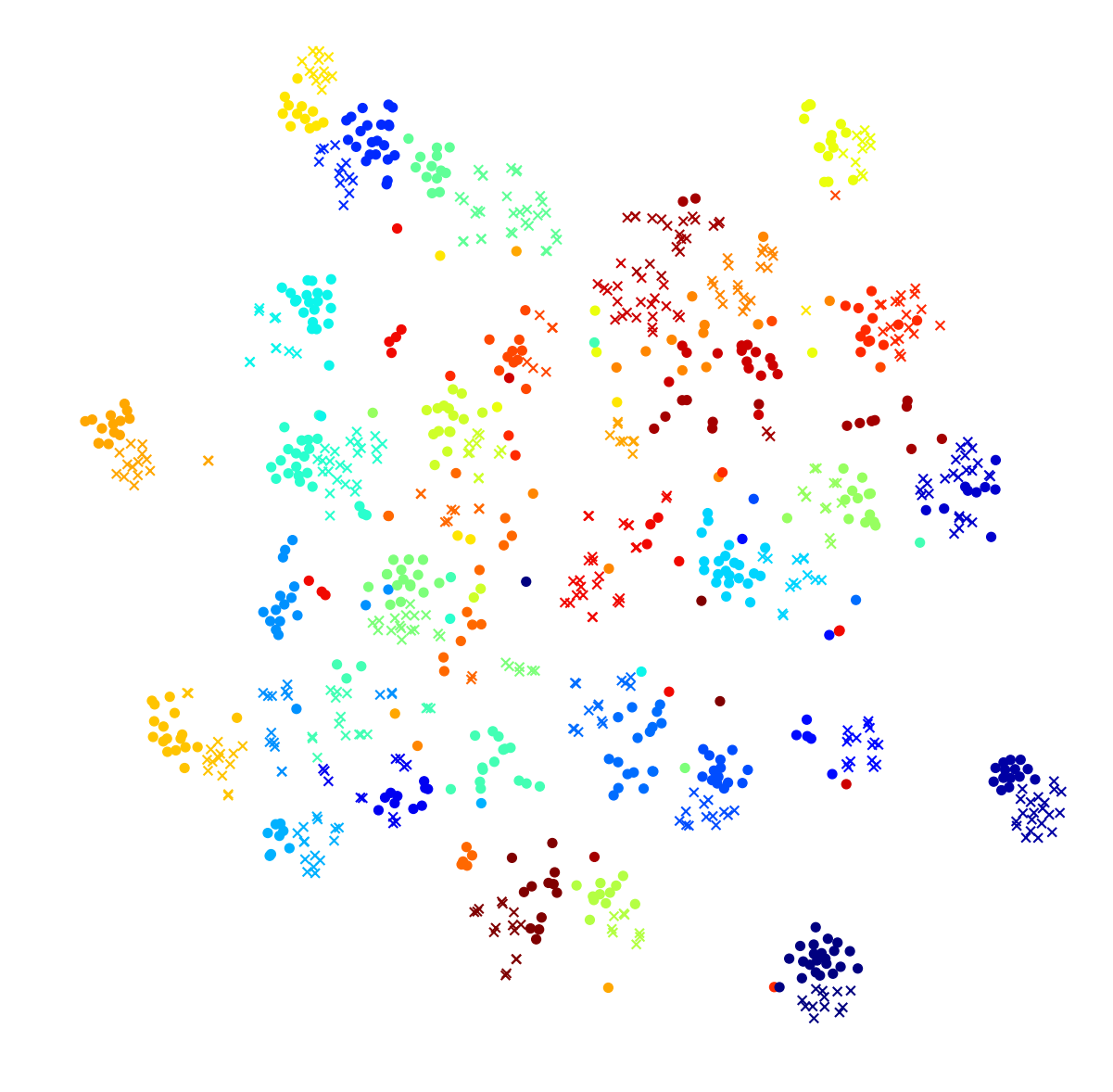}

}~~~~~~~~~~~~~~ \subfloat[LDROT.\label{fig:LDROT-AD}]{}{\includegraphics[width=0.4\textwidth]{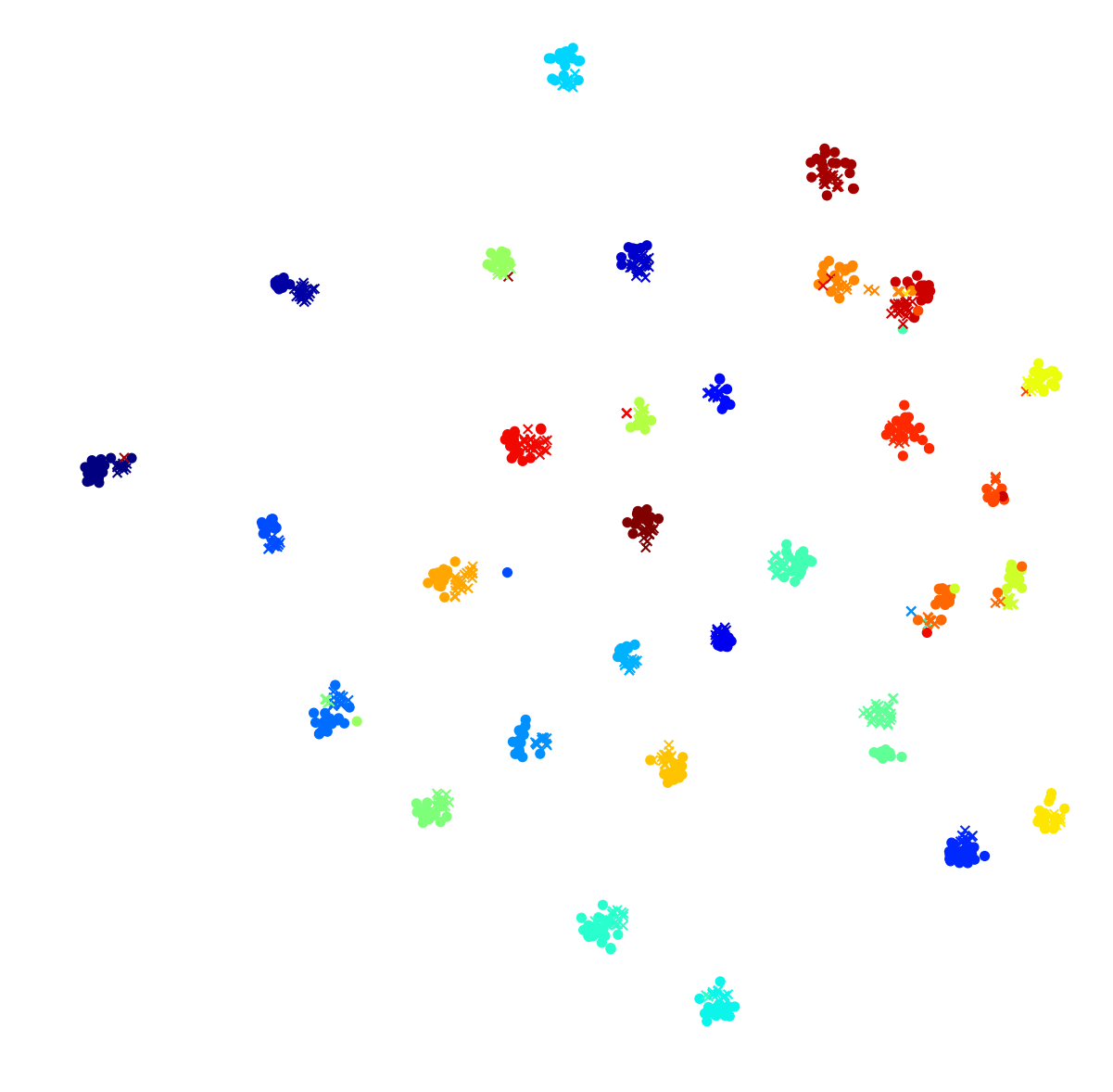}}\caption{The t-SNE visualization of \textbf{A}\textrightarrow \textbf{D} (Figure
a, b) tasks with label and domain information. Each color denotes
a class while the circle and cross markers represent the source and
target data respectively.\label{fig:The-t-SNE-visualization_AD}}
\vspace{-3mm}
 
\end{figure}

\paragraph{Feature visualization. }

We visualize the features of ResNet-50 and our methods on \textbf{A}\textrightarrow \textbf{D}
(\emph{Office-31}) and \textbf{P}\textrightarrow \textbf{C} (\emph{ImageCLEF-DA})
tasks by \emph{t}-SNE \cite{vanDerMaaten2008} in Figure \ref{fig:The-t-SNE-visualization_AD}.
The visualizations in Figure \ref{fig:ResNet-AD} and \ref{fig:ResNet-PC}
show that ResNet-50 classifies quite well on source domains (\textbf{A}
and \textbf{P}) but poorly on target domains (\textbf{D} and \textbf{C}).
While the representation in Figure \ref{fig:LDROT-AD} and \ref{fig:LDROT-PC}
is generated by our method with better alignment. LDROT achieves exactly
31 and 12 clusters corresponding to 31 and 12 classes of \emph{Office-31}
and \emph{ImageCLEF-DA}, which represents generalization ability of
our model in which the classifier generalizes well not only on the
source domain but also on the target domain.

\begin{figure}[h]
\centering \subfloat[ResNet.\label{fig:ResNet-PC}]{ 

}{\includegraphics[width=0.4\columnwidth]{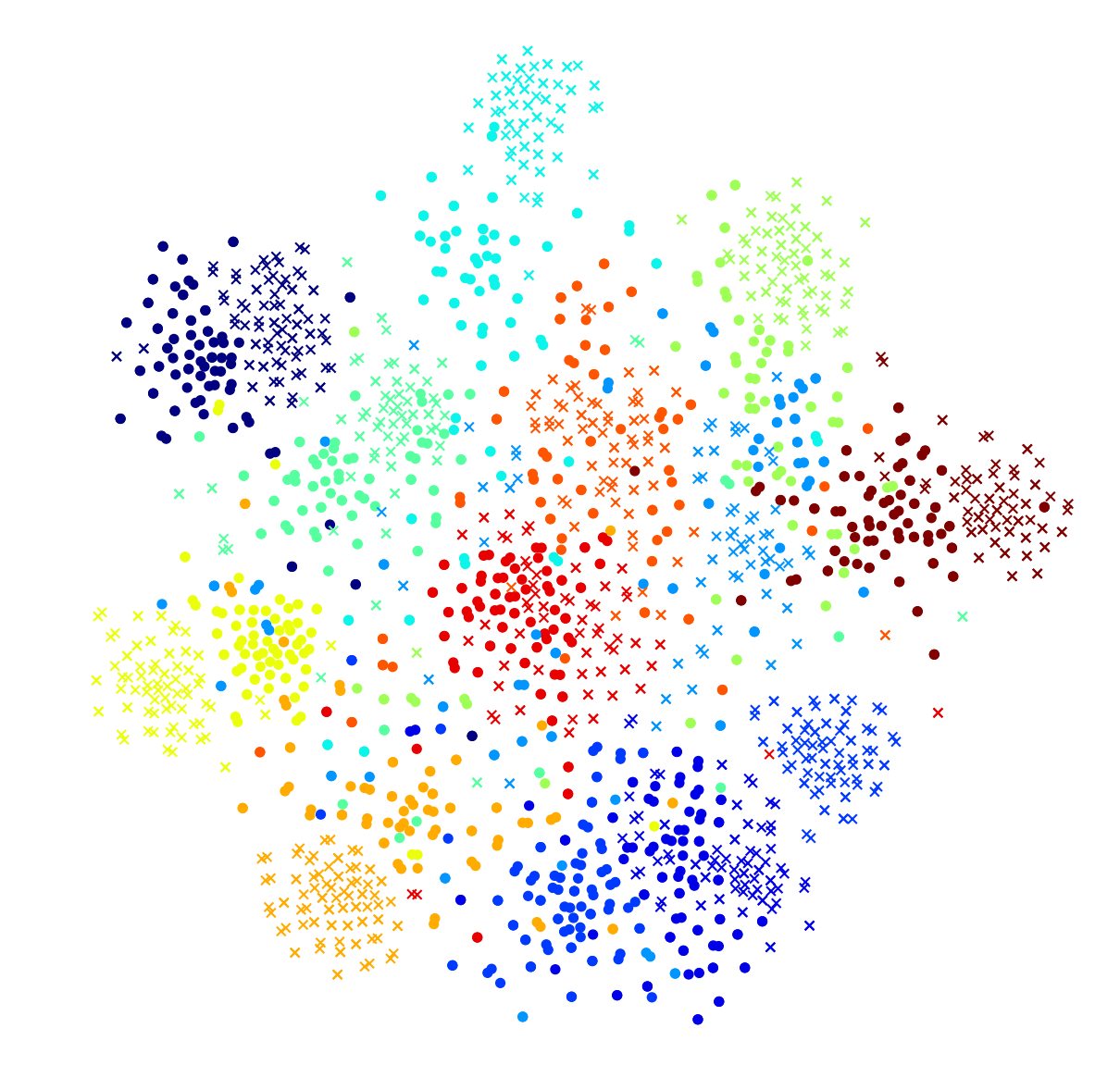}}
~~~~~~~~~~~~~~ \subfloat[LDROT.\label{fig:LDROT-PC}]{ 

}{\includegraphics[width=0.4\columnwidth]{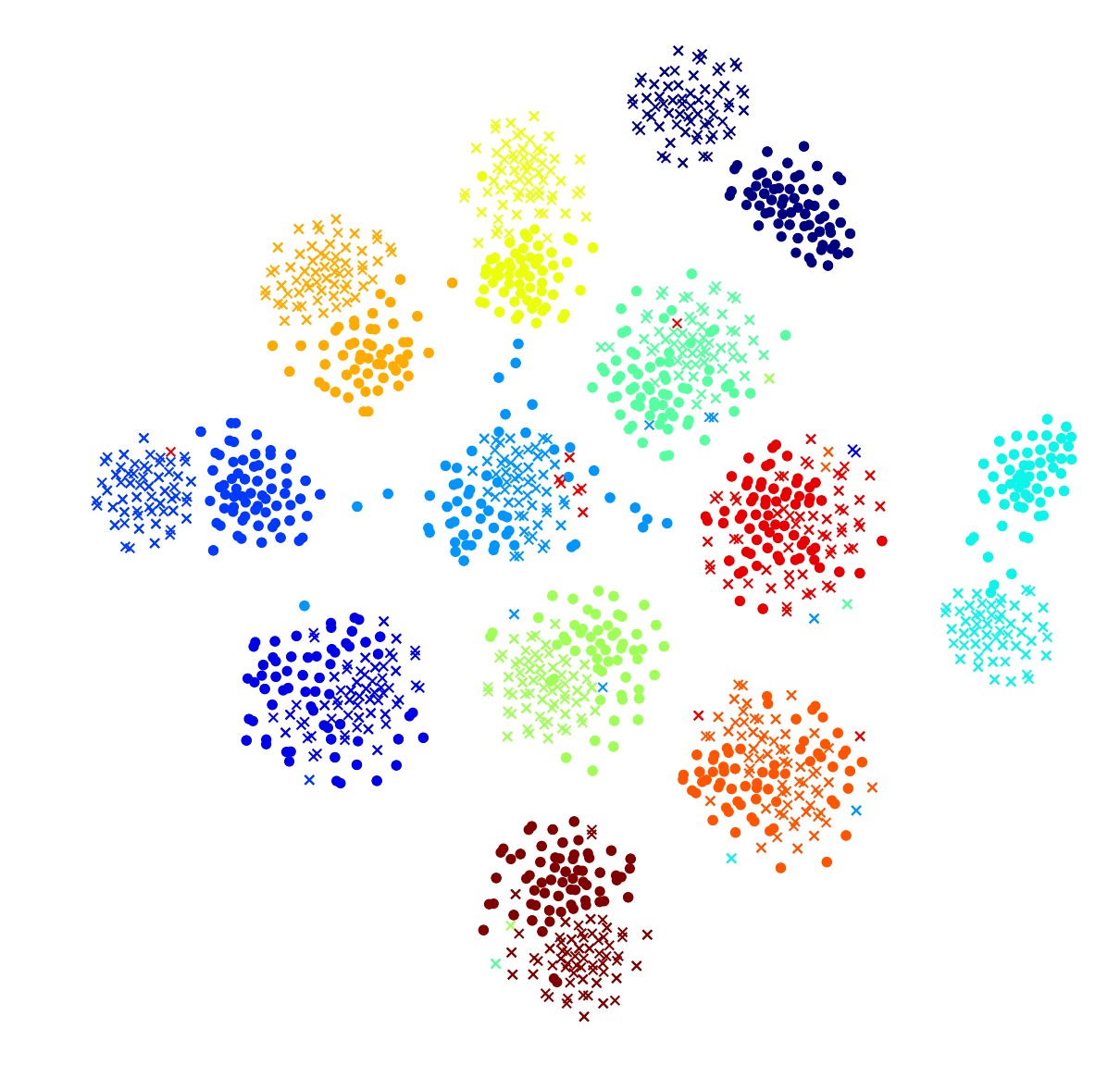}}
\caption{The t-SNE visualization of and \textbf{P}\textrightarrow \textbf{C}
(Figure a, b) tasks with label and domain information. Each color
denotes a class while the circle and cross markers represent the source
and target data respectively.\label{fig:The-t-SNE-visualization-PC}}
\vspace{-3mm}
 
\end{figure}

\paragraph{Convergence. }

We testify the convergence of our LDROT with the test errors on \textbf{A}\textrightarrow \textbf{D}
task, as shown in Figure \ref{fig:convergence}. We conduct experiments
on three methods including \emph{Source} (test error is achieved with
classifier trained on source data without adaptation), \emph{SWD}
\cite{chenyu2019swd}, and our \emph{LDROT}. For fair comparison,
the methods are applied the same optimizer (Adam with learning rate
of $0.0001$) and batch size. The results show that the error of \emph{LDROT}
on the target domain is remarkably lower, which illustrates better
generalization capability of the source classifier. During training,
our method encourages a target sample to actively moving to a suitable
group or cluster of source examples in a similarity-aware manner.
This phenomenon implies that \emph{LDROT} enjoys faster and stable
convergence than the other settings. 

\end{document}